\documentclass[twoside,11pt]{article}

\usepackage[preprint]{jmlr2e}

\usepackage{lastpage}
\jmlrheading{26}{2025}{1-\pageref{LastPage}}{1/24; Revised
	5/25}{6/25}{24-0059}{Ilyas Fatkhullin, Igor Sokolov, Eduard Gorbunov, Zhize Li, Peter Richtárik}

\ShortHeadings{EF21 with Bells \& Whistles:  \\Six Algorithmic Extensions of Modern Error Feedback}{Fatkhullin, Sokolov, Gorbunov, Li and Richtárik}
\firstpageno{1}

\usepackage{microtype}
\usepackage{graphicx}
\usepackage{subcaption}
\usepackage{booktabs} 

\usepackage{hyperref}

\usepackage{amsmath}
\usepackage{amssymb}
\usepackage{mathtools}

\usepackage[capitalize,noabbrev]{cleveref}

\usepackage[utf8]{inputenc} 
\usepackage[T1]{fontenc}    
\usepackage{hyperref}       
\usepackage{url}            
\usepackage{booktabs}       
\usepackage{amsfonts}       
\usepackage{nicefrac}       
\usepackage{microtype}      
\usepackage{xcolor}         

\usepackage{xspace}
\definecolor{mygreen}{HTML}{006B3C}
\newcommand{\algname}[1]{{{\sf \footnotesize \color{mygreen} #1}}\xspace}

\usepackage[flushleft]{threeparttable} 
\usepackage{caption}
\usepackage{multirow}
\usepackage{colortbl}
\definecolor{bgcolor}{rgb}{0.8,1,1}
\definecolor{bgcolor2}{rgb}{0.8,1,0.8}

\definecolor{myblue}{rgb}{0,0,0}
\definecolor{mylightblue}{rgb}{0,0.5,10.5}

\input{preamble.tex}

\begin{document}

\title{EF21 with Bells \& Whistles:  \\Six Algorithmic Extensions of Modern Error Feedback \thanks{The work was done when E.~Gorbunov was a PhD student at MIPT, and I.~Fatkhullin was a Master student at TU Munich and a summer intern at KAUST. \\ This is the extended version of a paper accepted for publication in JMLR. The extended version includes additional proofs and experimental results. } }

\author{\name Ilyas Fatkhullin \email  ilyas.fatkhullin@ai.ethz.ch \\
	\addr 
 Technical University of Munich, Germany\\
    King Abdullah University of Science and Technology, Saudi Arabia	\\
	ETH Zurich \& ETH AI Center, Switzerland
	\AND
	\name Igor Sokolov \email igor.sokolov.1@kaust.edu.sa \\
	\addr 
 King Abdullah University of Science and Technology, Saudi Arabia
	\AND
	\name Eduard Gorbunov \email eduard.gorbunov@mbzuai.ac.ae \\
	\addr Mohamed bin Zayed University of Artificial Intelligence, United Arab Emirates\\ Moscow Institute of Physics and Technology, Russia
		\AND
	\name Zhize Li \email zhizeli@smu.edu.sg \\
    \addr King Abdullah University of Science and Technology, Saudi Arabia\\
    Singapore Management University, Singapore
	\AND
	\name Peter Richtárik \email peter.richtarik@kaust.edu.sa \\
	\addr 
 King Abdullah University of Science and Technology, Saudi Arabia
	}


\editor{Zaid Harchaoui}

\maketitle

\begin{abstract}
	First proposed by \citet{Seide2014} as a heuristic, error feedback (\algname{EF}) is a very popular mechanism for enforcing convergence of distributed gradient-based optimization methods enhanced with communication compression strategies based on the application of contractive compression operators. However, existing theory of \algname{EF} relies on very strong assumptions (e.g., bounded gradients), and provides pessimistic convergence rates (e.g., while the best known rate for \algname{EF} in the smooth nonconvex regime, and when full gradients are compressed, is $O(1/T^{2/3})$, the rate of gradient descent in the same regime is $O(1/T)$). Recently, \citet{EF21} proposed a new error feedback mechanism, \algname{EF21}, based on the construction of a Markov compressor induced by a contractive compressor. \algname{EF21} removes the aforementioned theoretical deficiencies of \algname{EF} and at the same time works better in practice. In this work we propose six practical extensions of \algname{EF21}, all supported by strong convergence theory: partial participation, stochastic approximation, variance reduction, proximal setting, momentum and bidirectional compression. To the best of our knowledge, several of these techniques have not been previously analyzed in combination with \algname{EF}, and in cases where prior analysis exists—such as for bidirectional compression—our theoretical convergence guarantees significantly improve upon existing results.
\end{abstract}

\begin{keywords}
	distributed computing, compressed communication, error feedback.
\end{keywords}


\section{Introduction}

In this paper, we consider the nonconvex distributed optimization problem of the form
\begin{eqnarray}\label{eq:finit_sum}
	\textstyle	\min \limits_{x \in \mathbb{R}^{d}}\left\{f(x)\eqdef \frac{1}{n} \sum \limits_{i=1}^{n} f_{i}(x)\right\},
\end{eqnarray}
where $n$ denotes the number of clients/nodes connected with a server/master and client $i$ has an access to the local loss function $f_{i}$ only. 
The local loss of each client is allowed to have the online/expectation form
\begin{equation}\label{eq:online_case}
	f_{i}(x)=\Expu{\xi_{i} \sim \cD_i}{f_{\xi_{i}}(x)},
\end{equation}
or the finite-sum form 
\begin{equation}\label{eq:finite_sum_case}
	\textstyle	f_{i}(x)=\frac{1}{m} \sum \limits_{j=1}^{m} f_{i j}(x).
\end{equation}
A notable application for problems with such structure is federated learning \citep{FEDLEARN, FL-big}, where training is performed directly on the clients' devices. In a quest for state-of-the-art performance, machine learning practitioners develop elaborate model architectures and  train their models on large data sets. Naturally, in order to make the training at this scale tractable, one needs to rely on distributed computing \citep{goyal2017accurate, You2020Large}. Moreover, massively over-parameterized models have recently shown a remarkable empirical success \citep{ACH-overparameterized-2018}. However, the application of these models puts an additional complication on the communication links during training. In order to address this issue, recent research activity and practice focuses on developing distributed optimization methods and systems capitalizing on (deterministic or randomized) \emph{lossy communication compression} techniques to reduce the amount of  communication traffic.

A compression mechanism is typically formalized as an operator $\cC: \R^d \mapsto \R^d$ mapping hard-to-communicate (e.g., dense) input messages into easy-to-communicate (e.g., sparse) output messages. The operator is allowed to be randomized, and typically operates on models~\citep{GDCI} or on gradients~\citep{alistarh2017qsgd,beznosikov2023biased}, both of which can be described as vectors in $\R^d$. Besides sparsification~\citep{Alistarh-EF-NIPS2018}, typical examples of useful compression mechanisms include quantization~\citep{alistarh2017qsgd, horvoth2022natural} and low-rank approximation~\citep{PowerSGD,safaryan2022fednl}.  

There are two large classes of compression operators often studied in the literature: i) \emph{unbiased} compression operators $\cC$, meaning that there exists $\omega\geq 0$ such that for all $x\in \R^d$
\begin{eqnarray}\label{eq:ub_compressor}
	\Exp{\cC(x)}=x, \quad  \Exp{\|\cC(x) - x\|^{2}}&\leq& \omega \|x\|^{2}; 
\end{eqnarray}
and ii) \emph{biased} compression operators $\cC$, meaning that there exists $0<\alpha\leq 1$ such that for all $x\in \R^d$
\begin{eqnarray}\label{eq:b_compressor_0}
	\Exp{\|\cC(x) - x\|^{2}} \leq \rb{1 - \alpha} \|x\|^{2}.
\end{eqnarray}
Note that the latter ``biased'' class contains the former one, i.e., if $\cC$ satisfies \eqref{eq:ub_compressor} with $\omega$, then a scaled version $(1+\omega)^{-1}\cC$ satisfies \eqref{eq:b_compressor_0} with $\alpha = \nfr{1}{(1+\omega)}$. {\color{myblue} Beyond this inclusion, if used appropriately, biased compressors (such as Top-$k$ sparsifier) often perform better than unbiased ones (such as Rand-$k$) \citep{beznosikov2023biased}.}
While distributed optimization methods with unbiased compressors \eqref{eq:ub_compressor} are  well understood \citep{alistarh2017qsgd, DCGD, mishchenko2024distributed, DIANA2, ADIANA, CANITA, Nonconvex-sigma_k, islamov2021distributed, MARINA},
\emph{biased} compressors \eqref{eq:b_compressor_0} are significantly harder to analyze.  One of the main reasons behind this is rooted in the observation that when deployed  within distributed gradient descent in a naive way, biased compressors may lead to (even exponential) divergence \citep{Karimireddy_SignSGD, beznosikov2023biased}.
\emph{Error Feedback} (\algname{EF}) (or \emph{Error Compensation} (\algname{EC}))---a technique originally proposed by \citet{Seide2014}---emerged as an empirical fix of  this problem. However, this technique remained poorly understood  until very recently.

Although several theoretical results were obtained supporting the \algname{EF} framework in recent years  \citep{Stich-EF-NIPS2018,Alistarh-EF-NIPS2018,beznosikov2023biased,Lin_EC_SGD,EC-Katyusha,DoubleSqueeze,Koloskova2019DecentralizedDL,stich2020error}, they use strong assumptions (e.g., convexity, bounded gradients, bounded dissimilarity), and do not get $\cO(\nicefrac{1}{\alpha T})$ convergence rates in the smooth nonconvex regime. Very recently, \citet{EF21} proposed a new \algname{EF}  mechanism called \algname{EF21}, which uses standard smoothness assumptions only, and also enjoys the desirable $O(\nicefrac{1}{\alpha T})$ convergence rate for the nonconvex case (in terms of number of communication rounds $T$ this matches the best-known rate $\cO(\nicefrac{(1+\nicefrac{\omega}{\sqrt{n}})}{T})$ obtained by \citet{MARINA} using unbiased compressors), improving the previous $O(\nicefrac{1}{(\alpha T)^{2/3}})$ rate of the standard \algname{EF}  mechanism \citep{Koloskova2019DecentralizedDL}. 

\section{Contributions}
While \citet{EF21} propose a new error feedback method,  the authors only study their \algname{EF21}  mechanism in a pure form, without any additional ``bells and whistles'' which are important in practice. Therefore, it remains elusive whether \algname{EF21} method is a standalone technique or it can be enhanced with other related techniques to benefit its potential use in practice. In this paper, we aim to push the \algname{EF21} framework beyond its pure form by extending it in several directions of high theoretical and practical importance. In particular, we further enhance the  \algname{EF21}  mechanism  with the following six useful and practical algorithmic extensions: {\em stochastic approximation}, {\em variance reduction}, {\em partial participation}, {\em bidirectional compression}, {\em momentum}, and {\em proximal (regularization)}. We do not stop at merely proposing these algorithmic enhancements: we derive {\em strong convergence results for all of these extensions}. Several of these techniques were never analyzed in conjunction with the original \algname{EF} mechanism before. This fact reveals the challenges in the analysis of \algname{EF}-based methods and neccesitates the development of novel analysis techniques. Moreover, in the cases when the mentioned techniques were analyzed with \algname{EF}-based methods, we obtain new results that are superior in several aspects. See Table~\ref{tab:comparison} for an overview of our results. In summary, our results constitute the new algorithmic and theoretical state-of-the-art in the area of error feedback.

\begin{table*}[t]
	\centering
	
	\caption{ \footnotesize Summary of the state-of-the-art complexity results for finding an \textbf{$\varepsilon$-stationary point} using error-feedback type methods,  where $\varepsilon>0$ is an accuracy level. That is we aim to find a point $\hat x$ such that $\Exp{\|\nabla f(\hat x)\|^2} \le \varepsilon^2$, for generally non-convex functions and an \textbf{$\varepsilon$-solution}, i.e., such a point $\hat x$ that $\Exp{f(\hat x) - f(x^*)} \le \varepsilon$, for functions satisfying P{\L}-condition. By (computation) complexity we mean the average number of (stochastic) first-order oracle calls needed to find an $\varepsilon$-stationary point (``Compl.\ (NC)'') or $\varepsilon$-solution (``Compl.\ (P\L)''). Removing the terms colored {\color{blue}in blue} from the complexity bounds shown in the table, one can get communication complexity bounds, i.e., the total number of communication rounds needed to find an $\varepsilon$-stationary point (``Compl.\ (NC)'') or $\varepsilon$-solution (``Compl.\ (P\L)''). Dependences on the numerical constants, ``quality'' of the starting point, and smoothness constants are omitted in the complexity bounds. Moreover, dependencies on $\log(\nicefrac{1}{\varepsilon})$ are also omitted in the column ``Compl.\ (P\L)''. Abbreviations: ``BC'' = bidirectional compression, ``PP'' = partial participation; ``Mom.'' = momentum; $T$ = the number of communications rounds needed to find an $\varepsilon$-stationary point; $\overline{\#\text{grads}}$ = the number of (stochastic) first-order oracle calls needed to find an $\varepsilon$-stationary point. Notation: $\alpha$ = the compression parameter, $\alpha_w$ and $\alpha_M$ = the compression parameters of  worker and master nodes respectively for \algname{EF21-BC}, $\sigma^2 = \frac{1}{n}\sum_{i=1}^n\sigma_i^2$ (see Example~\ref{ex:UBV_case}), $\Delta^{\inf} = f^{\inf} - \frac{1}{n}\sum_{i=1}^n\frac{1}{m}\sum_{j=1}^{m} f_{ij}^{\inf}$ (see Example~\ref{ex:subsampling_case}), $p$ = probability of sampling the client in \algname{EF21-PP}, $\eta$ = momentum parameter. To the best of our knowledge, combinations of error feedback with partial participation (\algname{EF21-PP}) and proximal versions of error feedback (\algname{EF21-Prox}) were never analyzed in the literature.}
	\label{tab:comparison}    
	\begin{threeparttable}
		\begin{tabular}{|c|c|c c c c|}
			\hline
			Setup & Method & Citation & Compl.\ (NC) & Compl.\ (P{\L}) & Comment\\ 
			\hline\hline
			\makecell{Full\\ grads} &\algname{EF21} &\scriptsize \citep{EF21}  & $\frac{1}{\alpha\varepsilon^2}$ & $\frac{1}{\alpha\mu}$ & \\    
			\hline\hline
			\multirow{5.5}{0.7cm}{\centering Stoch.\ grads}&\algname{Choco-SGD} &\scriptsize\citep{Koloskova2019DecentralizedDL} & $\frac{1}{\varepsilon^2} + \frac{G}{\alpha\varepsilon^3} + \frac{\sigma^2}{n\varepsilon^4}$ &  N/A & $\|\nabla f_i(x)\| \leq G$\\
			&\algname{EF21-SGD} &\scriptsize \citep{EF21} & $\frac{1}{\alpha\varepsilon^2} + {\color{blue}\frac{\sigma^2}{\alpha^3\varepsilon^4}}$& $\frac{1}{\alpha\mu} + {\color{blue}\frac{\sigma^2}{\mu^2\alpha^3\varepsilon}}$ & UBV (Ex.~\ref{ex:UBV_case})\\
			&\algname{EF21-SGD} &{(this work)} & $\frac{1}{\alpha\varepsilon^2} + {\color{blue}\frac{1+\Delta^{\inf}}{\alpha^3\varepsilon^4}}$ & $\frac{1}{\alpha\mu} + {\color{blue}\frac{1 + \Delta^{\inf}}{\mu^2\alpha^3\varepsilon}}$ & IS (Ex.~\ref{ex:subsampling_case})\\
			&\algname{EF21-PAGE} &{(this work)} & $\frac{\sqrt{m} + \nicefrac{1}{\alpha}}{\varepsilon^2} + {\color{blue} m}$ & $\frac{\sqrt{m} + \nicefrac{1}{\alpha}}{\mu} + {\color{blue} m}$ & \makecell{Finite sum \\ form \eqref{eq:finite_sum_case}} 
			\\
			\hline\hline
			PP &\algname{EF21-PP} &{(this work)} & ${\color{red}\frac{1}{p\alpha\varepsilon^2}}\tnote{{\color{blue}(1)}}\;\; + {\color{blue}\frac{1}{\alpha\varepsilon^2}}$ & ${\color{red}\frac{1}{p\alpha\mu}}\tnote{{\color{blue}(1)}}\;\; + {\color{blue}\frac{1}{\alpha\mu}}$  & Full grads\\    
			\hline\hline
			\multirow{2}{0.7cm}{\centering BC}&\algname{DoubleSqueeze} &\scriptsize\citep{DoubleSqueeze} & $\frac{1}{\varepsilon^2} + \frac{\Delta}{\varepsilon^3} + \frac{\sigma^2}{n\varepsilon^4}$ &  N/A & \small $\mathbb E \|\cC(x) - x\| \leq \Delta$\\
			&\algname{EF21-BC} &{(this work)} & $\frac{1}{\alpha_w \alpha_M \varepsilon^2}$ & $\frac{1}{\alpha_w \alpha_M \mu}$  & Full grads\\			
			\hline\hline
			\multirow{2}{0.7cm}{\centering Mom.}&\algname{M-CSER} &\scriptsize\citep{CSER}\tnote{{\color{blue}(2)}} & $\frac{1}{\varepsilon^2} + \frac{G}{(1-\eta)\alpha\varepsilon^3}$ &  N/A & \small $\|\nabla f_i(x)\| \leq G$\\
			&\algname{EF21-HB} &{(this work)} & $\frac{1}{\varepsilon^2}\rb{\frac{1}{1-\eta}+\frac{1}{\alpha}}$ & N/A & Full grads\\
			\hline\hline
			\makecell{Prox} &\algname{EF21-Prox} &{(this work)} & $\frac{1}{\alpha\varepsilon^2}$ &  $\frac{1}{\alpha\mu}$\tnote{{\color{blue}(3)}} & Full grads \\    
			\hline
		\end{tabular}
		\begin{tablenotes}
			{\scriptsize
				\item [{\color{blue}(1)}]  Red term = number of communication rounds, blue term = expected number of gradient computations per client.				
				\item [{\color{blue}(2)}]   \citet{CSER} consider Nesterov's momentum. Moreover, they analyzed the version with stochastic gradients, bidirectional compression and local steps. However, the derived result is not better than state-of-the-art ones with either stochastic gradients or bidirectional compression. Therefore, to maintain the table compact, we do not include the results of \citet{CSER} in the other parts of the table.
				\item [{\color{blue}(3)}] This result is obtained under the generalized P{\L}-condition for composite optimization, see Appendix I.2. in \citep{EF21BW_2021}.
			}
		\end{tablenotes}
	\end{threeparttable}
\end{table*}

We now briefly comment on each extension proposed in this paper:

{\bf $\diamond$ Stochastic approximation.} 
Vanilla \algname{EF21} method requires all clients to compute the exact/full gradient in each round.\footnote{While \citet{EF21} do consider a stochastic extension of \algname{EF21} in their Appendix F, they do not formalize their result, and only consider the simplistic scenario of uniformly bounded variance, which does not in general hold for stochasticity coming from subsampling \citep{khaled2020better}.}  However, exact gradients are not available in  the stochastic/online setting \eqref{eq:online_case}, and in the finite-sum setting \eqref{eq:finite_sum_case}  it is more efficient in practice to use subsampling and work with stochastic gradients  instead. In our paper, we extend \algname{EF21} to a more general stochastic approximation framework than the simplistic full gradient setting considered in the original paper. 

{\bf $\diamond$ Variance reduction.} 
As mentioned above, \algname{EF21} relies on full gradient computations at all clients. 
This incurs a high or unaffordable computational cost, especially when local clients hold  large training sets, i.e., if $m$ is very large in \eqref{eq:finite_sum_case}. One important technique for accelerating convergence is to incorporate a \textit{variance reduction} mechanism, which makes use of stochastic gradient estimates obtained in the previous iterations. To the best of our knowledge, it is an open question whether any \algname{EF}-type mechanism can be enhanced with variance reduction for non-convex objectives. We answer this question in this work by proposing  \algname{EF21-PAGE} method and developing an analysis based on a new Lyapunov function. 

{\bf $\diamond$ Partial participation.} 
Pure \algname{EF21} method requires  \emph{full participation} of clients for solving problem \eqref{eq:finit_sum}, i.e., in each round, the server needs to communicate with all $n$ clients.  However, full participation is usually  impractical or very hard to achieve in massively distributed (e.g., federated) learning problems \citep{FEDLEARN, Power_of_Choice, FL-big, ZeroSARAH, FedPAGE}. To remedy this situation, we propose a \emph{partial participation} (PP) variant of \algname{EF21}, \algname{EF21-PP} (Algorithm~\ref{alg:PP-EF21}), which allows to sample only a random subset of clients at each iteration.

{\bf $\diamond$ Bidirectional compression.}  
In many distributed computing systems the {\em upstream}  of communication of messages is the main bottleneck. However, in other architechtures, the {\em downstream } communication is also costly \citep{horvoth2022natural, DoubleSqueeze, Artemis2020} or even has a fixed bandwidth, which can significantly slow down training. 
In order to address this issue, we further enhance \algname{EF21} method by backward compression and propose \algname{EF21-BC} (Algorithm~\ref{alg:EF21-BC}). Our biderectional compression method carefully employs the Markov compressor based on \algname{EF21} on the master and client nodes simultaneously. Moreover, we design a novel analysis for the proposed algorithm, which is reminiscent of the our analysis of \algname{EF21-PAGE}.

{\bf $\diamond$ Momentum.}  
A very successful and popular technique for enhancing both optimization and generalization is momentum/acceleration \citep{Heavy-ball,nesterov1,lan2018optimal,allen2017katyusha,Varag,li2021anita,SMOMENTUM}. Moreover, momentum is a key building block behind the widely-used \algname{Adam} method~\citep{Adam}. However, in the context of error feedback, acceleration is notoriously difficult to analyze. For instance, in convex regime, additional full vector communication is needed for the analysis \citep{EC-Katyusha}. In non-convex case, the best-known complexity is acheived by \algname{M-CSER} method \citep{CSER}, which is clearly suboptimal in terms $\varepsilon$, $\eta$ and $\alpha$, see Table~\ref{tab:comparison}. In this work, we overcome this difficulty by carefully incorporating momentum into \algname{EF21}. We name the resulting method \algname{EF21-HB} (Algorithm~\ref{alg:EF21_HB}) and offer a simple intuitive proof with improved convergence guarantees.

{\bf $\diamond$ Proximal setting.}  
It is common practice to solve {\em regularized} versions of empirical risk minimization problems instead of their vanilla variants \citep{shai_book}.  Thus we consider the regularized (proximal/composite) problem
\begin{equation}\label{eq:composite_optimization_0}
	\textstyle	\min \limits_{x\in \R^d} \left\{ \Phi(x) \eqdef \frac{1}{n} \sum \limits_{i=1}^n f_i(x) + r(x) \right\},
\end{equation}
where $r(x):\R^d \to \R\cup \{+\infty\}$ is a regularizer, e.g., $\ell_1$ regularizer $\|x\|_1$ or $\ell_2$ regularizer $\|x\|_2^2$. To broaden the applicability of error feedback to such problems, we propose a proximal variant of \algname{EF21} to solve the more general composite problems~\eqref{eq:composite_optimization_0}, which leads to our \algname{EF21-Prox} method (Algorithm~\ref{alg:Prox-ef21}). Again, we are not aware of any method, which can provably solve problem \eqref{eq:composite_optimization_0} using the Top-$k$ sparsifier in distributed non-convex setting.

Our theoretical complexity results are summarized in Table~\ref{tab:comparison}. {\color{myblue} We describe each algorithm in detail in Section~\ref{sec:methods} and present the main results of the convergence analysis for all extensions in Section~\ref{sec:theory}. Proof sketches and formal proofs are deferred to their respective sections in the Appendix.} In addition, we also analyze \algname{EF21-SGD}, \algname{EF21-PAGE}, \algname{EF21-PP}, \algname{EF21-BC} under Polyak-{\L}ojasiewicz (P\L) condition \citep{polyak1963gradient, lojasiewicz1963topological} and \algname{EF21-Prox} under the generalized P{\L}-condition \citep{li2018simple} for composite optimization problems. Due to space limitations, we defer all the details about the analysis under the P{\L}-condition to the extended version of this work \citep{EF21BW_2021} and provide only simplified rates in Table~\ref{tab:comparison}. We comment on some preliminary  experimental results in Section~\ref{sec:exp}. More experiments including deep learning experiments are presented in Appendix~\ref{sec:exp_extra}.

{\color{myblue} 
\section{Notations} We adopt the common conventions $[n] = \cb{1, \ldots, n}$ for a set of indicies and  $\Prob(\mathcal A) $ for a probability of event $\mathcal A$. Throughout the paper, $\norm{\cdot}$ denotes the Euclidean norm $\norm{\cdot}_2$ unless otherwise stated. For algorithmic notations, we refer to specific algorithms in the next section. In Appendix~\ref{sec:notations_methods}, we summarize all notations used in our theoretical analysis. 
}

\section{Methods: Six Algorithmic Extensions}\label{sec:methods}

The proposed methods are extensions of \algname{EF21}, thus they share some features, and are presented in a unified way in Table~\ref{tab:methods}. For all methods, at each iteration, worker $i$ computes the compressed vector $c_i^t$ and sends it to the master. The methods \algname{EF21-SGD}, \algname{EF21-PAGE}, \algname{EF21-PP} differ in how the compressed vectors $c_i^t$ are computed, while the aggregation and parameter update rules are the same:
\begin{eqnarray}
	\textstyle	x^{t+1} = x^t - \gamma g^t,\quad g_i^{t+1} = g_i^t + c_i^t, \notag \\ g^{t+1} = \frac{1}{n}\sum\limits_{i=1}^n g_i^{t+1} = g^t + \frac{1}{n}\sum\limits_{i=1}^n c_i^t.\label{eq:EF21_update}
\end{eqnarray}
The methods \algname{EF21-BC}, \algname{EF21-HB}, \algname{EF21-Prox} compute the compressed vectors via $c_i^t = \cC(\nabla f_i(x^{t+1}) - g_i^t)$, while the aggregation rule and parameter updates are specific to each method. The pseudocodes of the algorithms are given below and important distinct parts are highlighted in {\color{mylightblue} light blue}. Below we briefly describe each method.\footnote{Note that in this work, we study the effect of each of the $6$ proposed extensions separately for pedagogical and clarity reasons. However, in practice, it can be desirable to combine more enhancements with EF21 simultaneously to achieve state-of-the-art performance. In fact, due to the flexibility of our analysis, several of the proposed extensions can be easily combined into one method.
}

{\bf $\diamond$ \algname{EF21-SGD}: Error feedback and \algname{SGD}.} \algname{EF21-SGD} is \algname{EF21} with  full gradients $\nabla f_i(x^{t+1})$ being replaced by their stochastic estimates $\hat g_i(x^{t+1})$  at each node. The pseudocode is given in Algorithm~\ref{alg:EF21-online}. Each client computes $c_i^t = \cC(\hat g_i(x^{t+1}) - g_i^t)$ and sends this sparsified vector to the server. Despite the simplicity of this extension, it is important for various applications of machine learning and statistics where exact gradients are either unavailable or prohibitively expensive to compute.

\begin{algorithm}[h]
	\centering
	\caption{\algname{EF21-SGD} }\label{alg:EF21-online}
	\begin{algorithmic}[1]
		\STATE \textbf{Input:} starting point $x^{0} \in \R^d$;  $g_i^0 \in \R^d$ (known by nodes); $g^0 = \frac{1}{n}\sum_{i=1}^n g_i^0$ (known by master); learning rate $\gamma>0$
		\FOR{$t=0,1, 2, \dots , T-1 $}
		\STATE Master computes $x^{t+1} = x^t - \gamma g^t$ and broadcasts $x^{t+1}$ to all nodes
		\FOR{{\bf all nodes $i =1,\dots, n$ in parallel}}
		\STATE {\color{mylightblue} Compute a stochastic gradient $\hgitpo  = \fr{1}{\tau} \sum_{j=1}^{\tau}\nabla f_{\xi_{i j}^{t}}(\xtpo)$}
		\STATE Compress $c_i^t = \cC({\color{mylightblue} \hgitpo } - g_i^t)$ and send $c_i^t $ to the master
		\STATE Update local state $g_i^{t+1} = g_i^t + c_i^t $
		\ENDFOR
		\STATE Master computes $g^{t+1} = \frac{1}{n} \sum_{i=1}^n  g_i^{t+1}$ via  $g^{t+1} = g^t + \frac{1}{n} \sum_{i=1}^n c_i^t $
		\ENDFOR
	\end{algorithmic}
\end{algorithm}

{\bf $\diamond$ \algname{EF21-PAGE}: Error feedback and variance reduction.} In the finite-sum setting \eqref{eq:finite_sum_case}, it is well known that variance reduced methods have better theoretical guarantees and often perform better than vanilla \algname{SGD} \citep{gower2020variance}. Therefore, we enhance \algname{EF21} with variance reduction technique aiming to acheive a stronger combined effect of variance reduction and compressed communication. Specifically, we replace $\nabla f_i(x^{t+1})$ in the formula for $c_i^t$ with the \algname{PAGE} estimator~$v_i^{t+1}$. With (typically small) probability $p$ this estimator equals the full gradient $v_i^{t+1} = \nabla f_i(x^{t+1})$, and with probability $1-p$ it is set to
	$\textstyle	v_i^{t+1} = v_i^{t}+ \fr{1}{\tau_i} \sum \limits_{j\in I_{i}^{t}} \rb{\nabla f_{ij}(x^{t+1}) - \nabla f_{ij}(\xt)},$
where $I_i^t$ is a minibatch of size $\tau_i$. Typically, the number of data points $m$ owned by each client is large, and $p \leq \nfr{1}{m}$ when $\tau_i \equiv 1$. As a result, computation of full gradients rarely happens during the optimization procedure: on average,  once in every $m$ iterations only. Although it is possible to use other variance-reduced estimators like in \algname{SVRG} or \algname{SAGA}, we use the \algname{PAGE}-estimator: unlike \algname{SVRG} or \algname{SAGA}, \algname{PAGE} is optimal for smooth nonconvex optimization, and therefore gives the best theoretical guarantees.\footnote{We have obtained results for both \algname{SVRG} and \algname{SAGA} and indeed, they are worse, and hence we do not include them.}

Notice that unlike \algname{VR-MARINA} \citep{MARINA}, which is a state-of-the-art distributed optimization method designed specifically for unbiased compressors and which {\em also} uses the \algname{PAGE}-estimator, by design our \algname{EF21-PAGE} does not require the communication of full (non-compressed) vectors at all. This is an important property of the algorithm since, in some distributed networks, and especially when $d$ is very large, as is the case in modern over-parameterized deep learning, full vector communication is prohibitive. 

\begin{algorithm}[H]
	\centering
	\caption{\algname{EF21-PAGE}}\label{alg:EF21-PAGE}
	\begin{algorithmic}[1]
		\STATE \textbf{Input:} starting point $x^{0} \in \R^d$; $g_i^0 $, $v_i^0 \in \R^d$ for $i=1,\dots, n$ (known by nodes); $g^0 = \fr{1}{n}\sum_{i=1}^n g_i^0$ (known by master); learning rate $\gamma>0$; probabilities $p_i \in (0,1]$; batch-sizes $1\leq\tau_i \leq m$
		\FOR{$t=0,1, 2, \dots , T-1 $}
		\STATE Master computes $x^{t+1} = x^t - \gamma g^t$  
		\FOR{{\bf all nodes $i =1,\dots, n$ in parallel}}
		\STATE {\color{mylightblue} Sample $ b_i^{t} \sim \operatorname{Be}(p_i)$}
		\STATE {\color{mylightblue} If $b_i^t = 0$, sample a minibatch of data samples $I_{i}^{t}$ with $|I_i^{t}|=\tau_i$}
		\STATE {\color{mylightblue} $v_i^{t+1} = \begin{cases}
			\nabla f_i(x^{t+1}) &\text{if } b_i^t  = 1,\\
			v_i^{t}+ \fr{1}{\tau_i} \sum \limits_{j\in I_i^{t}} \rb{\nabla f_{ij}(x^{t+1}) - \nabla f_{ij}(\xt)}&\text{if } b_i^t  = 0
		\end{cases}$  }
		\STATE Compress $c_i^t = \cC(v_i^{t+1} - g_i^t)$ and send $c_i^t $ to the master
		\STATE Update local state $g_i^{t+1} = g_i^t + c_i^t $
		
		\ENDFOR
		
		\STATE Master computes $g^{t+1} = \fr{1}{n} \sum_{i=1}^n  g_i^{t+1}$ via  $g^{t+1} = g^t + \fr{1}{n} \sum_{i=1}^n c_i^t $
		\ENDFOR
	\end{algorithmic}
\end{algorithm}

{\bf $\diamond$ \algname{EF21-PP}: Error feedback and partial participation.} In this setting, we assume that only a subset of clients is available for computation/communication at each round. We model such situation as follows. First, we select a subset of clients $S_t\subseteq\{1,\dots,n\}$ randomly such that $\Prob(i \in S_t) = p_i > 0 $ for all $i = 1, \dots, n$, where $\cb{p_i}_{i=1}^{n}$ are unknown probabilities. {\color{myblue} We allow for an arbitrary sampling  strategy of a subset $S_t$ at the master node. The only requirement is that $p_i > 0 $ for all $i = 1, \dots, n$, which is often referred to as a \textit{proper arbitrary} sampling.\footnote{It is natural to focus on \textit{proper} samplings only since otherwise there is a node $i$, which never communicaties. This would be a critical issue when trying to minimize \eqref{eq:finit_sum} as we do not assume any similarity between $f_i(\cdot)$.} Many popular sampling procedures fell into this setting, for instance, independent sampling with/without replacement, $\tau$-nice sampling.\footnote{We do not discuss particular sampling strategies here, more details on specific sampling procedures can be found, e.g., in \citep{qu2016coordinate}.}} After we select a subset $S_t$, each client from $S_t$ computes $c_i^t = \cC(\nfixtpo - g_i^t)$ and communicates this information to the server, while other clients do not participate in the round, which is mathematically equivalent to setting $c_i = 0$. Finally, the server aggregates the communicated vectors and forms a gradient estimator by setting $g_i^{t+1} = g_i^t + c_i^t$ for the clients in $S_t$ and reusing the previous estimate $g_i^{t+1} = g_i^t $ for those nodes which did not take part in this round.

{\color{myblue} The modified method (Algorithm~\ref{alg:PP-EF21}) is called
\algname{EF21-PP}. Note, that all other clients (nodes) $i \notin S_t$ participate neither in the computation nor in communication at iteration $t$, which can save additional computational effort. }

\begin{algorithm}[H]
	\centering
	\caption{\algname{EF21-PP} (EF21 with partial participation)}\label{alg:PP-EF21}
	\begin{algorithmic}[1]
		\STATE \textbf{Input:} starting point $x^{0} \in \R^d$;  $g_i^0 \in \R^d $ for $i=1,\dots, n$ (known by nodes); $g^0 = \fr{1}{n}\sum_{i=1}^n g_i^0$ (known by master); learning rate $\gamma>0$ 
		\FOR{$t=0,1, 2, \dots , T-1 $}
		\STATE Master computes $x^{t+1} = x^t - \gamma g^t$  
		\STATE {\color{mylightblue} Master samples a subset $S_t$ of nodes ($|S_t| \leq n$) such that $\Prob\rb{i \in S_t} = p_i$ }
		\STATE {\color{mylightblue} Master broadcasts $x^{t+1}$ to the nodes with $i\in S_t$}
		\FOR{{\bf all nodes $i =1,\dots, n$ in parallel}}
		\IF{$i\in S_t$}
		\STATE Compress $c_i^t = \cC(\nabla f_i(x^{t+1}) - g_i^t)$ and send $c_i^t $ to the master
		\STATE Update local state $g_i^{t+1} = g_i^t + c_i^t $
		\ENDIF
	{\color{mylightblue} 	\IF{$i\notin S_t$}
		\STATE  Do not change local state $g_i^{t+1} = g_i^t$ 
		\ENDIF }
		\ENDFOR
		{\color{mylightblue} \STATE Master updates $g_i^{t+1} = g_i^t$, $c_i^t = 0$ for $i\notin S_t$}
		\STATE Master computes $g^{t+1} = \fr{1}{n} \sum_{i=1}^n  g_i^{t+1}$ via  $g^{t+1} = g^t + \fr{1}{n} \sum_{i=1}^n c_i^t $
		\ENDFOR
	\end{algorithmic}
\end{algorithm}

{\bf $\diamond$ \algname{EF21-BC}: Error feedback and bidirectional compression.} We extend \algname{EF21} to the case when it is desirable to obtain efficient communication between the clients and the server in \textit{both directions}. We present the formal pseudocode of the method in Algorithm~\ref{alg:EF21-BC}. Note that $\cC_M$ and $\cC_w$ stand for contractive compressors of the type \eqref{def:contractive_compressor} of master and workers respectively. In general, different $\al_M$ and $\al_w$ are accepted. At each iteration of \algname{EF21-BC}, clients compute and send to the master node $c_i^t = \cC_w(\nabla f_i(x^{t+1}) - \wg_i^t)$ and update $\wg_i^{t+1} = \wg_i^t + c_i^t$ in the usual way, i.e., clients apply \algname{EF21} mechanism. The key enhancement in \algname{EF21-BC} is that the master node in \algname{EF21-BC} also follows a similar procedure: it computes and broadcasts to clients the compressed vector $b^{t+1} = \cC_M( \wg^{t+1} - \gt)$ and updates $\gtpo = \gt + b^{t+1}$, where  $\wg^{t+1} = \frac{1}{n} \sum_{i=1}^n  \wg_i^{t+1}$. Vector $\gt$ is  maintained by the master {\em and}  clients. Therefore, the clients are able to update it via  $\gtpo = \gt + b^{t+1}$ and compute $x^{t+1} = x^t - \gamma g^t$ once they receive $b^{t+1}$. 

\begin{algorithm}[H]
	\centering
	\caption{\algname{EF21-BC} (EF21 with bidirectional biased compression)}\label{alg:EF21-BC}
	\begin{algorithmic}[1]
		\STATE \textbf{Input:} starting point $x^{0} \in \R^d$; $g^0$, $b^0$, $\wg_i^0 \in\R^d$ for $i=1,\dots, n$ (known by nodes); $\wg^0 = \frac{1}{n}\sum_{i=1}^n \wg_i^0$ (known by master) ; learning rate $\gamma>0$
		\FOR{$t=0,1, 2, \dots , T-1 $}
		\STATE Master updates $x^{t+1} = x^t - \gamma g^t$
		\FOR{{\bf all nodes $i =1,\dots, n$ in parallel}}
		\STATE Update $x^{t+1} = x^t - \gamma g^t$, $\gtpo = \gt + b^{t}$
		\STATE compress $c_i^t = \cC_w(\nabla f_i(x^{t+1}) - \wg_i^t)$, send $c_i^t $ to the master, and \label{alg_line:BD-EF21_cit_def}
		\STATE update local state $\wg_i^{t+1} = \wg_i^t + c_i^t $
		\ENDFOR
	{\color{mylightblue} 	\STATE Master computes $\wg^{t+1} = \frac{1}{n} \sum_{i=1}^n  \wg_i^{t+1}$ via  $\wg^{t+1} = \wg^t + \frac{1}{n} \sum_{i=1}^n c_i^t $  \label{alg_line:BD-EF21_gtpo_update}
		\STATE compreses $b^{t+1}= \cC_M(\wgtpo - \gt)$, broadcast $b^{t+1} $ to workers, and
		\STATE updates $\gtpo = \gt + b^{t+1}$ } \label{alg_line:BD-EF21_gtpo_update_C}
		\ENDFOR
	\end{algorithmic}
\end{algorithm}

{\bf $\diamond$ \algname{EF21-HB}: Error feedback with momentum.} We design a momentum \citep{Heavy-ball} variant of \algname{EF21} by computing a moving average estimator based on the vector~$g^t$ formed by \algname{EF21}:
\begin{eqnarray}
	\textstyle	x^{t+1} &=& x^t - \gamma v^t,\quad v^{t+1} = \eta v^t + g^{t+1},\notag\\ g_i^{t+1} &=& g_i^t + c_i^t, \quad
	 g^{t+1} = \frac{1}{n}\sum\limits_{i=1}^n g_i^{t+1} = g^t + \frac{1}{n}\sum\limits_{i=1}^n c_i^t. \notag
\end{eqnarray}
The resulting method obtains an improved iteration complexity compared to the current state-of-the-art momentum based method \algname{M-CSER} in terms of the dependence on $\varepsilon$, $\eta$ and $\al$, see Table~\ref{tab:comparison}. Compared to \algname{EF21}, its momentum variant  \algname{EF21-HB} has the same complexity  (in terms of  $\varepsilon$ and $\alpha$), i.e., momentum does not  provably improve the convergence rate.\footnote{Unfortunately, this is a common issue for a wide range of results for momentum methods \cite{SMOMENTUM}. However, it is important to theoretically analyze momentum-extensions such as \algname{EF21-HB} due to their importance in practice and generalization behaviour.}

\begin{algorithm}[H]
	\centering
	\caption{\algname{EF21-HB}}\label{alg:EF21_HB}
	\begin{algorithmic}[1]
		\STATE \textbf{Input:} starting point $x^{0} \in \R^d$;  $g_i^0 \in \R^d$ for $i=1,\dots, n$ (known by nodes); $v^0 = g^0 = \frac{1}{n}\sum_{i=1}^n g_i^0$ (known by master); learning rate $\gamma>0$; momentum parameter $0 \leq \eta < 1$
		\FOR{$t=0,1, 2, \dots , T-1 $}
		\STATE Master computes $x^{t+1} = x^t - \gamma v^t$ and broadcasts $x^{t+1}$ to all nodes \label{line:EF21_HB_master_step}
		\FOR{{\bf all nodes $i =1,\dots, n$ in parallel}}
		\STATE Compress $c_i^t = \cC(\nabla f_i(x^{t+1}) - g_i^t)$ and send $c_i^t $ to the master
		\STATE Update local state $g_i^{t+1} = g_i^t + c_i^t $
		\ENDFOR
		\STATE Master computes $g^{t+1} = \frac{1}{n} \sum_{i=1}^n  g_i^{t+1}$ via  $g^{t+1} = g^t + \frac{1}{n} \sum_{i=1}^n c_i^t $, and {\color{mylightblue} $v^{t+1} = \eta v^t + g^{t+1}$} \label{line:EF21_HB_momentum_upd}
		\ENDFOR
	\end{algorithmic}
\end{algorithm}

{\bf $\diamond$ \algname{EF21-Prox}: Error feedback for composite problems.} Finally, we make \algname{EF21} applicable to the composite optimization problems \eqref{eq:composite_optimization_0} by simply taking the prox-operator from the right-hand side of the $x^{t+1}$ update rule \eqref{eq:EF21_update}: $$x^{t+1} = \operatorname{prox}_{\gamma r} \rb{x^t - \gamma g^t} \eqdef \argmin_{x\in \R^d}\cb{\gamma r(x) + \frac{1}{2}\|x - x^t + \gamma g^t\|^2  }.$$ This modification is simple, but, surprisingly, \algname{EF21-Prox} is the first distributed method with error-feedback that provably converges for composite problems \eqref{eq:composite_optimization_0}. The technical reason for this is that the perturbed iterate analysis of the original \algname{EF} \citep{Stich-EF-NIPS2018,stich2020error} is difficult to extend to the composite/constrainted setting due to additional bias of the proximal operator.

\begin{algorithm}[H]
	\centering
	\caption{\algname{EF21-Prox}}\label{alg:Prox-ef21}
	\begin{algorithmic}[1]
		\STATE \textbf{Input:} starting point $x^{0} \in \R^d$; $g_i^0 \in \R^d$  for $i=1,\dots, n$ (known by nodes); $g^0 = \fr{1}{n}\sum_{i=1}^n g_i^0$ (known by master); learning rate $\gamma>0$
		\FOR{$t=0,1, 2, \dots , T-1 $}
		{\color{mylightblue} \STATE Master computes $x^{t+1} = \operatorname{prox}_{\gamma r} \rb{x^t - \gamma g^t}$  }
		\FOR{{\bf all nodes $i =1,\dots, n$ in parallel}}
		\STATE Compress $c_i^t = \cC(\nabla f_i(x^{t+1}) - g_i^t)$ and send $c_i^t $ to the master
		\STATE Update local state $g_i^{t+1} = g_i^t + c_i^t $
		\ENDFOR
		\STATE Master computes $g^{t+1} = \frac{1}{n} \sum_{i=1}^n  g_i^{t+1}$ via  $g^{t+1} = g^t + \frac{1}{n} \sum_{i=1}^n c_i^t $
		\ENDFOR
	\end{algorithmic}
\end{algorithm}

\section{Theoretical Convergence Results}\label{sec:theory}

In this section, we formulate a single corollary derived from the main convergence theorems for our six enhancements of \algname{EF21}, and formulate the assumptions that we use in the analysis. The complete statements of the theorems and their proofs are provided in the appendices. In Table~\ref{tab:comparison} we compare our new findings with existing results.

\subsection{Assumptions}

In this subsection, we list and discuss the assumptions that we use in the analysis. 

\subsubsection{General assumptions} 

To derive our convergence results, we invoke the following standard smoothness assumption.

\begin{assumption}[Smoothness and lower boundedness]\label{as:main}	 
	Every $f_i$ has $L_i$-Lipschitz gradient, i.e., \\$\norm{\nabla f_i(x) - \nabla f_i(y)} \le L_i\norm{x - y}$ for all $i\in [n], x, y\in \R^d$, and $\finf \eqdef \inf_{x\in \R^d} f(x)>-\infty $.
\end{assumption}

We also assume that the compression operators used by all algorithms satisfy the following property.

\begin{definition}[Contractive compressors]\label{def:contractive_compressor}
	We say that a (possibly randomized) map $\cC: \R^{d} \rightarrow \R^{d}$ is a {\em contractive compression operator}, or simply {\em contractive compressor}, if   there exists a constant $0<\alpha\leq 1$ such that	
		\begin{eqnarray}\label{eq:b_compressor}
		\Exp{\|\cC(x) - x\|^{2}} \leq \rb{1 - \alpha} \|x\|^{2}, \qquad \forall x\in \R^d.
	\end{eqnarray}
\end{definition}

We emphasize that we do {\em not} assume  $\cC$ to be unbiased. Hence, in particular, our theory works with the popular greedy Top-$k$ sparsifer~\citep{Alistarh-EF-NIPS2018}, which selects the largest $k$ coordinates of the compressed vector in the magnitude.

\subsubsection{Addtional assumptions for \algname{EF21-SGD}}

We analyze \algname{EF21-SGD} under the assumption that local stochastic gradients $\nabla f_{\xi_{i j}^{t}}(\xt)$ satisfy the following inequality (see Assumption~2 of \citet{khaled2020better}).

\begin{assumption}[General assumption for stochastic gradients]\label{as:general_as_for_stoch_gradients}
	We assume that for all $i = 1,\ldots,n$ and $j\geq 1$, we have $\Exp{\nabla f_{\xi_{i j}^{t}}(\xt) \mid x^t} = \nabla f_i(\xt)$, and  there exist parameters $A_i, C_i \ge 0$, $B_i \ge 1$ such that
	\begin{eqnarray}
		\Exp{\|\nabla f_{\xi_{i j}^{t}}(\xt)\|^2\mid x^t} &\le& 2A_i\left(f_i(x^t) - f_i^{\inf}\right)   \label{eq:general_second_mom_upp_bound}  + B_i\|\nabla f_i(x^t)\|^2 + C_i,  
	\end{eqnarray}
	where\footnote{When $A_i = 0$ one can ignore the first term in the right-hand side of \eqref{eq:general_second_mom_upp_bound}, i.e., assumption $\inf_{x\in\R^d}f_i(x) > -\infty$ is not required in this case.} $f_i^{\inf} = \inf_{x\in\R^d}f_i(x) > -\infty$.
\end{assumption}
{\color{myblue}
Stochastic gradient $\hgit$ is computed using a mini-batch of $\tau_i$ independent samples satisfying \eqref{eq:general_second_mom_upp_bound}: 
\begin{equation*}
	\textstyle \hgit \eqdef \fr{1}{\tau_i} \sum_{j=1}^{\tau_i}\nabla f_{\xi_{i j}^{t}}(\xt).
\end{equation*}
}
Below we provide two examples of stochastic gradients fitting this assumption (for more detail, see \citep{khaled2020better}).
\begin{example}\label{ex:UBV_case}
	Consider $\nabla f_{\xi_{i j}^{t}}(\xt)$ such that
	\begin{eqnarray*}
		&&\Exp{\nabla f_{\xi_{i j}^{t}}(\xt)\mid \xt} = \nabla f_{i}(\xt) \quad \text{and}  \qquad   \Exp{\sqnorm{\nabla f_{\xi_{i j}^{t}}(\xt) - \nabla f_{i}(\xt)}\mid \xt} \leq \sigma_i^2
	\end{eqnarray*}		
	for some $\sigma_i \geq 0$. Then, due to variance decomposition,\eqref{eq:general_second_mom_upp_bound} holds with $A_i = 0$, $B_i = 0$, $C_i = \sigma_i^2$.
\end{example}

\begin{example}\label{ex:subsampling_case}
	Let $f_i(x) = \frac{1}{m}\sum_{j=1}^{m} f_{ij}(x)$, $f_{ij}$ be $L_{ij}$-smooth and $f_{ij}^{\inf} = \inf_{x\in\R^d}f_{ij}(x) > -\infty$. Following \citet{gower2019sgd}, we consider a stochastic reformulation
	\begin{eqnarray}
		\textstyle	f_i(x) &=& \Expu{v_i\sim\cD_i}{f_{v_i}(x)} = \Expu{v_i\sim\cD_i}{\frac{1}{m}\sum\limits_{j=1}^{m}f_{v_{ij}}(x)}, \notag 
	\end{eqnarray}
	where $\Expu{v_i\sim\cD_i}{v_{ij}} = 1$. One can show (see Proposition 2 of \citet{khaled2020better}) that under the assumption that $\Expu{v_i\sim\cD_i}{v_{ij}^2}$ is finite for all $j$ stochastic gradient $\nabla f_{\xi_{i j}^{t}}(\xt) = \nabla f_{v_i^t}(\xt)$ with $v_i^t$ sampled from $\cD_i$ satisfies \eqref{eq:general_second_mom_upp_bound} with $A_i = \max_j L_{ij} \Expu{v_i\sim\cD_i}{v_{ij}^2}$, $B_i = 1$, $C_i = 2A_i\Delta_i^{\inf}$, where $\Delta_i^{\inf} = \frac{1}{m}\sum_{j=1}^{m} (f_i^{\inf} - f_{ij}^{\inf})$. In particular, if $\Prob(\nabla f_{\xi_{i j}^{t}}(\xt) = \nabla f_{ij}(\xt)) = \frac{L_{\blue{i}j}}{\sum_{l=1}^{m}L_{il}}$, then $A_i = \overline{L}_i = \frac{1}{m}\sum_{j=1}^{m}L_{ij}$, $B_i = 1$, and $C_i = 2A_i\Delta_i^{\inf}$.
\end{example}

\subsubsection{Additional assumptions for \algname{EF21-PAGE}}

In the analysis of  \algname{EF21-PAGE}, we rely on the following assumption.

\begin{assumption}[Average $\cL$-smoothness]\label{as:avg_smoothness_page}
	Let every $f_i$ have the form \eqref{eq:finite_sum_case}. Assume that for all $t\ge 0$, $i = 1, \dots, n$, and batch $I_{i}^{t}$ (of size $\tau_i$), the minibatch stochastic gradients difference $\widetilde{\Delta}_i^t \eqdef \frac{1}{\tau_i}\sum_{j\in I_{i}^{t}}(\nabla f_{ij}(x^{t+1}) - \nabla f_{ij}(x^t))$ computed on the node $i$, satisfies $\Exp{\widetilde{\Delta}_i^t\mid x^t,x^{t+1}} = \Delta_i^t$ and
	\begin{equation}
		\textstyle	\Exp{\left\|\widetilde{\Delta}_i^t - \Delta_i^t\right\|^2\mid x^t, x^{t+1}} \le \frac{\cL_i^2}{\tau_i}\|x^{t+1}-x^t\|^2\label{eq:avg_smoothness_page}
	\end{equation}
	with some $\cL_i \ge 0$, where $\Delta_i^t \eqdef \nabla f_i(x^{t+1}) - \nabla f_i(x^t)$. We also define $\wcL^2 \eqdef \suminn \fr{(1-p_i)\mathcal{L}_{i}^2}{\tau_i}$.
\end{assumption}

This assumption is satisfied for many standard/popular sampling strategies. For example, if $I_{i}^{t}$ is a full batch, then $\cL_i = 0$. Another example is \textit{uniform sampling} on $\cb{1, \dots, m}$, and each $f_{i j }$ is $L_{i j }$-smooth. In this regime, one may verify that $\cL_i \leq \max_{1\leq j \leq m} L_{i j }$. 

\subsubsection{Additional assumptions for global convergence}

We now introduce an additional assumption, which enables us to obtain (global) convergence results for the function value.
\begin{assumption}[Polyak-{\L}ojasiewicz] \label{ass:PL} There exists $\mu>0$ such that $f(x) - f(x^\star) \leq  \frac{1}{2 \mu}\sqnorm{\nabla f(x)}$ for all $x\in \R^d$, where $x^\star  \in \arg\min_{x\in\R^d} f(x) \neq \emptyset$.
\end{assumption}
The results under this assumption (and its generalization to composite case) are briefly summarized in Table~\ref{tab:comparison} in the column "Compl. (P{\L})". The detailed statements of the results are deferred to the extended version of this work \citep{EF21BW_2021}. 

\subsection{Main results}

Below, we formulate the corollary establishing the complexities for each method. The complete version of this result is formulated and rigorously derived for each method in the appendix. {\color{myblue} We also include the proof sketch for each result in the corresponding sections.}

\begin{corollary}\label{cor:monster_corollary}
	Suppose that Assumption~\ref{as:main} holds. Then, there exist appropriate choices of parameters for \algname{EF21-PP}, \algname{EF21-BC}, \algname{EF21-HB}, \algname{EF21-Prox} such that the number of communication rounds $T$ and the (expected) number of gradient computations at each node $\# \text{grad}$ for these methods to find an $\varepsilon$-stationary point, i.e., a point $\hat{x}^{T}$ such that $\mathbb{E}[\|\nabla f(\hat{x}^{T})\|^2] \leq \varepsilon^2$ for \algname{EF21-PP}, \algname{EF21-BC}, \algname{EF21-HB} and $\mathbb{E}[\|\mathcal{G}_{\g}(\hat{x}^{T})\|^2] \leq \varepsilon^2$ for \algname{EF21-Prox}, where $\mathcal{G}_{\g}(x)= \nicefrac{1}{\g} \rb{ x-\opn{prox}_{\g r}(x-\g \nabla f(x)) }$, are 
	\begin{eqnarray*}
		\textstyle	\text{\algname{EF21-PP}:}& & \textstyle T = \cO\rb{    \fr{   \wL \delta^0}{ p \alpha \varepsilon^2} },\quad \# \text{grad} =  \cO\rb{    \fr{   \wL \delta^0}{\alpha \varepsilon^2} }\\
		\textstyle	\text{\algname{EF21-BC}:}& & \textstyle T = \# \text{grad} =   \cO\rb{    \fr{   \wL \delta^0}{ \alpha_w \alpha_M \varepsilon^2} }\\
		\textstyle	\text{\algname{EF21-HB}:}& & \textstyle T = \# \text{grad} =   \cO\rb{    \fr{   \wL \delta^0}{  \varepsilon^2}  \rb{\fr{1}{\al} + \fr{1}{1-\eta}}}\\
		\textstyle		\text{\algname{EF21-Prox}:}& & \textstyle T = \# \text{grad} =   \cO\rb{    \fr{   \wL \delta^0}{ \alpha \varepsilon^2} },
	\end{eqnarray*}
	where $\wL\eqdef \sqrt{\frac{1}{n}\sum_{i=1}^n L_i^2}$, $\delta_0\eqdef f(x^0) - \finf$ (for \algname{EF21-Prox} $\delta^0 = \Phi(x^0) - \Phi^{inf}$), $p$ is the probability of sampling the client in \algname{EF21-PP}, $\alpha_w$ and $\alpha_M$ are contraction factors for compressors applied on the workers' and the master's sides respectively in \algname{EF21-BC}, and $\eta \in [0,1)$ is the momentum parameter in \algname{EF21-HB}. 
	
	If Assumptions~\ref{as:main}~and~\ref{as:general_as_for_stoch_gradients} in the setup from Example~\ref{ex:UBV_case} hold, then there exist appropriate choices of parameters for \algname{EF21-SGD} such that the corresponding $T$ and the averaged number of gradient computations at each node $\overline{\#\text{grad}}$ are
	\begin{eqnarray*}
		\text{\algname{EF21-SGD}:}& &\textstyle  T = \cO\left(\frac{\wL\delta^0}{\alpha\varepsilon^2}\right), \qquad  \overline{\#\text{grad}} =  \cO\left(\frac{\wL\delta^0}{\alpha\varepsilon^2} + \frac{\wL\delta^0\sigma^2}{\alpha^3\varepsilon^4} \right),
	\end{eqnarray*}
	where $\sigma = \frac{1}{n}\sum_{i=1}^n\sigma_i^2$. 
	
	If Assumptions~\ref{as:main}~and~\ref{as:avg_smoothness_page} hold, then there exist appropriate choices of parameters for \algname{EF21-PAGE} such that the corresponding $T$ and $\overline{\#\text{grad}}$ are
	\begin{eqnarray*}
		\text{\algname{EF21-PAGE}:} & & T =  \textstyle \cO\rb{\fr{(\wL+\widetilde{\cL})\delta^0}{\alpha\varepsilon^2}  + \fr{\sqrt{m}\widetilde{\cL}\delta^0}{\varepsilon^2} }, \qquad  \overline{\#\text{grad}} = \textstyle  \cO\rb{m + \fr{(\wL+\widetilde{\cL})\delta^0}{\alpha\varepsilon^2}  + \fr{\sqrt{m}\widetilde{\cL}\delta^0}{\varepsilon^2} },
	\end{eqnarray*}
	where $\widetilde{\cL} = \sqrt{\frac{1-p}{n}\sum_{i=1}^n \cL_i^2}$, $\tau_i\equiv\tau =1$.
\end{corollary}

\topic{Discussion of results and comparison to prior work}


$\bullet$ For \algname{EF21-PP} and \algname{EF21-Prox}, none of previous error feedback methods work on these two settings (partial participation and proximal/composite case). Thus, we provide the \emph{first} convergence results for them. Moreover, we show that the gradient (computation) complexity for both \algname{EF21-PP} and \algname{EF21-Prox} is $\cO(\nicefrac{1}{\alpha\varepsilon})$, matching the original vanilla \algname{EF21}. It means that we extend \algname{EF21} to both settings for free.

$\bullet$ For \algname{EF21-BC}, we show $\cO(\nicefrac{1}{\alpha_w\alpha_M\varepsilon^2})$ complexity result, which naturally extends $\cO(\nicefrac{1}{\alpha_w \varepsilon^2})$  complexity to the case when compression is also applied by server. The most related method, which applies a biased compression in both directions, is  \algname{DoubleSqueeze} of \citet{DoubleSqueeze}. This algorithm acheives only $\cO(\nicefrac{\Delta}{\varepsilon^3})$, moreover,  it uses a strong assumption on the compressors ($\Exp{\|\cC(x) - x\|} \leq \Delta$), which is not satisfied for practically interesing examples such as Top-$K$. Therefore, our analysis improves upon the previous result by acheving better rates and using a more flexible class of compressors.

$\bullet$ Our iteration complexity for \algname{EF21-HB} is of order $\cO(\nicefrac{1}{\varepsilon^2})$. In contrast, the previous result of \algname{M-CSER} is $\cO(\nicefrac{G}{\varepsilon^3})$ and its analysis requires an additional bounded gradient assumption. Moreover, we improve the dependence on momentum and contraction parameters by splitting the product of $(1-\eta)^{-1}$ and $\al^{-1}$ into the sum. 

$\bullet$ For \algname{EF21-SGD} and \algname{EF21-PAGE}, we want to reduce the gradient complexity by using (variance-reduced) stochastic gradients instead of full gradient in the vanilla \algname{EF21}. Note that $\sigma^2$ and $\Delta^{\inf}$ in \algname{EF21-SGD} could be much smaller than $G$ in \algname{Choco-SGD}, while $\sigma^2$ and $\Delta^{\inf}$ are often dimension-free parameters (particularly, they are very small if the functions/data samples are similar). Thus, for high dimensional problems (e.g., deep neural networks), \algname{EF21-SGD} can be better than \algname{Choco-SGD}. Besides, in the finite-sum case \eqref{eq:finite_sum_case}, especially if the number of data samples $m$ on each client is not very large, then \algname{EF21-PAGE} is much better since its sample complexity is  $\cO(\nicefrac{\sqrt{m}}{\varepsilon^2})$ while for \algname{EF21-SGD} it is of order $\cO(\nicefrac{\sigma^2}{\varepsilon^4})$. We can also observe that the sample complexities of \algname{EF21-SGD} and \algname{EF21-PAGE} do not have the linear speedup in the number of nodes (i.e., the devision by $n$) as it is present in (distributed) \algname{SGD} \citep{khaled2020better}. In Appendix~\ref{subsec:tightness_n} we experimentally verify the tightness of our rates w.r.t. $n$.  

 {\color{myblue}
\subsection{Proof sketches}
In this section we provide insights into convergence analysis for several of our extensions: \algname{EF21-PP}, \algname{EF21-PAGE} and \algname{EF21-HB}. Proof sketches for \algname{EF21-SGD}, \algname{EF21-BC} and \algname{EF21-Prox} are deferred to Appendix and can be found in the corresponding sections. 

\paragraph{Partial participation.} The idea of our analysis of \algname{EF21-PP} is to develop a recursion on the error term $G_i^{t+1} \eqdef \sqnorm{g_i^{t+1} - \nabla f_i(x^{t+1})}$ analogous to that in original EF21 analysis (see Lemma~\ref{lem:theta-beta}). We do this by conditioning on the events $\mathcal A_i$ that node $i$ is sampled to participate in communication round, i.e., $i \in S_t$. That is, we consider the following two terms:
	$$
	\Exp{G_i^{t+1} \mid i \in S_t} \qquad \text{and} \quad \Exp{G_i^{t+1} \mid i \notin S_t} .
	$$
 The strategy for controlling these two terms is different. In the first case, when node $i$ participates in training, progress is made toward improving the accuracy of the $g_i^{t+1}$  estimator. In the second case, an additional cost arises because node $i$ skips the communication round.
	If we can bound each term above efficiently, we can continue by computing the full (unconditional) expectation using conditional expectations derived from events $\mathcal{A}_i$ and its complement.
	$$
	\Exp{G_i^{t+1}} =  p_i \, \Exp{G_i^{t+1}\mid i \in S_t} + (1-p_i)\, \Exp{G_i^{t+1}  \mid i \notin S_t} , 
	$$
	 Finally, combining the established recursion on the expected error term, $\Exp{G_i^{t+1}}$, with the standard descent lemma and performing a careful calculation of the final communication and iteration complexities allows us to establish convergence guarantees. The full proof is deferred to Appendix~\ref{sec:partial_participation}.
	
	\paragraph{Variance reduction.}
		The key strategy for analyzing \algname{EF21-PAGE} involves splitting the error into two parts,
		$$
		\sqnorm{\nfixt - g_i^{t}} \leq 2  \sqnorm{ \nabla f_i(\xt) - v_i^{t} } + 2 \sqnorm{v_i^{t} - g_i^t} ,
		$$
		and bounding each term separately. The first term corresponds to an error due to variance reduction, and the second term is related to the EF21 mechanism with the compressor. The strategy of controlling the first term (see Lemma~\ref{le:ineq_3_ef21_page-dist}) is similar to the analysis of error deviation of \algname{PAGE} estimator in \citep{PAGE2021}. However, controlling the second term (see Lemma~\ref{le:ineq_2_ef21_page-dist}) is more involved due to the interplay between the two errors. Indeed, while both sequences $v_i^{t}$ and $g_i^t$ change dynamically, the key challenge is to efficiently control the accumulated error from both and build up a recursion of type 
		$$
		\Exp{\sqnorm{v_i^{t+1} - g_i^{t+1}}} \leq (1-\theta) 		\Exp{\sqnorm{v_i^{t} - g_i^{t}}} + C_1 \Exp{\sqnorm{\nabla f_i(x^t) - v_i^t}} + C_2 \Exp{\sqnorm{x^{t+1} - x^t}} ,
		$$
		where $\theta \in (0,1)$ is a contraction factor, and $C_1, C_2 > 0$ are constants determined by problem structure and algorithm's parameters. Once this recursion is established, it is combined with a similar recursion for $\Exp{\sqnorm{\nabla f_i(x^t) - v_i^t}}$ and descent lemma (see Lemma~\ref{le:aux_smooth_lemma}), which results in the following Lyapunov function
		$$
		f(x^t) - \finf + \frac{\gamma }{\theta n } \sumin \sqnorm{ \nabla f_i(\xt) - v_i^{t} } + \frac{C_3}{n} \sumin \sqnorm{v_i^{t} - g_i^t}  ,
		$$
		where $\gamma > 0$ is step-size and $C_3 >0$. See Appendix~\ref{sec:variance_reduction} for more details. 
		
	\paragraph{Heavy ball momentum. } The key idea of the convergence analysis of \algname{EF21-HB} is in line with \citep{Unified_momentum,liu2020improved}, where an additional virtual sequence $\cb{z^t}_{t\geq 0}$ is defined as
		$$
		z^{t+1} = x^{t+1} - \fr{\g \eta }{1-\eta} v^t , 
		$$
		where $\gamma > 0$ is step-size and $\eta \in [0, 1)$ is momentum parameter of Algorithm~\ref{alg:EF21_HB}. The main challenge is to control the error term introduced by the EF21 mechanism with a contractive compressor, while accounting for the momentum step, which replaces the simple gradient descent step used in the original \algname{EF21}. The error term due to compression is controlled as in the \algname{EF21} analysis by showing
		\begin{equation}\label{eq:contraction_main}
		 \Exp{ \sqnorm{\nabla f_i(x^{t+1}) - g_i^{t+1}}  } \leq (1-\theta)  \Exp{ \sqnorm{\nabla f_i(x^{t}) - g_i^t}  } + \beta L_i^2  \Exp{ \sqnorm{\xtpo - \xt}  },
		\end{equation}
		where $\theta \in (0,1)$ represents a contraction factor, and $\beta > 0$ depends on contraction factor $\alpha$. However, the Lyapunov function used in \algname{EF21-HB} differs from that of \algname{EF21}:
		$$
		f(z^t) - \finf +  \frac{C}{n} \sum_{i=1}^n \sqnorm{g_i^t - \nabla f_i(x^t)} , 
		$$	
		 where $C>0$, since a virtual sequence $z^t$ appears in function value in the first term instead of $x^t$. This difference causes a technical difficulty in controlling the last term in \eqref{eq:contraction_main} since it is different from $\sqnorm{z^{t+1}- z^t}$ involved in the descent type lemma for \algname{EF21-HB}. We overcome this challenge by relating these two terms after summation as
		 $$
		 \sum_{t=0}^{T-1} \Exp{\sqnorm{x^{t+1} - x^{t}}} \leq 2 (1 + 4\eta^2) \sum_{t=0}^{T-1} \Exp{\sqnorm{z^{t+1} - z^{t}}}.
		 $$ 
		 We refer to Lemma~\ref{le:rec_HB} in Appendix~\ref{sec:HB} for a rigorous proof. 
	}

\section{Experiments}\label{sec:exp}

In this section, we consider a logistic regression problem with a non-convex regularizer, i.e., $f(x) = \frac{1}{N} \sum\limits_{i=1}^{N} \log \left(1+\exp \left(-b_{i}a_i^{\top} x\right)\right) + \lam \sum \limits_{j=1}^d \frac{x_j^2}{1 + x_j^2}$, 
{\small
}where $a_{i}  \in \mathbb{R}^{d}, b_{i} \in\{-1,1\}$ are the training data, and $\lambda>0$ is the regularization parameter, which is set to $\lambda =0.1$ in all experiments.
We use~$n = 20$ for experiments $1, 3$ and $n=100$ for experiment $2$, and split datapoints heterogeneously.
In all algorithms involving compression, we use Top-$k$ \citep{alistarh2017qsgd} as a canonical example of contractive compressor $\cC$, and fix the compression ratio $\nfr{k}{d} \approx 0.01$, where $d$ is the number of features in the data set. For all algorithms, at each iteration we compute the squared norm of the exact/full gradient for comparison of the methods performance. We terminate our algorithms either if they reach the certain number of iterations or the following stopping criterion is satisfied:  $\sqnorm{\nf{\xt}} \le 10^{-7}$. 
We tune the step-sizes for each method individually and report the best one based on the minimal number of bits required to acheive the desired accuracy. We refer the reader to Appendix~\ref{sec:exp_extra} for more detailed experimental setup, and additional experiments, including other proposed methods such as \algname{EF21-HB} and \algname{EF21-BC}.\footnote{Implementation of all our algorithms is publicly available at \href{https://github.com/IgorSokoloff/ef21_b-w_experiements_source_code}{https://github.com/IgorSokoloff/ef21\_b-w\_experiements\_source\_code}.} The main goal of the following numerical experiments is to illustrate our key theoretical findings. This way we further motivate the proposed algorithmic enhancements of \algname{EF21}.

\paragraph{Experiment 1: Fast convergence with variance reduction.}\label{subsecexp:ef21_page-ef21_sgd}
In our first experiment, we showcase the computation and communication benefit of \algname{EF21-PAGE} (Alg. \ref{alg:EF21-PAGE}) over \algname{EF21-SGD}. Figure \ref{fig:ef21_page_relasim} illustrates that, in all cases, \algname{EF21-PAGE} perfectly reduces the accumulated variance and converges to the desired tolerance, whereas \algname{EF21-SGD} is stuck at some accuracy level. Moreover, \algname{EF21-PAGE} turns out to be surprisingly efficient with small batchsizes (eg, $1.5\%$ of the local data ) both in terms of the number of epochs and the \# bits sent to the server per client. Interestingly, for most data sets, a further increase of batchsize does not considerably improve the convergence.

\begin{figure}[H]
	\begin{subfigure}{0.5\textwidth}
		\centering
		\includegraphics[width=\linewidth]{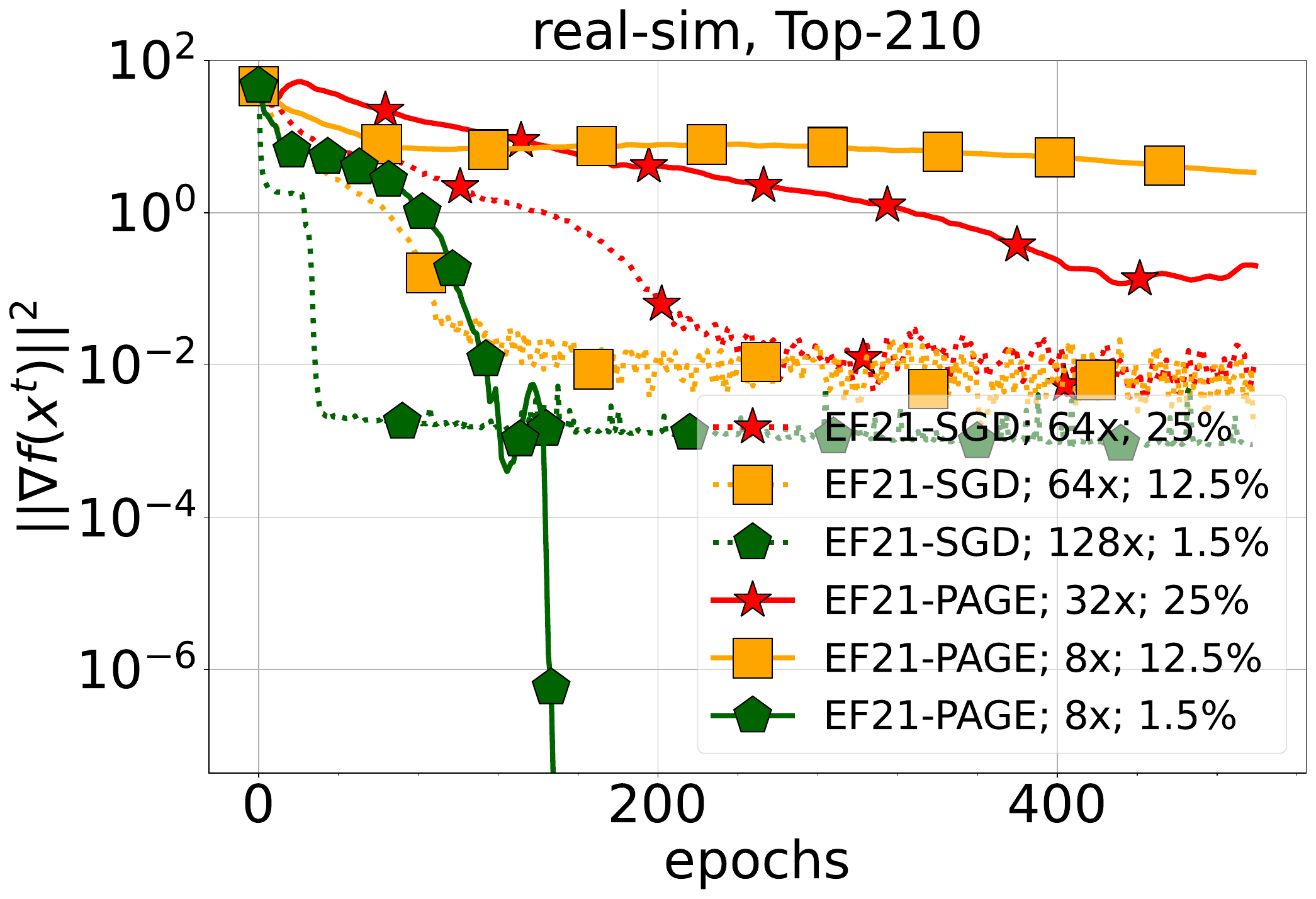}
		\caption{Convergence in epochs.\\${}$}
	\end{subfigure}
	\begin{subfigure}{0.5\textwidth}
		\includegraphics[width=\linewidth]{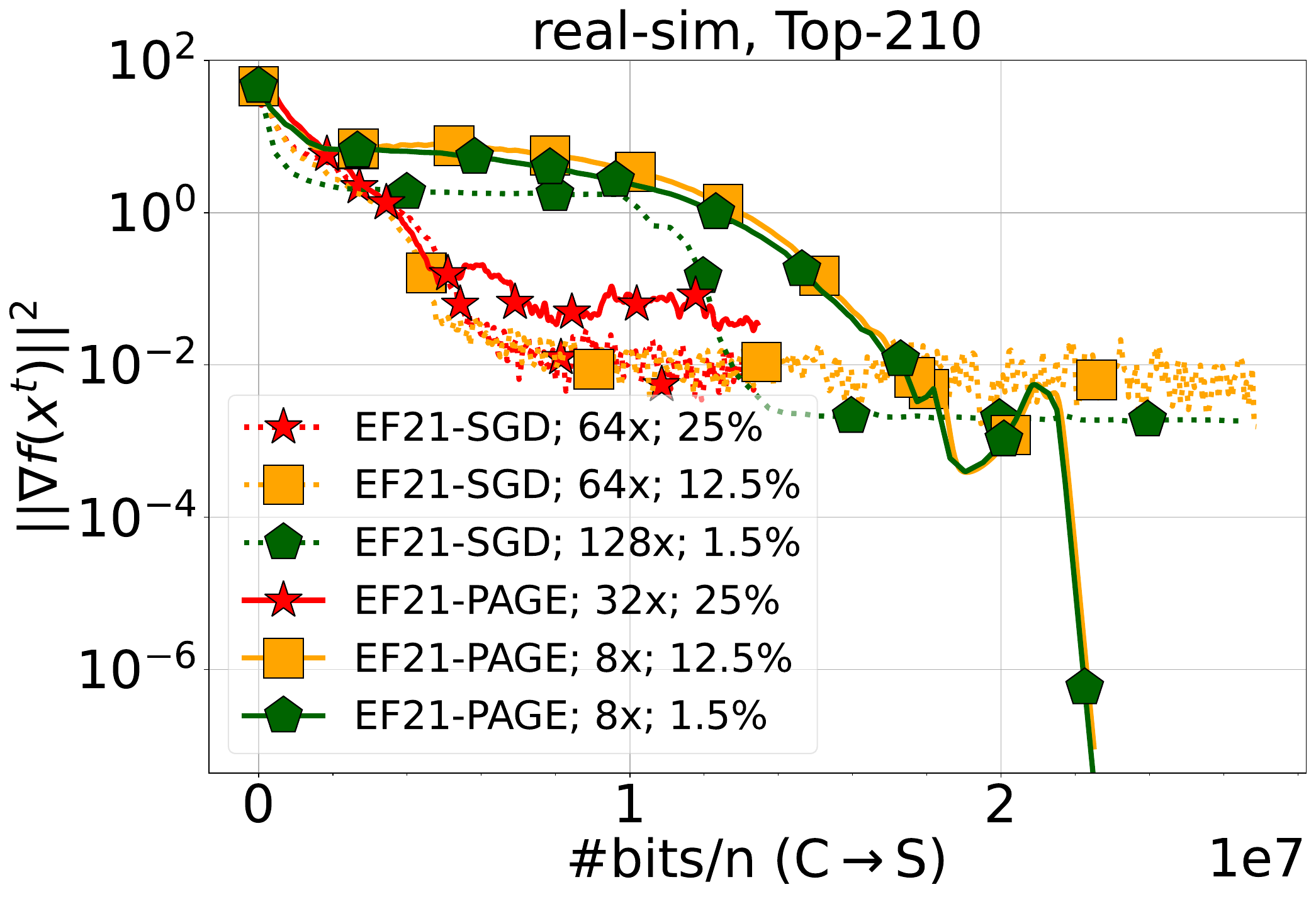} 
		\caption{Convergence in terms of  total number of bits sent from \textbf{C}lients to the \textbf{S}erver divided by $n$.}
	\end{subfigure}
	\caption{
		{\color{myblue} Comparison of \algname{EF21-PAGE} and \algname{EF21-SGD} with tuned parameters. By $1\times, 2\times, 4\times$ (and so on) we indicate that the stepsize was set to a multiple of  the largest stepsize predicted by theory for \algname{EF21}. By $25\%$, $12.5\% $ and $1.5\% $ we refer to batchsizes equal $\lfloor0.25 N_i\rfloor$, $\lfloor0.125N_i\rfloor$ and $\lfloor0.015N_i\rfloor$ for all clients $i=1,\dots,n$, where $N_i$ denotes the size of local data set.}}\label{fig:ef21_page_relasim}%
\end{figure}

\paragraph{Experiment 2: On the effect of partial participation of clients.}\label{subsecexp:ef21_pp-ef21_fg}
This experiment shows that \algname{EF21-PP} (Alg.~\ref{alg:PP-EF21}) has potential to reduce communication cost. For this comparison, we consider $n=100$, and apply a different data partitioning, see Table~\ref{tbl:datasets_summary_2} in Appendix~\ref{sec:exp_extra} for more details. It is predicted by our theory (Corollary~\ref{cor:monster_corollary}) that, in terms of the number of iterations/communication rounds, \textit{partial participation} slows down the convergence of \algname{EF21} by a fraction of participating clients. 
However, since for \algname{EF21-PP} the communications are considerably cheaper it is able to outperform \algname{EF21} in terms of the number of bits sent to the server per client on average (see Figure \ref{fig:ef21_pp_bits_realsim}).

\begin{figure}[H]
	\begin{subfigure}{0.5\textwidth}
		\centering
		\includegraphics[width=\linewidth]{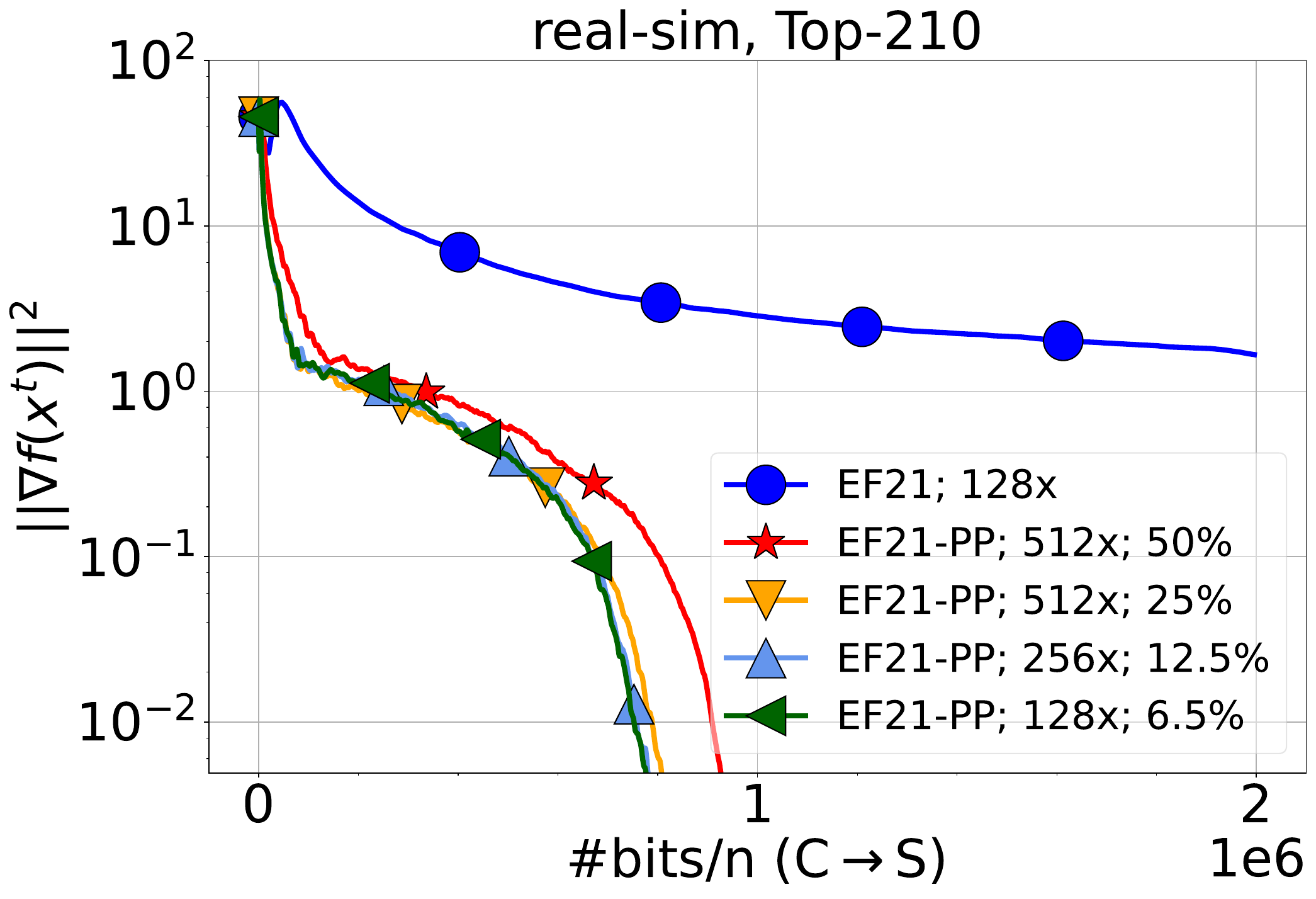}
		\caption{Convergence in terms of  total number of bits sent from \textbf{C}lients to the \textbf{S}erver divided by $n$.}
		\label{fig:ef21_pp_bits_realsim}
	\end{subfigure}
	\begin{subfigure}{0.5\textwidth}
		\includegraphics[width=\linewidth]{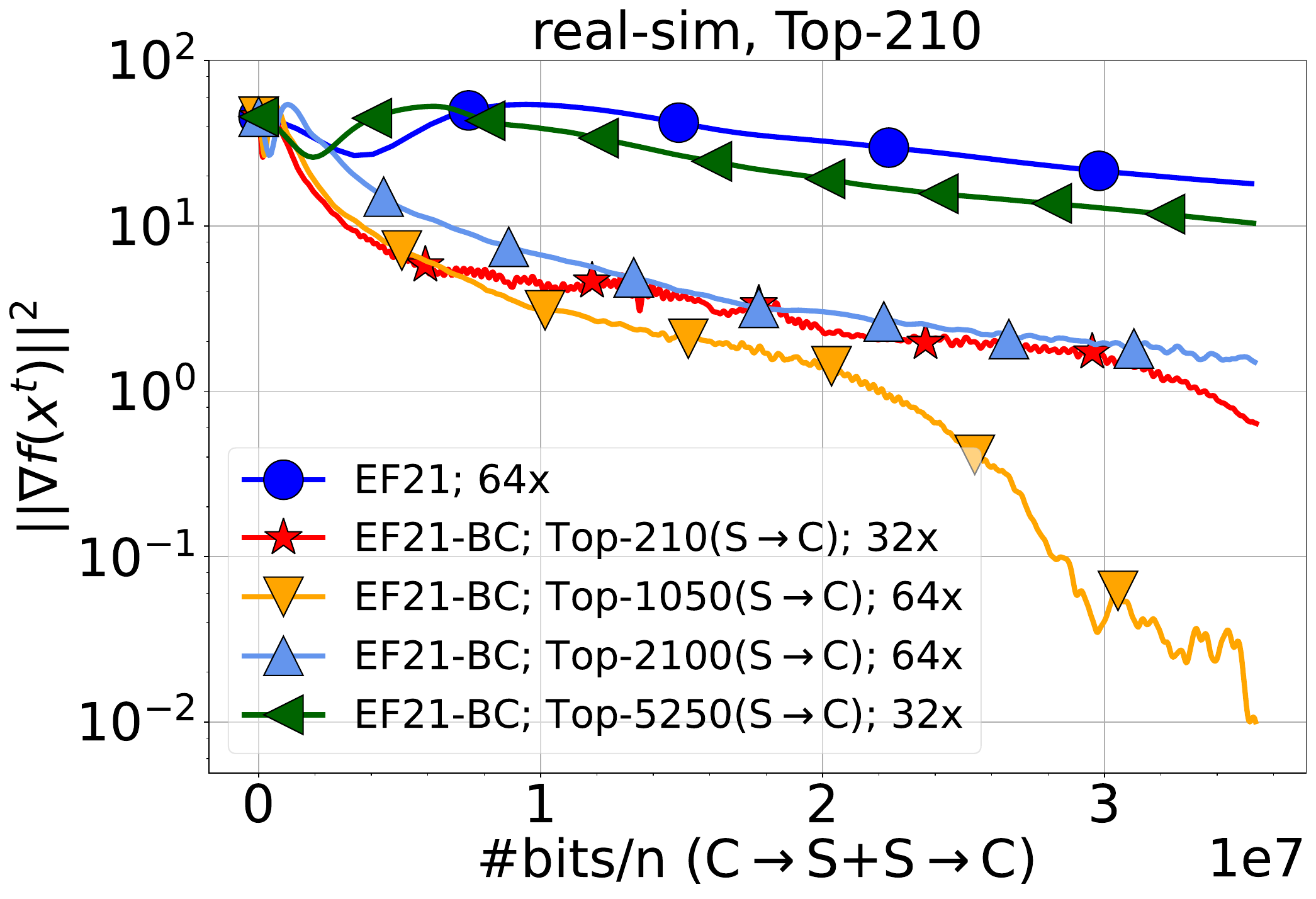}
		\caption{Convergence measured by (bits sent from \textbf{C}lients to the \textbf{S}erver + bits from \textbf{S}erver to \textbf{C}lients) $ / n$.}
		\label{fig:ef21_bc_bd_bits_realsim}
	\end{subfigure}
	\caption{
		{\color{myblue} Comparison of \algname{EF21}, \algname{EF21-PP} and \algname{EF21-BC} with tuned parameters. By $1\times, 2\times, 4\times$ (and so on) we indicate that the stepsize was set to a multiple of  the largest stepsize predicted by theory for \algname{EF21}.} }\label{fig:ef21_page_relasim_main}%
\end{figure}

{\color{myblue}
\paragraph{Experiment 3: On the advantages of bidirectional biased compression.}\label{subsecexp:ef21_bc-ef21_fg_main}
Our next experiment demonstrates that the application of the \textbf{S}erver $\rightarrow$ \textbf{C}lients  compression in \algname{EF21-BC} (Alg.~\ref{alg:EF21-BC}) improves convergence in terms of bits to be transmitted from \textbf{S}erver $\rightarrow$ \textbf{C}lients and \textbf{C}lients $\rightarrow$ \textbf{S}erver together. Indeed, Figure \ref {fig:ef21_bc_bd_bits_realsim} illustrates that \algname{EF21-BC} outperforms \algname{EF21} in terms of the total number of bits communicated even when communicating only $5\% - 15\%$ of data.\footnote{The range $5\% - 15\%$  comes from the fractions $\nicefrac{k}{d}$ for each data set and this observation is consistent accross several data sets.} Note that \algname{EF21} communicates full vectors from the \textbf{S}erver $\rightarrow$ \textbf{C}lients, which slows down communication at each round. We refer to Appendix~\ref{sec:exp_extra} for more ablation studies.

\section*{Conclusion}

This work extends the capabilities of \algname{EF21} by introducing six practical enhancements: partial participation, stochastic approximation, variance reduction, proximal settings, momentum, and bidirectional compression. These extensions address key limitations of earlier error feedback methods, offering improved theoretical guarantees and practical performance.
While our results highlight significant progress, further work is needed to explore their full potential in real-world scenarios, such as federated learning in highly heterogeneous regimes. We hope these contributions will inspire continued advancements in communication-efficient optimization.

}

\acks{The authors would like to thank the anonymous reviewers and the handling editor for their constructive feedback and suggestions, which helped improve the quality and clarity of this paper. This work was supported by funding from King Abdullah University of Science and Technology (KAUST) Baseline Research Scheme. I. Fatkhullin is partially funded by ETH AI Center Doctoral Fellowship. The work of E. Gorbunov was partially supported by a grant for research centers in the field of artificial intelligence, provided by the Analytical Center for the Government of the Russian Federation in accordance with the subsidy agreement (agreement identifier 000000D730321P5Q0002) and the agreement with the Moscow Institute of Physics and Technology dated November 1, 2021 No. 70-2021-00138.  }

\newpage
\appendix
\onecolumn

\renewcommand{\contentsname}{Table of Contents}
{\setlength{\parskip}{0.15em}\tableofcontents}

\newpage

\section{Tables with Notations and Methods}\label{sec:notations_methods}


Table~\ref{tab:notation_table} summarizes the most frequently used notations in our analysis. Additionally, we comment on the main quantities here. Following~\citet{EF21}, we denote {\color{myblue} the deviation of EF21 estimator $g_i^t$ from the local gradient} by $G_i^t \eqdef \sqnorm{\nfixt - g_i^t} $, and by $G^t \eqdef \suminn G_i^t$, {\color{myblue} the average of this quantity over multiple nodes. This notation is common for analysis of all algorithms in this work.  In Section~\ref{sec:variance_reduction} for \algname{EF21-PAGE}, it is useful to split the deviation $G_i^t$ further and define the corresponding deviations of variance reduced estimator $v_i^t$ from exact local gradient $\nfixt $, $P_i^t \eqdef \sqnorm{\nfixt - v_i^t}$, and $v_i^t$ from EF21 estimator $g_i^t$,  $V_i^t \eqdef \sqnorm{v_i^t - g_i^t}$. Similarly, in Section~\ref{sec:BC} for \algname{EF21-BC}, it is helpful to consider additionally the deviation between EF21 estimator on the clients and the exact gradient $P_i^t \eqdef \sqnorm{\wg_i^t - \nfixt}$ and the deviation between EF21 estimator on the server and the average of EF21 estimators on the clients $V^t \eqdef \sqnorm{\suminn \wg_i^t - g^t}$. }

For analysis of most algorithms, we define $\delta^{t} \eqdef f(x^{t})-\finf $,\footnote{If, additionally, Assumption~\ref{ass:PL} holds, then $\finf$ can be replaced by $f(x^\star)$ for  $x^\star \in \arg\min_{x\in\R^d} f(x) \neq \emptyset$.}  $R^{t} \eqdef \sqnorm{x^{t+1}-x^{t}} $. In the analysis of \algname{EF21-HB}, it is useful to modify this notation to  {\color{myblue} $\delta^{t} \eqdef f(z^{t})-\finf $ and } $R^t \eqdef (1-\eta)^2 \sqnorm{z^{t+1} -  z^t }$, where $\{z^t\}_{t\geq 0}$ is the sequence of virtual iterates introduced in Section~\ref{sec:HB}.

\begin{table}[t]
	\centering
	\caption{Summary of frequently used notations in the proofs.}\label{tab:notation_table}
	\begin{tabular}{|c|l|}
		\hline
		\textbf{Algorithm}   & \textbf{Notation}  \\   
		\hline
		\hline
		\makecell{{\color{myblue} for all} \\ {\color{myblue} algorithms}}   & {\color{myblue} $G_i^t =  \sqnorm{g_i^t - \nfixt}$, $G^t = \suminn G_i^t $}  \\   
		\hline
		\algname{EF21}      & \multirow{3}{*}{$R^t = \sqnorm{\xtpo-\xt}$, $\delta^t = f(x^t) - \finf$} \\ 
		\algname{EF21-SGD}  &\\ 
		\algname{EF21-PP}   &\\ 
		\hline
		\algname{EF21-PAGE} &
		\begin{tabular}[c]{@{}l@{}}$R^t = \sqnorm{\xtpo-\xt}, \delta^t = f(x^t) - \finf$, \\ 
			$P_i^t = \sqnorm{\nfixt - v_i^t}$, $V_i^t = \sqnorm{v_i^t - g_i^t}$, \\
			$P^t = \suminn P_i^t $, $V^t = \suminn V_i^t$ \end{tabular} \\ 
		\hline
		\algname{EF21-BC}   & 
			\begin{tabular}[c]{@{}l@{}} $R^t = \sqnorm{\xtpo-\xt}, \delta^t = f(x^t) - \finf $, \\
				$P_i^t =  \sqnorm{\wg_i^t - \nfixt}$, $P^t = \suminn P_i^t $ 
			\end{tabular} \\  
		\hline
		\algname{EF21-HB}   & $R^t = (1-\eta)^2 \sqnorm{z^{t+1} -  z^t }, \delta^t = f({\color{myblue} z^t}) - \finf$ \\ 
		\hline
		\algname{EF21-Prox} &
		\begin{tabular}[c]{@{}l@{}} $R^t = \sqnorm{\xtpo-\xt}$, $\Phi(x)  = f(x) + r(x)$, $ \delta^t = \Phi(x^t) - \Phi^{inf}$, \\ 
			$\cG_{\g}(x) = \fr{1}{\g} \rb{x - \opn{prox}_{\g r}(x -\g \nabla f(x)) } $ \end{tabular}\\
		\hline 
	\end{tabular}
\end{table}

\begin{table*}
	\centering
	\caption{  Description of the methods developed and analyzed in the paper. For the ease of comparison, we also provide a description of \algname{EF21}. In all methods only compressed vectors $c_i^t$ are transmitted from workers to the master and the master broadcasts non-compressed iterates $x^{t+1}$ (except \algname{EF21-BC}, where the master broadcasts compressed vector $b^{t+1}$). Initialization of $g_i^0$, $i = 1, \dots, n$ can be arbitrary (possibly randomized). One possible choice is $g_i^0 = \cC(\nabla f_i(x^0))$. The pseudocodes for each method are given in the appendix.}
	\label{tab:methods}    
	\begin{threeparttable}
		\begin{tabular}{|c|c|c  c|}
			\hline
			Method & \algname{EF21-} &  $c_i^t$ & Comment\\ 
			\hline\hline
			\multirow{11}{2.5cm}{\centering $x^{t+1} = x^t - \gamma g^t$,\\ $g^t = \frac{1}{n}\sum\limits_{i=1}^n g_i^t$,\\ $g_i^{t+1} = g_i^{t} + c_i^t$}& \makecell{n/a \\ Alg.~\ref{alg:EF21}}  &	$\cC(\nabla f_i(x^{t+1}) - g_i^t)$ & \\
			\cline{2-4}
			&\makecell{\algname{SGD} \\ Alg.~\ref{alg:EF21-online}} &	$\cC(\hat g_i(x^{t+1}) - g_i^t)$& $\hat g_i(x^{t+1})$ satisfies As.~\ref{as:general_as_for_stoch_gradients}\\
			\cline{2-4}			
			&\makecell{\algname{PAGE} \\ Alg.~\ref{alg:EF21-PAGE} } & 	$\cC(v_i^{t+1} - g_i^t)$& \makecell{$b_i^t \sim \operatorname{Be}(p)$,\\ $v_i^{t+1} = \nabla f_i(x^{t+1})$, if $b_i^t  = 1$,\\ $v_i^{t+1} = v_i^t + \fr{1}{\tau_i} \sum \limits_{j\in I_i^{t}} \nabla f_{ij}(x^{t+1})$\\ $- \fr{1}{\tau_i} \sum \limits_{j\in I_i^{t}} \nabla f_{ij}(\xt)$, if $b_i^t  = 0$,\\ $I_{i}^{t}$ is a minibatch, $|I_i^{t}|=\tau_i$} \\
			\cline{2-4}			
			&\makecell{\algname{PP}\\ Alg.~\ref{alg:PP-EF21}}  & \makecell{$\cC(\nabla f_i(x^{t+1}) - g_i^t)$\\ $0$}& \makecell{if $i \in S_t$\\if $i \not\in S_t$} \\
			\hline\hline
			\makecell{$x^{t+1} = x^t - \gamma g^t$,\\ $\gtpo = \gt + b^{t+1}$,\\ $b^{t+1}= \cC_M(\wgtpo - \gt)$, \\ $\wgtpo = \frac{1}{n} \sum_{i=1}^n  \wg_i^{t+1}$,\\ $\wg_i^{t+1} = \wg_i^t + c_i^t$} &\makecell{\algname{BC} \\ Alg.~\ref{alg:EF21-BC}}  & $\cC_w(\nabla f_i(x^{t+1}) - \wg_i^t)$ & \makecell{Master broadcasts $b^{t+1}$;\\ $\cC_w$ used by workers,\\ $\cC_M$ used by master}\\
			\hline\hline
			\makecell{$x^{t+1} = x^t - \gamma v^t$,\\ $v^{t+1} = \eta v^t + g^{t+1}$,\\ $g^{t+1} = \frac{1}{n} \sum_{i=1}^n  g_i^{t+1}$, \\ $g_i^{t+1} = g_i^t + c_i^t$} & \makecell{\algname{HB} \\ Alg.~ \ref{alg:EF21_HB} } & $\cC(\nabla f_i(x^{t+1}) - g_i^t)$ & \makecell{$\eta \in [0,1)$ \\is momentum parameter } \\
			\hline\hline
			\makecell{$x^{t+1} = \operatorname{prox}_{\gamma r} \rb{x^t - \gamma g^t}$,\\ $g^{t+1} = \frac{1}{n} \sum_{i=1}^n  g_i^{t+1}$, \\ $g_i^{t+1} = g_i^t + c_i^t$} & \makecell{\algname{Prox} \\ Alg.~\ref{alg:Prox-ef21} } & $\cC(\nabla f_i(x^{t+1}) - g_i^t)$ & \makecell{For problem \eqref{eq:composite_optimization_0};\\ $\opn{prox}_{\g r}(x)$ is defined in \eqref{eq:def_prox}}\\
			\hline
		\end{tabular}
	\end{threeparttable}
\end{table*}

\newpage

\section{EF21}
For completeness, we provide here the pseudocode and the detailed convergence proof for \algname{EF21} \citep{EF21}.

\begin{algorithm}[H]
	\centering
	\caption{\algname{EF21}}\label{alg:EF21}
	\begin{algorithmic}[1]
		\STATE \textbf{Input:} starting point $x^{0} \in \R^d$;  $g_i^0\in \R^d$ for $i=1,\dots, n$ (known by nodes); $g^0 = \frac{1}{n}\sum_{i=1}^n g_i^0$ (known by master); learning rate $\gamma>0$
		\FOR{$t=0,1, 2, \dots , T-1 $}
		\STATE Master computes $x^{t+1} = x^t - \gamma g^t$ and broadcasts $x^{t+1}$ to all nodes
		\FOR{{\bf all nodes $i =1,\dots, n$ in parallel}}
		\STATE Compress $c_i^t = \cC(\nabla f_i(x^{t+1}) - g_i^t)$ and send $c_i^t $ to the master
		\STATE Update local state $g_i^{t+1} = g_i^t + c_i^t $
		\ENDFOR
		\STATE Master computes $g^{t+1} = \frac{1}{n} \sum_{i=1}^n  g_i^{t+1}$ via  $g^{t+1} = g^t + \frac{1}{n} \sum_{i=1}^n c_i^t $
		\ENDFOR
	\end{algorithmic}
\end{algorithm}

\begin{lemma}\label{lem:theta-beta} Let $\cC$ be a contractive compressor, then for all $i = 1, \dots, n$ 
	\begin{equation}\label{eq:rec_1} 
		\Exp{ G_i^{t+1} } \leq (1-\theta) \Exp{  G_i^t } + \beta  L_i^2 \Exp{ \sqnorm{\xtpo - \xt}  }, \text{  and}
	\end{equation}
	\begin{equation}\label{eq:rec_1_avg} \Exp{ G^{t+1} } \leq (1-\theta) \Exp{  G^t } + \beta  \wL^2  \Exp{ \sqnorm{\xtpo - \xt}  },
	\end{equation}
	where $\theta \eqdef 1- (1- \alpha )(1+s), \quad \beta \eqdef (1- \alpha ) \left(1+ s^{-1} \right) \quad \text{for any } s > 0$.
\end{lemma}

\begin{proof}
	Define $W^t \eqdef \{g_1^t, \dots, g_n^t, x^t, x^{t+1}\}$, then 
	\begin{eqnarray}
		\Exp{  G_i^{t+1} } & = & \Exp{ \Exp{ G_i^{t+1} \;|\; W^t} }\notag \\
		& = & \Exp{  \Exp{  \sqnorm{g_i^{t+1} - \nabla f_i(x^{t+1})}  \;|\; W^t} }	\notag  \\
		&=& \Exp{  \Exp{  \sqnorm{g_i^t + \cC ( \nabla f_i(x^{t+1}) - g_i^t) - \nabla f_i(x^{t+1})}  \;|\; W^t}	} \notag \\
		&\overset{\eqref{eq:b_compressor}}{\leq} &  (1-\alpha) \Exp{ \sqnorm{\nabla f_i(x^{t+1}) - g_i^t} }\notag \\
		&\overset{(i)}{\leq} & (1-\alpha)  (1+ s) \Exp{ \sqnorm{\nabla f_i(x^{t}) - g_i^t}  } \notag \\
		&& \qquad  + (1-\alpha)  \left(1+s^{-1}\right) \Exp{ \sqnorm{\nabla f_i(x^{t+1}) - \nabla f_i(x^t)}   } 	\label{eq:beta-theta-without-smoothness}	 \\
		&\overset{(ii)}{\leq} & (1-\alpha)  (1+ s) \Exp{ \sqnorm{\nabla f_i(x^{t}) - g_i^t}  }\notag \\
		&& \qquad  + (1-\alpha)  \left(1+s^{-1}\right) L_i^2  \Exp{ \sqnorm{\xtpo - \xt}  } \notag \\
		&\overset{(iii)}{\leq} & (1-\theta)  \Exp{ \sqnorm{\nabla f_i(x^{t}) - g_i^t}  } + \beta L_i^2  \Exp{ \sqnorm{\xtpo - \xt}  } ,\notag 
	\end{eqnarray}
	where $(i)$ follows by Young's inequality~\eqref{eq:facts:young_ineq2},  $(ii)$ holds by Assumption~\ref{as:main}, and in $(iii)$ we apply the definition of $\theta$ and $\beta$. Averaging the above inequalities over $i = 1, \dots, n$, we obtain \eqref{eq:rec_1_avg}.
\end{proof}


\begin{theorem}\label{thm:main-distrib}
	Let Assumption~\ref{as:main} hold, and let the stepsize in Algorithm~\ref{alg:EF21} be set as
	\begin{equation} \label{eq:manin-nonconvex-stepsize}
		0<\gamma \leq \rb{L + \wL\sqrt{\frac{\beta}{\theta}}}^{-1}.
	\end{equation}
	Fix $T \geq 1$ and let $\hat{x}^{T}$ be chosen from the iterates $x^{0}, x^{1}, \ldots, x^{T-1}$  uniformly at random. Then
	\begin{equation} \label{eq:main-nonconvex}		
		\Exp{\sqnorm{\nabla f(\hat{x}^{T})} } \leq \frac{2\left(f(x^{0})-f^{\text {inf }}\right)}{\gamma T} + \frac{\Exp{G^0}}{\theta T} ,
	\end{equation}		
	where $\wL = \sqrt{\frac{1}{n}\sum_{i=1}^n L_i^2}$, $\theta = 1- (1- \alpha )(1+s)$,  $\beta =  (1- \alpha ) \left(1+ s^{-1} \right)$ for any $s > 0$.
\end{theorem}

\begin{proof}
	According to our notation, for Algorithm~\ref{alg:EF21} $R^t = \sqnorm{\xtpo - \xt}$. By Lemma~\ref{lem:theta-beta}, we have
	
	\begin{eqnarray}\label{eq:main_recursion_distrib}
		\Exp{G^{t+1}} &\leq& \rb{1 - \theta}\Exp{G^t}+ \beta \wL^2 \Exp{ R^t} .
	\end{eqnarray}
	Next, using Lemma~\ref{le:aux_smooth_lemma} and Jensen's inequality \eqref{eq:facts:young_ineq3_avg}, we obtain the bound 
	\begin{eqnarray}
		f(x^{t+1}) &\leq & f(x^{t})-\frac{\g}{2}\sqnorm{\nabla f(x^{t})}-\left(\frac{1}{2 \g}-\frac{L}{2}\right) R^t +\frac{\g}{2}\sqnorm{\suminn \rb{g_i^t -\nfixt}}\notag \\
		&\leq & f(x^{t})-\frac{\g}{2}\sqnorm{\nabla f(x^{t})}-\left(\frac{1}{2 \g}-\frac{L}{2}\right) R^t +\frac{\g}{2} \suminn \sqnorm{ g_i^t-\nfixt }\notag \\
		&{=}&
		f(x^{t})-\frac{\g}{2}\sqnorm{\nabla f(x^{t})}-\left(\frac{1}{2 \g}-\frac{L}{2}\right) R^t +\frac{\g}{2} G^t. \label{eq:aux_smooth_lemma_distrib}
	\end{eqnarray}
	Subtracting $f^{\text {inf }}$ from both sides of the above inequality, taking expectation and using the notation $\delta^t = f(x^{t})-\finf$, we get
	\begin{eqnarray}\label{eq:func_diff_distrib}
		\Exp{\delta^{t+1}} &\leq& \Exp{\delta^t}-\frac{\gamma}{2} \Exp{\sqnorm{\nabla f(x^{t})}}  - \left(\frac{1}{2 \gamma}-\frac{L}{2}\right) \Exp{R^t}+ \frac{\gamma}{2}\Exp{G^t}.
	\end{eqnarray}
	Then by adding \eqref{eq:func_diff_distrib} with a $\frac{\gamma}{2 \theta}$ multiple of \eqref{eq:main_recursion_distrib} we obtain
	\begin{eqnarray*}
		\Exp{ \delta^{t+1}}+\frac{\gamma}{2 \theta} \Exp{G^{t+1} } &\leq& \Exp{\delta^{t}}-\frac{\gamma}{2} \Exp{\sqnorm{\nabla f(x^{t})} } - \left(\frac{1}{2 \gamma}-\frac{L}{2}\right) \Exp{ R^{t} }+\frac{\gamma}{2} \Exp{G^{t}} \\
		&& \qquad +\frac{\gamma}{2 \theta}\rb{ \beta \wL^2 \Exp{R^t} + (1 - \theta) \Exp{G^{t}} } \\
		&=&\Exp{ \delta^{t} } +\frac{\gamma}{2\theta} \Exp{G^{t}}-\frac{\gamma}{2} \Exp{ \sqnorm{\nabla f(x^{t})} } \\
		&& \qquad - \rb{\frac{1}{2\g} -\frac{L}{2}  - \frac{\g}{2\theta}\beta \wL^2  } \Exp{R^{t}} \\
		& \leq& \Exp{ \delta^{t}}+\frac{\gamma}{2\theta} \Exp{ G^{t} } -\frac{\gamma}{2} \Exp{\sqnorm{\nabla f(x^{t})}}.
	\end{eqnarray*}
	The last inequality follows from the bound $\g ^2\frac{\beta \wL^2}{\theta} + L\g \leq 1,$ which holds because of Lemma \ref{le:stepsize_page_fact} and our assumption on the stepsize.
	By summing up inequalities for $t =0, \ldots, T-1,$ and rearranging we get \eqref{eq:main-nonconvex}, since $\hat{x}^{T}$ is chosen from $x^{0}, x^{1}, \ldots, x^{T-1}$ uniformly at random.			
\end{proof}

\begin{corollary}\label{cor:ef21}
	Let assumptions of Theorem~\ref{thm:main-distrib} hold, 
	\begin{eqnarray*}
		g_i^0 &=& \nabla f_i (x^0), \qquad i = 1,\ldots, n, \\
		\g &=&  \rb{L   + \wL \sqrt{\nfr{\beta}{\theta} }}^{-1} .
	\end{eqnarray*}
	Then, after $T$ iterations/communication rounds of \algname{EF21} we have $\Exp{\sqnorm{\nabla f(\hat{x}^{T})} } \leq \eps^2$. It requires 
	\begin{eqnarray}\notag
		T = \# \text{grad} =  \cO\rb{    \fr{   \wL \delta^0}{ \alpha \varepsilon^2} } 
	\end{eqnarray}
	iterations/communications rounds/gradint computations at each node,  where $\wL = \sqrt{\frac{1}{n}\sum_{i=1}^n L_i^2}$,  $\delta^0 = f(x^0) - f^{inf}$.
\end{corollary}

\begin{proof}
	Since $g_i^0 = \nabla f_i(x^0) $, $i = 1, \dots, n$ , we have $G^0 = 0$ and by Theorem~\ref{thm:main-distrib}
	
	\begin{eqnarray*}
		\# \text{grad} &=& T \overset{(i)}{\leq }  \fr{  2\delta^0}{\g \varepsilon^2} \overset{(ii)}{\leq }  \fr{  2\delta^0}{ \varepsilon^2} \rb{L + \wL\sqrt{\frac{\beta}{\theta}}}
		\overset{(iii)}{\leq }  \fr{  2\delta^0}{ \varepsilon^2} \rb{L + \wL \rb{\fr{2}{\alpha} - 1}} \\		
		&\leq & \fr{  2\delta^0}{ \varepsilon^2} \rb{L +  \fr{2\wL}{\alpha}} \overset{(iv)}{\leq }   \fr{  2\delta^0}{ \varepsilon^2} \rb{\fr{\wL}{\alpha} + \fr{2\wL  }{\alpha} } =  \fr{  6 \wL \delta^0}{ \alpha \varepsilon^2},
	\end{eqnarray*}
	where in $(i)$ is due to the rate \eqref{eq:main-nonconvex} given by Theorem~\ref{thm:main-distrib}. In $(ii)$ we plug in the stepsize, in $(iii)$ we use Lemma~\ref{le:optimal_t-Peter}, and $(iv)$ follows by the  inequalities $\alpha \leq 1$, and $L \leq \wL$.
\end{proof}

\newpage

\section{Stochastic Gradients}\label{sec:online}

In this section, we study the extension of \algname{EF21} to the case when stochastic gradients are used instead of full gradients. {\color{myblue} The main idea of the proof is to design an analogous recursion as in Lemma~\ref{lem:theta-beta} for the EF21 error term
	$$
	G_i^{t+1} = \sqnorm{g_i^{t+1} - \nfixtpo}, 
	$$
	where
	$$
	 g_i^{t+1} = g_i^t + \cC(\hgitpo - g_i^t),  \qquad \hgitpo  = \fr{1}{\tau} \sum_{j=1}^{\tau}\nabla f_{\xi_{i j}^{t}}(\xtpo) .
	$$ 
	However, due to additional noise from sampling stochastic gradients, extra error terms occur. The goal of the next lemma is to efficiently control such error terms uing Young's inequality several times and applying Assumption~\ref{as:general_as_for_stoch_gradients} on stochastic gradients.
	
	As in the previous section, we use notations $G^{t} = \suminn G_i^t$,  $G_i^{t} = \sqnorm{\nfixt - g_i^{t}}$.
	}


\begin{lemma}\label{le:ef21_sgd_G_t_bound}
	Let Assumptions~\ref{as:main} and \ref{as:general_as_for_stoch_gradients} hold. Then for all $t \geq 0$ and all constants $\rho,\nu > 0$ \algname{EF21-SGD} satisfies
	\begin{eqnarray}
		\Exp{G^{t+1}} &\leq & (1 - \hat{\theta})\Exp{G^t }  + \hat{\beta}_1 \wL^2 \Exp{\sqnorm{\xtpo -\xt  } }\notag\\
		&&\qquad  + \widetilde{A}\hat\beta_2\Exp{f(x^{t+1}) - f^{\inf}} + \widetilde{C}\hat\beta_2 , \label{eq:ef21_sgd_G_t_bound}
	\end{eqnarray}
	where $\hat{\theta} \eqdef 1 - \rb{1-\alpha} (1+\rho) (1+\nu)$, $\hat{\beta}_1 \eqdef 2\rb{1-\alpha} (1+\rho) \rb{1+\fr{1}{\nu}}$, $\hat{\beta}_2 \eqdef 2 \rb{1-\alpha} (1+\rho) \rb{1+\fr{1}{\nu}} + \rb{1+\fr{1}{\rho}}$, $\widetilde{A} = \max_{i=1,\ldots,n}\frac{2(A_i+L_i(B_i-1))}{\tau_i}$, $\widetilde{C} = \frac{1}{n}\sum\limits_{i=1}^n\left(\frac{2(A_i+L_i(B_i-1))}{\tau_i}\left(f^{\inf} - f_i^{\inf}\right) + \frac{C_i}{\tau_i}\right)$.
\end{lemma}
\begin{proof}
 {\color{myblue}Applying Young’s inequality with parameter $\rho > 0$}
	\begin{eqnarray}
		\Exp{G_i^{t+1}} & = &\Exp{\sqnorm{g_i^{t+1} - \nfixtpo} } 
		\leq (1+\rho) \Exp{\sqnorm{\cC\rb{\hgitpo - g_i^{t}}  - \rb{\hgitpo - g_i^{t}} } } \notag \\
		&& \qquad + \rb{1+\fr{1}{\rho}} \Exp{\sqnorm{ \hgitpo -  \nfixtpo } } \notag \\
		&\leq&\rb{1-\alpha} (1+\rho) \Exp{\sqnorm{g_i^{t} - \hgitpo  } }  + \rb{1+\fr{1}{\rho}} \Exp{\sqnorm{ \hgitpo - \nfixtpo } } . \notag 
	\end{eqnarray}
	 {\color{myblue} Further applying Young’s inequality with parameters $\nu > 0$ and $s = 2$}
	\begin{eqnarray}
			\Exp{G_i^{t+1}} &\leq& \rb{1-\alpha} (1+\rho) (1+\nu)\Exp{\sqnorm{g_i^{t} -\nabla f_{i}(\xt)  } } \notag \\
		&& \qquad + 2\rb{1-\alpha} (1+\rho) \rb{1+\fr{1}{\nu}} \Exp{\sqnorm{\nfixtpo - \hgitpo  } } \notag \\
		&& \qquad + 2\rb{1-\alpha} (1+\rho) \rb{1+\fr{1}{\nu}} \Exp{\sqnorm{\nfixtpo -\nfixt  } } \notag \\
		&& \qquad + \rb{1+\fr{1}{\rho}} \Exp{\sqnorm{ \hgitpo - \nfixtpo } } \notag \\
		&\leq& (1 - \hat{\theta})\Exp{G_i^t }  + \hat{\beta}_1 L_i^2 \Exp{\sqnorm{\xtpo -\xt  } } + \hat{\beta}_2 \Exp{\sqnorm{ \hgitpo - \nfixtpo } }, \notag 
	\end{eqnarray}
	where we introduced $\hat{\theta} \eqdef 1 - \rb{1-\alpha} (1+\rho) (1+\nu)$, $\hat{\beta}_1 \eqdef 2\rb{1-\alpha} (1+\rho) \rb{1+\fr{1}{\nu}}$, $\hat{\beta}_2 \eqdef 2 \rb{1-\alpha} (1+\rho) \rb{1+\fr{1}{\nu}} + \rb{1+\fr{1}{\rho}}$. Next we use independence of $\nabla f_{\xi_{i j}^{t}}(\xt)$, variance decomposition, and \eqref{eq:general_second_mom_upp_bound} to estimate the last term:
	\begin{eqnarray*}
		\Exp{G_i^{t+1}} &\leq& (1 - \hat{\theta})\Exp{G_i^t }  + \hat{\beta}_1 L_i^2 \Exp{\sqnorm{\xtpo -\xt  } } \\
		&& \qquad + \frac{\hat{\beta}_2 }{\tau_i^2}\sum_{j=1}^{\tau_i}\Exp{\sqnorm{ \nabla f_{\xi_{i j}^{t}}(x^{t+1}) - \nfixtpo } }\\
		&=& (1 - \hat{\theta})\Exp{G_i^t }  + \hat{\beta}_1 L_i^2 \Exp{\sqnorm{\xtpo -\xt  } } \\
		&& \qquad + \frac{\hat{\beta}_2 }{\tau_i^2}\sum_{j=1}^{\tau_i}\rb{\Exp{\sqnorm{ \nabla f_{\xi_{i j}^{t}}(x^{t+1}) } } - \Exp{\sqnorm{\nfixtpo } }}\\
		&\overset{\eqref{eq:general_second_mom_upp_bound}}{\leq}& (1 - \hat{\theta})\Exp{G_i^t }  + \hat{\beta}_1 L_i^2 \Exp{\sqnorm{\xtpo -\xt  } } \\
		&& \qquad + \frac{2A_i\hat{\beta}_2 }{\tau_i}\Exp{f_i(x^{t+1}) - f_i^{\inf}} + \frac{\hat{\beta}_2 (B_i - 1)}{\tau_i}\Exp{\sqnorm{\nabla f_i(x^{t+1})}} + \frac{C_i\hat{\beta}_2 }{\tau_i}\\
		&\leq& (1 - \hat{\theta})\Exp{G_i^t }  + \hat{\beta}_1 L_i^2 \Exp{\sqnorm{\xtpo -\xt  } } \\
		&& \qquad + \frac{2(A_i+L_i(B_i-1))\hat{\beta}_2 }{\tau_i}\Exp{f_i(x^{t+1}) - f_i^{\inf}} + \frac{C_i\hat{\beta}_2 }{\tau_i} .
	\end{eqnarray*}
	Averaging the obtained inequality for $i=1,\ldots,n$ we get
	\begin{eqnarray*}
		\Exp{G^{t+1}} &\leq& (1 - \hat{\theta})\Exp{G^t }  + \hat{\beta}_1 \wL^2 \Exp{\sqnorm{\xtpo -\xt  } } \\
		&& \qquad + \frac{1}{n}\sum\limits_{i=1}^n\left(\frac{2(A_i+L_i(B_i-1))\hat{\beta}_2 }{\tau_i}\Exp{f_i(x^{t+1}) - f_i^{\inf}} + \frac{C_i\hat{\beta}_2 }{\tau_i}\right) \\
		&\leq& (1 - \hat{\theta})\Exp{G^t }  + \hat{\beta}_1 \wL^2 \Exp{\sqnorm{\xtpo -\xt  } } \\
		&& \qquad + \frac{1}{n}\sum\limits_{i=1}^n\left(\frac{2(A_i+L_i(B_i-1))\hat{\beta}_2 }{\tau_i}\Exp{f_i(x^{t+1}) - f^{\inf}}\right)\\
		&&\qquad + \frac{\hat\beta_2}{n}\sum\limits_{i=1}^n\left(\frac{2(A_i+L_i(B_i-1))}{\tau_i}\left(f^{\inf} - f_i^{\inf}\right) + \frac{C_i}{\tau_i}\right)\\
		&\leq& (1 - \hat{\theta})\Exp{G^t }  + \hat{\beta}_1 \wL^2 \Exp{\sqnorm{\xtpo -\xt  } }  + \widetilde{A}\hat\beta_2\Exp{f(x^{t+1}) - f^{\inf}} + \widetilde{C}\hat\beta_2 .
	\end{eqnarray*}
\end{proof}


\begin{theorem}\label{thm:EF21_SGD_conv_general_assumption}
	Let Assumptions~\ref{as:main} and \ref{as:general_as_for_stoch_gradients} hold, and let the stepsize in Algorithm~\ref{alg:EF21-online} be set as 
	\begin{equation} \label{eq:stepsize_EF21_SGD}
		0<\gamma \leq \left(L + \wL\sqrt{\frac{\hat \beta_1}{\hat\theta}}\right)^{-1},
	\end{equation}
	where $\wL = \sqrt{\frac{1}{n}\sum_{i=1}^n L_i^2}$, $\hat{\theta} \eqdef 1 - \rb{1-\alpha} (1+\rho) (1+\nu)$, $\hat{\beta}_1 \eqdef 2\rb{1-\alpha} (1+\rho) \rb{1+\fr{1}{\nu}}$, and $\rho,\nu > 0$ are some positive numbers. Assume that batchsizes $\tau_1,\ldots,\tau_i$ are such that $\frac{\gamma\widetilde{A}\hat\beta_2}{2\hat\theta} < 1$, where $\widetilde{A} = \max_{i=1,\ldots,n}\frac{2(A_i+L_i(B_i-1))}{\tau_i}$ and $\hat{\beta}_2 \eqdef 2 \rb{1-\alpha} (1+\rho) \rb{1+\fr{1}{\nu}} + \rb{1+\fr{1}{\rho}}$. Fix $T \geq 1$ and let $\hat{x}^{T}$ be chosen from the iterates $x^{0}, x^{1}, \ldots, x^{T-1}$  with following probabilities:
	\begin{equation*}
		\Prob\left\{\hat{x}^{T} = x^t\right\} = \frac{w_t}{W_T},\quad w_t = \left(1 - \frac{\gamma\widetilde{A}\hat\beta_2}{2\hat\theta}\right)^t,\quad W_T = \sum\limits_{t=0}^{T}w_t.
	\end{equation*}	
	Then
	\begin{equation} \label{eq:main_EF21_SGD}		
		\Exp{\sqnorm{\nabla f(\hat{x}^{T})} } \leq \frac{2(f(x^0) - f^{\inf})}{\gamma T \left(1 - \frac{\gamma\widetilde{A}\hat\beta_2}{2\hat\theta}\right)^T} + \frac{\Exp{G^0}}{\hat{\theta} T \left(1 - \frac{\gamma\widetilde{A}\hat\beta_2}{2\hat\theta}\right)^T} + \frac{\widetilde{C}\hat\beta_2}{\hat\theta},
	\end{equation}
	where $\widetilde{C} = \frac{1}{n}\sum\limits_{i=1}^n\left(\frac{2(A_i+L_i(B_i-1))}{\tau_i}\left(f^{\inf} - f_i^{\inf}\right) + \frac{C_i}{\tau_i}\right)$.
\end{theorem}
\begin{proof}
	We notice that inequality \eqref{eq:func_diff_distrib} holds for \algname{EF21-SGD} as well, i.e., we have
	\begin{eqnarray*}
		\Exp{\delta^{t+1}} &\leq& \Exp{\delta^t}-\frac{\gamma}{2} \Exp{\sqnorm{\nabla f(x^{t})}}  - \left(\frac{1}{2 \gamma}-\frac{L}{2}\right) \Exp{R^t}+ \frac{\gamma}{2}\Exp{G^t}.
	\end{eqnarray*}
	Summing up the above inequality with a $\frac{\gamma}{2\hat\theta}$ multiple of \eqref{eq:ef21_sgd_G_t_bound}, we derive
	\begin{eqnarray*}
		\Exp{\delta^{t+1} + \frac{\gamma}{2\hat\theta}G^{t+1}} &\leq& \Exp{\delta^t}-\frac{\gamma}{2} \Exp{\sqnorm{\nabla f(x^{t})}}  - \left(\frac{1}{2 \gamma}-\frac{L}{2}\right) \Exp{R^t}+ \frac{\gamma}{2}\Exp{G^t}\\
		&& \qquad + \frac{\gamma}{2\hat\theta}(1 - \hat{\theta})\Exp{G^t }  + \frac{\gamma}{2\hat\theta}\hat{\beta}_1 \wL^2 \Exp{R^t}\\
		&&\qquad  + \frac{\gamma}{2\hat\theta}\widetilde{A}\hat\beta_2\Exp{\delta^{t+1}} + \frac{\gamma}{2\hat\theta}\widetilde{C}\hat\beta_2\\
		&\leq& \frac{\gamma\widetilde{A}\hat\beta_2}{2\hat\theta}\Exp{\delta^{t+1}} + \Exp{\delta^{t} + \frac{\gamma}{2\hat\theta}G^{t}} -\frac{\gamma}{2} \Exp{\sqnorm{\nabla f(x^{t})}} + \frac{\gamma}{2\hat\theta}\widetilde{C}\hat\beta_2\\
		&&\qquad - \left(\frac{1}{2 \gamma}-\frac{L}{2} - \frac{\gamma\hat{\beta}_1 \wL^2}{2\hat\theta}\right)\Exp{R^t}\\
		&\overset{\eqref{eq:stepsize_EF21_SGD}}{\leq}& \frac{\gamma\widetilde{A}\hat\beta_2}{2\hat\theta}\Exp{\delta^{t+1}} + \Exp{\delta^{t} + \frac{\gamma}{2\hat\theta}G^{t}} -\frac{\gamma}{2} \Exp{\sqnorm{\nabla f(x^{t})}} + \frac{\gamma}{2\hat\theta}\widetilde{C}\hat\beta_2,
	\end{eqnarray*}
	where $\hat{\theta} \eqdef 1 - \rb{1-\alpha} (1+\rho) (1+\nu)$, $\hat{\beta}_1 \eqdef 2\rb{1-\alpha} (1+\rho) \rb{1+\fr{1}{\nu}}$, $\hat{\beta}_2 \eqdef 2 \rb{1-\alpha} (1+\rho) \rb{1+\fr{1}{\nu}} + \rb{1+\fr{1}{\rho}}$, and $\rho,\nu > 0$ are some positive numbers. Next, we rearrange the terms
	\begin{eqnarray*}
		\Exp{\sqnorm{\nabla f(x^{t})}} 
		&\leq& \frac{2}{\gamma}\left(\Exp{\delta^{t} + \frac{\gamma}{2\hat\theta}G^{t}} - \left(1 - \frac{\gamma\widetilde{A}\hat\beta_2}{2\hat\theta}\right)\Exp{\delta^{t+1} + \frac{\gamma}{2\hat\theta}\Exp{G^{t+1}}}\right) + \frac{\widetilde{C}\hat\beta_2}{\hat\theta},
	\end{eqnarray*}
	sum up the obtained inequalities for $t=0,1,\ldots,T$ with weights $\nicefrac{w_t}{W_T}$, and use the definition of $\hat x^T$
	\begin{eqnarray*}
		\Exp{\sqnorm{\nabla f(\hat x^T)}} &=& \frac{1}{W_K}\sum\limits_{t=0}^T w_t\Exp{\sqnorm{\nabla f(x^{t})}}\\
		&\leq& \frac{2}{\gamma W_T}\sum\limits_{t=0}^T\left(w_t\Exp{\delta^{t} + \frac{\gamma}{2\hat\theta}G^{t}} - w_{t+1}\Exp{\delta^{t+1} + \frac{\gamma}{2\hat\theta}\Exp{G^{t+1}}}\right)  + \frac{\widetilde{C}\hat\beta_2}{\hat\theta}\\
		&\leq& \frac{2\delta^0}{\gamma W_T} + \frac{\Exp{G^0}}{\hat\theta W_T} + \frac{\widetilde{C}\hat\beta_2}{\hat\theta}.
	\end{eqnarray*}
	Finally, we notice
$
		W_T = \sum\limits_{t=0}^{T}w_t \geq (T+1)\min\limits_{t=0,1,\ldots,T} w_t >  T\left(1 - \frac{\gamma\widetilde{A}\hat\beta_2}{2\hat\theta}\right)^T
$
	that finishes the proof.
\end{proof}

\begin{corollary}\label{cor:EF21_SGD_general}
	Let assumptions of Theorem~\ref{thm:EF21_SGD_conv_general_assumption} hold, $\rho = \nicefrac{\alpha}{2}$, $\nu = \nicefrac{\alpha}{4}$,
	\begin{eqnarray*}
		\gamma &=& \frac{1}{L + \wL\sqrt{\frac{\hat \beta_1}{\hat\theta}}},\\
		\tau_i &=& \left\lceil\max\left\{1, \frac{2T\gamma\left(A_i + L_i(B_i-1)\right)\hat\beta_2}{\hat\theta}, \frac{8\left(A_i + L_i(B_i-1)\right)\hat\beta_2}{\hat\theta \varepsilon^2}\delta_i^{\inf}, \frac{4C_i\hat\beta_2}{\hat\theta \varepsilon^2}\right\}\right\rceil,\\
		T &=& \left\lceil\max\left\{\frac{16\delta^0}{\gamma\varepsilon^2}, \frac{8\Exp{G^0}}{\hat\theta\varepsilon^2}\right\} \right\rceil,
	\end{eqnarray*}
	where $\delta_i^{\inf} = f^{\inf} - f_i^{\inf}$, $\delta^0 = f(x^0) - f^{\inf}$. Then, after $T$ iterations of \algname{EF21-SGD} we have $\Exp{\sqnorm{\nabla f(\hat x^T)}} \leq \varepsilon^2$. It requires
	\begin{equation}
		T = \cO\left(\frac{\wL\delta^0+\Exp{G^0}}{\alpha\varepsilon^2}\right) \notag
	\end{equation}
	iterations/communications rounds,
	\begin{eqnarray}
		\#\text{grad}_i &=& \tau_i T = \cO\Bigg(\frac{\wL\delta^0+\Exp{G^0}}{\alpha\varepsilon^2} + \frac{\left(\wL\delta^0+\Exp{G^0}\right)\left(\hat A_i(\delta^0 + \delta_i^{\inf}) + C_i\right)}{\alpha^3\varepsilon^4}\notag\\
		&&\qquad\qquad\qquad\qquad\qquad\qquad\qquad\qquad + \frac{(\wL\delta^0+\Exp{G^0})\hat A_i \Exp{G^0}}{\alpha^2(\alpha L + \wL)\varepsilon^4}\Bigg) \notag
	\end{eqnarray}
	stochastic oracle calls for worker $i$, and
	\begin{eqnarray}
		\overline{\#\text{grad}} &=& \frac{1}{n}\sum\limits_{i=1}^n\tau_i T \notag\\
		&=& \cO\Bigg(\frac{\wL\delta^0+\Exp{G^0}}{\alpha\varepsilon^2} + \frac{1}{n}\sum\limits_{i=1}^n\frac{\left(\wL\delta^0+\Exp{G^0}\right)\left(\hat A_i(\delta^0 + \delta_i^{\inf}) + C_i\right)}{\alpha^3\varepsilon^4}\notag\\
		&&\qquad\qquad\qquad\qquad\qquad\qquad\qquad\qquad + \frac{1}{n}\sum\limits_{i=1}^n\frac{(\wL\delta^0+\Exp{G^0})\hat A_i \Exp{G^0}}{\alpha^2(\alpha L + \wL)\varepsilon^4}\Bigg) \notag
	\end{eqnarray}
	stochastic oracle calls per worker on average, where $\hat A_i = A_i + L_i(B_i-1)$.
\end{corollary}
\begin{proof}
	The given choice of $\tau_i$ ensures that $\left(1 - \frac{\gamma\widetilde{A}\hat\beta_2}{2\hat\theta}\right)^T = \cO(1)$ and $\nicefrac{\widetilde{C}\hat\beta_2}{\hat\theta} \leq \nicefrac{\varepsilon}{2}$. Next, the choice of $T$ ensures that the right-hand side of \eqref{eq:main_EF21_SGD} is smaller than $\varepsilon$. Finally, after simple computation we get the expression for $\tau_i T$. 
\end{proof}

\begin{corollary}\label{cor:EF21_SGD_UBV}
	Consider the setting described in Example~\ref{ex:UBV_case}. Let assumptions of Theorem~\ref{thm:EF21_SGD_conv_general_assumption} hold, $\rho = \nicefrac{\alpha}{2}$, $\nu = \nicefrac{\alpha}{4}$,
	\begin{eqnarray*}
		\gamma = \frac{1}{L + \wL\sqrt{\frac{\hat \beta_1}{\hat\theta}}},\quad \tau_i = \left\lceil\max\left\{1, \frac{4\sigma_i^2\hat\beta_2}{\hat\theta \varepsilon^2}\right\}\right\rceil,\quad T = \left\lceil\max\left\{\frac{16\delta^0}{\gamma\varepsilon^2}, \frac{8\Exp{G^0}}{\hat\theta\varepsilon^2}\right\} \right\rceil,
	\end{eqnarray*}
	where $\delta^0 = f(x^0) - f^{\inf}$. Then, after $T$ iterations of \algname{EF21-SGD} we have $\Exp{\sqnorm{\nabla f(\hat x^T)}} \leq \varepsilon^2$. It requires
	\begin{equation}
		T = \cO\left(\frac{\wL\delta^0+\Exp{G^0}}{\alpha\varepsilon^2}\right) \notag
	\end{equation}
	iterations/communications rounds,
	\begin{eqnarray}
		\#\text{grad}_i &=& \tau_i T = \cO\left(\frac{\wL\delta^0+\Exp{G^0}}{\alpha\varepsilon^2} + \frac{\left(\wL\delta^0+\Exp{G^0}\right)\sigma_i^2}{\alpha^3\varepsilon^4} \right) \notag
	\end{eqnarray}
	stochastic oracle calls for worker $i$, and
	\begin{eqnarray}
		\overline{\#\text{grad}} &=& \frac{1}{n}\sum\limits_{i=1}^n\tau_i T = \cO\left(\frac{\wL\delta^0+\Exp{G^0}}{\alpha\varepsilon^2} + \frac{\left(\wL\delta^0+\Exp{G^0}\right)\sigma^2}{\alpha^3\varepsilon^4} \right) \notag
	\end{eqnarray}
	stochastic oracle calls per worker on average, where $\sigma^2 = \frac{1}{n}\sum_{i=1}^n\sigma_i^2$. 
\end{corollary}

\begin{corollary}\label{cor:EF21_SGD_subsampling}
	Consider the setting described in Example~\ref{ex:subsampling_case}. Let assumptions of Theorem~\ref{thm:EF21_SGD_conv_general_assumption} hold, $\rho = \nicefrac{\alpha}{2}$, $\nu = \nicefrac{\alpha}{4}$,
	\begin{eqnarray*}
		\gamma &=& \frac{1}{L + \wL\sqrt{\frac{\hat \beta_1}{\hat\theta}}},\\
		\tau_i &=& \left\lceil\max\left\{1, \frac{2T\gamma\overline{L}_i\hat\beta_2}{\hat\theta}, \frac{8\overline{L}_i\hat\beta_2}{\hat\theta \varepsilon^2}\delta_i^{\inf}, \frac{8\overline{L}_i\Delta_i^{\inf}\hat\beta_2}{\hat\theta \varepsilon^2}\right\}\right\rceil,\\
		T &=& \left\lceil\max\left\{\frac{16\delta^0}{\gamma\varepsilon^2}, \frac{8\Exp{G^0}}{\hat\theta\varepsilon^2}\right\} \right\rceil,
	\end{eqnarray*}
	where $\delta_i^{\inf} = f^{\inf} - f_i^{\inf}$, $\delta^0 = f(x^0) - f^{\inf}$, $\overline{L}_i = \frac{1}{m}\sum_{j=1}^{m}L_{ij}$, $\Delta_i^{\inf} = \frac{1}{m}\sum_{j=1}^{m} (f_i^{\inf} - f_{ij}^{\inf})$. Then, after $T$ iterations of \algname{EF21-SGD} we have $\Exp{\sqnorm{\nabla f(\hat x^T)}} \leq \varepsilon^2$. It requires
	\begin{equation}
		T = \cO\left(\frac{\wL\delta^0+\Exp{G^0}}{\alpha\varepsilon^2}\right) \notag
	\end{equation}
	iterations/communications rounds,
	\begin{eqnarray}
		\#\text{grad}_i &=& \tau_i T = \cO\Bigg(\frac{\wL\delta^0+\Exp{G^0}}{\alpha\varepsilon^2} + \frac{\left(\wL\delta^0+\Exp{G^0}\right)\left(\overline{L}_i(\delta^0 + \delta_i^{\inf}) + \overline{L}_i\Delta_i^{\inf}\right)}{\alpha^3\varepsilon^4}\notag\\
		&&\qquad\qquad\qquad\qquad\qquad\qquad\qquad\qquad + \frac{(\wL\delta^0+\Exp{G^0})\overline{L}_i \Exp{G^0}}{\alpha^2(\alpha L + \wL)\varepsilon^4}\Bigg) \notag
	\end{eqnarray}
	stochastic oracle calls for worker $i$, and
	\begin{eqnarray}
		\overline{\#\text{grad}} &=& \frac{1}{n}\sum\limits_{i=1}^n\tau_i T \notag\\
		&=& \cO\Bigg(\frac{\wL\delta^0+\Exp{G^0}}{\alpha\varepsilon^2} + \frac{1}{n}\sum\limits_{i=1}^n\frac{\left(\wL\delta^0+\Exp{G^0}\right)\left(\overline{L}_i(\delta^0 + \delta_i^{\inf}) + \overline{L}_i\Delta_i^{\inf}\right)}{\alpha^3\varepsilon^4}\notag\\
		&&\qquad\qquad\qquad\qquad\qquad\qquad\qquad\qquad + \frac{1}{n}\sum\limits_{i=1}^n\frac{(\wL\delta^0+\Exp{G^0})\overline{L}_i \Exp{G^0}}{\alpha^2(\alpha L + \wL)\varepsilon^4}\Bigg) \notag
	\end{eqnarray}
	stochastic oracle calls per worker on average.
\end{corollary}

\newpage

\section{Variance Reduction}\label{sec:variance_reduction}

In this part, we modify the \algname{EF21} framework to better handle \textit{finite-sum} problems with smooth summands. Unlike the \textit{online/streaming case} where \algname{SGD} has the optimal complexity (without additional assumption on the smoothness of stochastic trajectories) \citep{arjevani2023lower}, in the \textit{finite sum} regime, it is well-known that one can hope for convergence to the exact stationary point rather than its neighborhood. To achieve this, variance reduction techniques are instrumental. One approach is to apply a \algname{PAGE}-estimator~\citep{PAGE2021} instead of a random minibatch applied in \algname{SGD}. 


We recall the notations used in this section: $P_i^t =  \sqnorm{ \nabla f_i(\xt) - v_i^{t} }$, $P^t = \suminn P_i^t$, $V_i^t =  \sqnorm{v_i^{t} - g_i^t}$, $V^t = \suminn V_i^t$, where $v_i^t$ is a PAGE estimator. As before, $G^{t} = \suminn G_i^t$,  $G_i^{t} = \sqnorm{\nfixt - g_i^{t}}$.

\begin{lemma}\label{le:ineq_3_ef21_page-dist}
	Let Assumption~\ref{as:avg_smoothness_page} hold, and let $v_i^{t+1}$ be a PAGE estimator, i. e. for $b_i^t \sim \opn{Be}(p_i)$
	\begin{equation*}
		v_i^{t+1} = \begin{cases}
			\nabla f_i(x^{t+1}) &\text{if }\quad  b_i^t  = 1,\\
			v_i^{t}+ \fr{1}{\tau_i} \sum \limits_{j\in I_{i}^{t}} \rb{\nabla f_{ij}(x^{t+1}) - \nabla f_{ij}(\xt)}&\text{if }\quad  b_i^t  = 0,
		\end{cases}
	\end{equation*}
	for all $i = 1, \dots, n$, $t\geq 0$.  Then 
	\begin{equation*}\label{eq:ineq_3_ef21_page-dist}
		\Exp{P^{t+1}} \leq (1-p_{\min}) \Exp{ P^t} + {\wcL^2}\Exp{\sqnorm{\xtpo - \xt} },
	\end{equation*}
	where $\wcL^2 = \suminn \fr{(1-p_i)\mathcal{L}_{i}^2}{\tau_i}$, $p_{\min} = \min_{i=1,\ldots,n} p_i$, and $P_i^t =  \sqnorm{ \nabla f_i(\xt) - v_i^{t} }$, $P^t = \suminn P_i^t$.
\end{lemma}
\begin{proof}
	\begin{eqnarray}
		\Exp{P_i^{t+1}} &=&\Exp{\sqnorm{v_i^{t+1}-\nabla f_i(\xtpo)}} 	\notag \\ &=& (1-p_i) \Exp{\sqnorm{v_i^{t}+\fr{1}{\tau_i} \sum_{j\in I_i^t} (\nabla f_{i j}(x^{t+1})- \nabla f_{i j}(x^{t})) - \nabla f_i(\xtpo)}}  \notag \\
		&=& (1-p_i) \Exp{\sqnorm{v_i^{t} - \nabla f_i(\xt)  +\widetilde{\Delta}_i^t - \nabla f_i(\xtpo)+ \nabla f_i(\xt) }} \notag \\
		&=& (1-p_i) \Exp{\sqnorm{v_i^{t} - \nabla f_i(\xt)  + \widetilde{\Delta}_i^t - \Delta_i^t  }} \notag \\
		&\overset{(i)}{=}&  (1-p_i)\Exp{ \sqnorm{v_i^{t} - \nabla f_i(\xt)}} + (1-p_i) \Exp{\sqnorm{\widetilde{\Delta}_i^t - \Delta_i^t }}  \notag \\
		&\overset{(ii)}{\leq}& (1-p_i) \Exp{P_i^t} + \fr{(1-p_i) \cL_i^2}{\tau_i}\Exp{\sqnorm{\xtpo - \xt} }\notag \\
		&\leq& (1-p_{\min}) \Exp{P_i^t} + \fr{(1-p_i) \cL_i^2}{\tau_i}\Exp{\sqnorm{\xtpo - \xt} },  \notag 
	\end{eqnarray}
	where equality $(i$) holds because $\Exp { \widetilde{\Delta}_i^t - \Delta_i^t  \mid \xt, \xtpo, v_i^t} = 0,$ and $(ii)$ holds by Assumption~\ref{as:avg_smoothness_page}.
	
	It remains to average the above inequality over $i = 1, \dots, n$.
\end{proof}

\begin{lemma}\label{le:ineq_2_ef21_page-dist}
	Let Assumptions~\ref{as:main} and \ref{as:avg_smoothness_page} hold, let $v_i^{t+1}$ be a PAGE estimator, i. e. for $b_i^t  \sim \opn{Be}(p_i)$ and for all $i = 1, \dots, n$, $t\geq 0$
	\begin{equation}\label{eq:page_est}
		v_i^{t+1} = \begin{cases}
			\nabla f_i(x^{t+1}) &\text{if }\quad  b_i^t  = 1,\\
			v_i^{t}+ \fr{1}{\tau_i} \sum \limits_{j\in I_{i}^{t}} \rb{\nabla f_{ij}(x^{t+1}) - \nabla f_{ij}(\xt)}&\text{if }\quad  b_i^t  = 0,
		\end{cases}
	\end{equation}
	and  let $g_i^{t+1}$ be an EF21 estimator, i. e. 
	\begin{equation}
		g_i^{t+1} = g_i^t + \cC( v_i^{t+1} - g_i^t), \quad g_i^0 =\cC\rb{ v_i^0} \notag 
	\end{equation}
	for all $i = 1, \dots, n$, $t\geq 0$. Then 
	\begin{equation}\label{eq:ineq_2_ef21_page-dist}
		\Exp { V^{t+1} } \leq (1-\theta) \Exp {V^{t}} + 2\beta p_{\max} \Exp {P^t } + \beta \rb{2  \wL^2 + { \wcL^2} } \Exp { \sqnorm {\xtpo -  \xt} },
	\end{equation}
	where  $\wcL = \suminn \fr{(1-p_i)\mathcal{L}_{i}^2}{\tau_i}$, $p_{\max} = \max_{i=1,\ldots,n} p_i$, $\theta = 1- (1- \alpha )(1+s)$,  $\beta =  (1- \alpha ) \left(1+ s^{-1} \right)$ for any $s > 0$, and $P_i^t =  \sqnorm{ \nabla f_i(\xt) - v_i^{t} }$, $P^t = \suminn P_i^t$, $V_i^t =  \sqnorm{v_i^{t} - g_i^t}$, $V^t = \suminn V_i^t$.
\end{lemma}
\begin{proof}
	Following the steps in proof of  Lemma~\ref{lem:theta-beta}, but with $\nfixtpo$ and $\nfixt$ being substituted by their estimators $v_i^{t+1}$ and $v_i^t$, we end up with an analogue of  \eqref{eq:beta-theta-without-smoothness}
	\begin{eqnarray}\label{eq:beta-theta-without-smoothness-page}
		\Exp { \sqnorm{g_i^{t+1} - v_i^{t+1}} } &{\leq} & (1-\theta) \Exp {\sqnorm{g_i^t - v_i^t}} + \beta \Exp {\sqnorm{v_i^{t+1} - v_i^t} },  
	\end{eqnarray}
	where $\theta = 1- (1- \alpha )(1+s)$,  $\beta =  (1- \alpha ) \left(1+ s^{-1} \right)$ for any $s > 0$. Then
	\begin{eqnarray}
		\Exp { V_i^t } & = & \Exp { \sqnorm{g_i^{t+1} - v_i^{t+1}} } \notag \\
		&\overset{\eqref{eq:beta-theta-without-smoothness-page}}{\leq} & (1-\theta) \Exp {\sqnorm{g_i^t - v_i^t}} + \beta \Exp {\sqnorm{v_i^{t+1} - v_i^t} } \notag \\
		&= & (1-\theta) \Exp {\sqnorm{g_i^t - v_i^t}} + \beta \Exp {\Exp{\sqnorm{v_i^{t+1} - v_i^t} \mid v_i^t, \xt,\xtpo } } \notag \\
		&\overset{(i)}{=}& (1-\theta) \Exp {V_i^t } + \beta p_i\Exp {\sqnorm{v_i^{t} - \nabla f_i(\xtpo)} } \notag \\
		&& \qquad + \beta (1-p_i) \Exp { \sqnorm { \fr{1}{\tau_i} \sum \limits_{j\in I_i^t} \rb{\nabla f_{i  j}(x^{t+1}) - \nabla f_{i j }(\xt)} } } \notag \\
		&=&(1-\theta) \Exp {V_i^t } + \beta p_i\Exp {\sqnorm{v_i^{t} - \nabla f_i(\xtpo)} } + \beta (1-p_i) \Exp { \sqnorm { \widetilde{\Delta}_i^t} } = (*)  \notag .
			\end{eqnarray}
			{\color{myblue} where in $(i)$ we use the definition of PAGE estimator \eqref{eq:page_est}. Next, we continue by using Young's inequality \eqref{eq:facts:young_ineq3_avg} with $s = 1$ in $(ii)$}
			\begin{eqnarray}
		(*) &\overset{(ii)}{=}&(1-\theta) \Exp {V_i^t } + 2\beta p_i\Exp {\sqnorm{v_i^{t} - \nabla f_i(\xt)} } \notag \\
		&& \qquad  + 2\beta p_i\Exp {\sqnorm{ \nabla f_i(\xtpo) - \nabla f_i(\xt) } }  + \beta (1-p_i)  \Exp { \sqnorm { \widetilde{\Delta}_i^t} } \notag \\
		&=&(1-\theta) \Exp {V_i^t } + 2\beta p_i\Exp {P_i^t }  + 2\beta p_i\Exp {\sqnorm{ \Delta_i^t} }  + \beta (1-p_i)  \Exp { \sqnorm { \widetilde{\Delta}_i^t} } \notag \\
		&\overset{(iii)}{=}&(1-\theta) \Exp {V_i^t } + 2\beta p_i\Exp {P_i^t}  + \beta (2 p_i + 1-p_i)\Exp {\sqnorm{ \Delta_i^t} } \notag \\
		&& \qquad   + \beta (1-p_i)  \Exp { \sqnorm { \widetilde{\Delta}_i^t - \Delta_i^t} } \notag \\
		&\overset{(iv)}{\leq} & (1-\theta) \Exp {V_i^t } + 2\beta p_i\Exp {P_i^t}  + \beta (1+p_i) L_i^2 \Exp { \sqnorm {\xtpo -  \xt} } \notag \\
		&& \qquad  +\beta \fr{ (1-p_i)\cL_i^2}{\tau_i}   \Exp { \sqnorm {\xtpo -  \xt} }  \notag \\
		&\leq & (1-\theta) \Exp {V_i^t} + 2\beta p_{\max}\Exp {P_i^t}  + \beta \rb{ 2 L_i^2 + \fr{(1-p_i) \cL_i^2}{\tau_i} }  \Exp { \sqnorm {\xtpo -  \xt} }  , \notag
	\end{eqnarray}
	 where $(iii)$ is due to bias-variance decomposition, $(iv)$ makes use of Assumptions~\ref{as:main} and \ref{as:avg_smoothness_page}, and the last step is due to $p_i \leq 1$, $p_i \leq p_{\max}$ . It remains to average the above inequality over $i = 1, \dots, n$.

\end{proof}


\begin{theorem}\label{th:ef21-page-dist}
	Let Assumptions~\ref{as:main} and \ref{as:avg_smoothness_page} hold, and let the stepsize in Algorithm~\ref{alg:EF21-PAGE} be set as  
	\begin{equation}\label{eq:ef21-page-stepsize}
		0<\gamma \leq \rb{L + \sqrt{\fr{4\beta}{\theta}\wL^2 + 2 \rb{\fr{3\beta}{\theta}\fr{p_{\max}}{p_{\min}}  + \fr{1}{p_{\min}} } {\wcL^2}}  }^{-1} .
	\end{equation}
	Fix $T \geq 1$ and let $\hat{x}^{T}$ be chosen from the iterates $x^{0}, x^{1}, \ldots, x^{T-1}$  uniformly at random. Then
	\begin{equation}\label{eq:ef21-page-dist_rate}
		\Exp{\sqnorm{\nabla f(\hat{x}^{T})} } \leq \fr{2 \Psi^0 }{\gamma T},
	\end{equation}
	where $\Psi^{t} \eqdef  f(x^t) - \finf + \frac{\gamma }{\theta} V^{t} + \fr{\g}{p_{\min}} \rb{1+\fr{2\beta p_{\min} }{\theta}} P^{t}$, $p_{\max} = \max_{i=1,\ldots,n} p_i$, $p_{\min} = \min_{i=1,\ldots,n} p_i$,  $\wL = \sqrt{\frac{1}{n}\sum_{i=1}^n L_i^2}$, $\theta = 1- (1- \alpha )(1+s)$,  $\beta =  (1- \alpha ) \left(1+ s^{-1} \right)$ for any $s > 0$.
\end{theorem}

\begin{proof}
	We apply Lemma \ref{le:aux_smooth_lemma} and split the error $\sqnorm{g_i^{t}-\nabla f_i (x^{t})}$ in two parts
	\begin{eqnarray}\label{eq:ineq_1}
		f(x^{t+1}) 	&\leq&  f(x^{t})  -\fr{\g}{2} \sqnorm{\nf{\xt}} -\left(\fr{1}{2 \gamma}-\fr{L}{2}\right) R^t + \fr{\gamma}{2}\sqnorm{g^{t}-\nabla f(x^{t})}\notag \\
		&\leq&{ f(x^{t}) } -\fr{\g}{2}\sqnorm{\nf{\xt}} -\left(\fr{1}{2 \gamma}-\fr{L}{2}\right) {R^t} \notag \\ 
		&& \qquad + \g {\sqnorm{g^{t}- v^t }} + \g\Exp{\sqnorm{ v^t - \nfxt }} \notag \\
		&\leq&{ f(x^{t}) } -\fr{\g}{2}{\sqnorm{\nf{\xt}}} -\left(\fr{1}{2 \gamma}-\fr{L}{2}\right) {R^t} \notag \\ 
		&& \qquad + \g \suminn{\sqnorm{g_i^{t}- v_i^t }} + \g \suminn{\sqnorm{ v_i^t - \nfixt }} \notag \\
		&=&{ f(x^{t}) } -\fr{\g}{2}{\sqnorm{\nf{\xt}}} -\left(\fr{1}{2 \gamma}-\fr{L}{2}\right) {R^t}  + \g {V^t} + \g {P^t} ,
	\end{eqnarray}
	where we used notation $R^t = \sqnorm{\g g^t} = \sqnorm{\xtpo - \xt}$, and applied \eqref{eq:facts:young_ineq2} and \eqref{eq:facts:young_ineq3_avg}.
	
	Subtracting $f^{\text {inf }}$ from both sides of the above inequality, taking expectation and using the notation $\delta^t = f(x^{t+1})-\finf$, we get
	\begin{eqnarray}\label{eq:ef21_page_0} 
		\Exp{\delta^{t+1}} &\leq&\Exp{\delta^{t}} -\fr{\g}{2}\Exp{\sqnorm{\nf{\xt}}} -\left(\fr{1}{2 \gamma}-\fr{L}{2}\right) \Exp{R^t} + \g \Exp{V^t} + \g \Exp{P^t} .
	\end{eqnarray}
	
	Further, Lemma~\ref{le:ineq_3_ef21_page-dist} and \ref{le:ineq_2_ef21_page-dist} provide the recursive bounds for the last two terms of (\ref{eq:ef21_page_0})
	
	\begin{eqnarray}\label{eq:ef21_page_2}
		\Exp{ P^{t+1}  } &\le& 	(1 - p_{\min}) \Exp{P^t} + {\wcL^2} \Exp{R_t},
	\end{eqnarray}
	
	\begin{eqnarray}\label{eq:ef21_page_1}
		\Exp{V^{t+1}} &\le& 	(1 - \theta) \Exp{V^t} + \beta \rb{2  \wL^2 + { \wcL^2} }  \Exp{R_t} + 2\beta p_{\max} \Exp{P^t}  .
	\end{eqnarray}

	Adding (\ref{eq:ef21_page_0}) with a $\fr{\g}{\theta}$ multiple of (\ref{eq:ef21_page_1}) we obtain
	\begin{eqnarray}
		\Exp{\delta^{t+1}}+\frac{\gamma }{\theta} \Exp{V^{t+1}} & \leq& \Exp{\delta^{t}}-\frac{\gamma}{2} \mathbb{E}\left[\left\|\nabla f\left(x^{t}\right)\right\|^{2}\right]-\left(\frac{1}{2 \gamma}-\frac{L}{2}\right) \Exp{R^{t}}  +\g \Exp{V^{t}} \notag \\
		&& \qquad +\gamma \Exp{P^{t}}  +\frac{\g}{\theta}\left(\left(1-\theta \right) \Exp{V^{t}}+A r^{t}+C \Exp{P^{t}}\right) \notag \\ 
		& \leq &\delta^{t}+\frac{\gamma }{\theta} \Exp{V^{t}}-\frac{\gamma}{2} \mathbb{E}\left[\left\|\nabla f\left(x^{t}\right)\right\|^{2}\right]-\left(\frac{1}{2 \gamma}-\frac{L}{2}-\frac{\gamma  A}{\theta}\right) \Exp{R^{t}}\notag \\
		&& \qquad+\gamma \rb{1+\fr{C}{\theta}} \Exp{P^{t}}, \notag
	\end{eqnarray}
	where we denote $A \eqdef \beta \rb{2  \wL^2 + {\wcL^2} }$, $C \eqdef 2\beta p_{\max} $.
	
	Then adding the above inequality with a $\fr{\g}{p_{\min}} \rb{1+\fr{C}{\theta}} $ multiple of (\ref{eq:ef21_page_2}), we get
	
	\begin{eqnarray}\label{eq:ef21_page_summed_all}
		\Exp{\Phi^{t+1}} &=& \Exp{\delta^{t+1}}+\frac{\gamma }{\theta} \Exp{V^{t+1}} + \fr{\g}{p_{\min}} \rb{1+\fr{C}{\theta}} \Exp{P^{t+1}} \notag \\
		& \leq &\delta^{t}+\frac{\gamma }{\theta} \Exp{V^{t}}-\frac{\gamma}{2} \mathbb{E}\left[\left\|\nabla f\left(x^{t}\right)\right\|^{2}\right]-\left(\frac{1}{2 \gamma}-\frac{L}{2}-\frac{\gamma  A}{\theta}\right) \Exp{R^{t}}\notag \\
		&& \qquad+\gamma \rb{1+\fr{C}{\theta}} \Exp{P^{t}} \notag \\
		&& \qquad + \fr{\g}{p_{\min}} \rb{1+\fr{C}{\theta}} \rb{ (1-p_{\min}) \Exp{P^t} +  \wcL^2 \Exp{R^t} } \notag \\
		& \leq& \Exp{\delta^{t}}+\frac{\gamma }{\theta} \Exp{V^{t}} + \fr{\g}{p_{\min}} \rb{1+\fr{C}{\theta}} \Exp{P^t} - \frac{\gamma}{2} \mathbb{E}\left[\left\|\nabla f\left(x^{t}\right)\right\|^{2}\right] \notag \\
		&& \quad -\left(\frac{1}{2 \gamma}-\frac{L}{2}-\frac{\gamma  A}{\theta} - \fr{\g}{p_{\min}} \rb{1+\fr{C}{\theta}} { \wcL^2} \right) \Exp{R^{t}} \notag \\
		& =& \Exp{\Phi^{t}}  - \frac{\gamma}{2} \mathbb{E}\left[\left\|\nabla f\left(x^{t}\right)\right\|^{2}\right] \notag \\
		&& \quad -\left(\frac{1}{2 \gamma}-\frac{L}{2}-\frac{\gamma  A}{\theta} - \fr{\g}{p_{\min}} \rb{1+\fr{C}{\theta}} { \wcL^2} \right) \Exp{R^{t}} . 
	\end{eqnarray}
	The coefficient in front of $\Exp{R^t}$ simplifies after substitution by $A$ and $C$
	
	$$
	\frac{\gamma  A}{\theta} + \fr{\g}{p_{\min}} \rb{1+\fr{C}{\theta}} { \wcL^2} \leq \fr{2\beta}{\theta}\wL^2 +  \rb{\fr{3\beta}{\theta}\fr{p_{\max}}{p_{\min}} + \fr{1}{p_{\min}} } {\wcL^2}.
	$$
	
	Thus by Lemma \ref{le:stepsize_page_fact} and the stepsize choice,
	the last term in \eqref{eq:ef21_page_summed_all} is not positive. By summing up inequalities for $t =0, \ldots, T-1,$ and rearranging we get \eqref{eq:ef21-page-dist_rate}.
	

\end{proof}

\begin{corollary}\label{cor:ef21-page-dist}
	Let assumptions of Theorem~\ref{th:ef21-page-dist} hold, 
	\begin{eqnarray*}
		v_i^0 &=& g_i^0 = \nabla f_i (x^0), \quad i = 1, \ldots, n, \\
		\gamma &=& \rb{L + \sqrt{\fr{4\beta}{\theta}\wL^2 + 2 \rb{\fr{3\beta}{\theta}\fr{p_{\max}}{p_{\min}}  + \fr{1}{p_{\min}} } {\wcL^2}}  }^{-1} ,\\
		p_i &=& \fr{\tau_i}{\tau_i + m}, \quad i = 1, \ldots, n.
	\end{eqnarray*}
	Then, after $T$ iterations/communication rounds of \algname{EF21-PAGE} we have $\Exp{\sqnorm{\nabla f(\hat{x}^{T})} } \leq \eps^2$. It requires 
	
	\begin{equation*}
		T = \cO\rb{\fr{(\wL+\widetilde{\cL})\delta^0}{\alpha\varepsilon^2} \sqrt{\fr{p_{\max}}{p_{\min}}} + \fr{\sqrt{m_{\max}}\widetilde{\cL}\delta^0}{\varepsilon^2} }
	\end{equation*}	 
	
	iterations/communications rounds,
	
	\begin{eqnarray*}
		\#\text{grad}_i	&=&	\cO\rb{m + \fr{\tau_i (\wL+\widetilde{\cL})\delta^0}{\alpha\varepsilon^2}\sqrt{\fr{p_{\max}}{p_{\min}}} + \fr{\tau_i \sqrt{m}\widetilde{\cL}\delta^0}{\varepsilon^2} }
	\end{eqnarray*}
	
	stochastic oracle calls for worker $i$, and
	
	\begin{eqnarray*}
		\overline{\#\text{grad}}	&=&  \cO\rb{m + \fr{\tau(\wL+\widetilde{\cL})\delta^0}{\alpha\varepsilon^2}\sqrt{\fr{p_{\max}}{p_{\min}}} + \fr{\tau\sqrt{m}\widetilde{\cL}\delta^0}{\varepsilon^2} }
	\end{eqnarray*}
	stochastic oracle calls per worker on average, where $\tau = \suminn \tau_i$, $p_{\max} = \max_{i=1,\ldots,n} p_i$, $p_{\min} = \min_{i=1,\ldots,n} p_i$.
	
\end{corollary}

\begin{proof}
	The proof is straightforward using Lemma~\ref{le:optimal_t-Peter} and the formula:
$
		\#\text{grad}_i =  m + T \rb{p_i m + (1-p_i)  \tau_i} .
$
\end{proof}


\newpage
\section{Partial Participation}\label{sec:partial_participation}
In this section, we further motivate the option for partial
participation of the clients -- a feature important in federated learning. Later, we continue with a rigorous proof of \algname{EF21-PP} algorithm.

Most of the works in compressed distributed optimization deal with full worker participation, i.e., the case when all clients are involved in computation and communication at every iteration. However, in the practice of federated learning, only a subset of clients are allowed to participate at each training round. This limitation comes mainly due to the following two reasons. First, clients (e.g., mobile devices) may wish to join or leave the network randomly. Second, it is often prohibitive to wait for all available clients since stragglers can significantly slow down the training process. Although many existing works \citep{MARINA,A_better_alternative,Artemis2020,SCAFFOLD, yang2021achieving,Power_of_Choice} allow for partial participation, they assume either unbiased compressors or no compression at all. 

We provide a simple analysis of partial participation, which works with \textit{biased compressors} and builds upon the  EF21 mechanism.

\begin{lemma}\label{le:rec_2-pp}
	For Algorithm~\ref{alg:PP-EF21} it holds
	\begin{equation}
		\Exp{G^{t+1}} \leq (1-\theta_p) \Exp{G^t} + B \Exp{\sqnorm{\xtpo - \xt}}
	\end{equation}
	with $\theta_p \eqdef  {\rho p_{min} + \theta p_{max}  -\rho - \rb{p_{max} - p_{min}}}$, $B \eqdef {\suminn \rb{ \beta p_i  +  \rb{1+\rho^{-1}} (1-p_i) } L_i^2}   $, $p_{max} \eqdef \max_{1\leq i \leq n} p_i$,  $p_{min} \eqdef \min_{1\leq i \leq n} p_i$,   $\theta = 1 - (1+s) (1-\al)$, $\beta = \rb{1+\fr{1}{s}} (1-\al)$ and small enough $\rho, s > 0$.
\end{lemma}
\begin{proof}
	By \eqref{eq:rec_1} in Lemma~\ref{lem:theta-beta}, we have for all $i\in S_t$
	\begin{equation}\label{eq:i_works}
		\Exp{G_i^{t+1}\mid i \in S_t}  \leq (1-\theta) \Exp{  G_i^t  } + \beta  L_i^2 \Exp{ \sqnorm{\xtpo - \xt} \mid i\in S_t }
	\end{equation}
	with $\theta = 1 - (1+s) (1-\al)$, $\beta = \rb{1+\fr{1}{s}} (1-\al)$ and arbitrary $s > 0$. 
	
	Define $W^t \eqdef \{g_1^t, \dots, g_n^t, x^t, x^{t+1}\}$ and let $i \notin S_t$, then
	\begin{eqnarray}\label{eq:i_rests}
		\Exp{G_i^{t+1}\mid i\notin S_t}& = & \Exp{ \Exp{G_i^{t+1}\mid W^t } \mid i\notin S_t} 
		 = \Exp{ \Exp{\sqnorm{\gitpo - \nfixtpo}\mid W^t } \mid i\notin S_t} \notag \\
		& \leq & (1+\rho) \Exp{ \Exp{\sqnorm{\git -\nfixt }\mid W^t} \mid i\notin S_t} \notag \\
		&& \qquad + \rb{1+\rho^{-1}} \Exp{ \Exp{\sqnorm{\nfixtpo - \nfixt}\mid W^t } \mid i\notin S_t} \notag \\
		& \leq & (1+\rho) \Exp{ G_i^t}  + \rb{1+\rho^{-1}} \Exp{ \sqnorm{\nfixtpo - \nfixt}  \mid i\notin S_t} \notag \\
		& \leq & (1+\rho) \Exp{ G_i^t} + \rb{1+\rho^{-1}} L_i^2 \Exp{ \sqnorm{\xtpo - \xt}  }.
	\end{eqnarray}
	Combining \eqref{eq:i_works} and \eqref{eq:i_rests}, we get
	\begin{eqnarray}
		\Exp{G^{t+1}} & =& \suminn \Exp{G_i^{t+1} } \notag \\
		& =& \suminn p_i \Exp{G_i^{t+1} \mid i\in S_t}  + \suminn \rb{1-p_i} \Exp{G_i^{t+1} \mid  i\notin S_t}   \notag \\
		&\overset{(\ref{eq:i_works}), (\ref{eq:i_rests})}{\le}& (1-\theta) \suminn  p_i \Exp{ G_i^t}   + \beta \rb{ \suminn p_i L_i^2 } \Exp{ \sqnorm{\xtpo-\xt} }  \notag\\
		&&  \qquad +  \rb{1+\rho} \suminn \rb{1-p_i}\Exp{ G_i^t } \notag\\
		&&  \qquad +  \rb{1+\rho^{-1}}\rb{\suminn (1-p_i) L_i^2} \Exp{ \sqnorm{\xtpo - \xt} } \notag \\
		& \leq & \bigg( 1 - \rb{\rho p_{min} + \theta p_{max}  -\rho - \rb{p_{max} - p_{min}}}\bigg) \Exp{G^t}    \notag \\
		&& \qquad  +\rb{\suminn \rb{ \beta p_i  +  \rb{1+\rho^{-1}} (1-p_i) } L_i^2}   \Exp{\sqnorm{\xtpo-\xt} }. \notag \\
		& =& \rb{ 1 - \theta_p} \Exp{G^t}    + B \Exp{ \sqnorm{\xtpo-\xt} } , \notag
	\end{eqnarray}
	where in the last inequality we replaced corresponding $p_i$ with $p_{max}$ and $p_{min}$, and rearranged the terms. 
\end{proof}

\begin{lemma}\label{le:stepsize_simplify}[To simplify the rates for partial participation]
	Let $B$ and $\theta_p$ be defined as in Theorem~\ref{th:pp-ef21}, and let $p_i = p > 0$ for all $i = 1, \dots, n$ . Then there exist $\rho, s >0$ such that
	\begin{eqnarray}\label{eq:pp_rate_simplify}
		\theta_p \geq \fr{p\al}{2},
	\end{eqnarray}
	\begin{eqnarray}\label{eq:pp_stepsize_simplify}
		0 < \fr{B}{\theta_p} \leq \rb{ \fr{4 \wL}{p\al} }^2.
	\end{eqnarray}
\end{lemma}
\begin{proof}
	Under the assumption that $p_i = p$ for all $i = 1, \dots, n$, the constants simplify to
	$$\theta_p =  {\rho p + \theta p  -\rho },$$ 
	$$B =  \rb{ \beta p  +  \rb{1+\rho^{-1}} (1-p) } \wL^2,   $$ $$p_{max} = p_{min} =  p.$$
	\textit{Case I:} let $\alpha = 1, p  = 1$, then the result holds trivially. 
	
	\textit{Case II:} let $0<\alpha < 1, p  = 1$, then $B = \beta \wL^2$ , $\theta_p = \theta = 1 - \sqrt{1-\alpha} \geq \fr{\al}{2}$ and \eqref{eq:pp_stepsize_simplify} follows by Lemma~\ref{le:optimal_t-Peter}. 
	
	\textit{Case III:} let $\alpha = 1$, and $ 0 < p < 1$, then $\theta = 1$ , $\beta = 0$ , $B = \rb{1+\rho^{-1}}(1-p) \wL^2 $, $\theta_p = p - \rho (1-p) $. Then the choice $\rho = \fr{p\al}{2(1-p)}$ simplifies
	$$
	\theta_p = \fr{p}{2},
	$$
	$$
	\fr{B}{\theta_p} = \fr{\rb{1+\rho^{-1}}(1-p) \wL^2}{p - \rho (1-p)}  = \fr{2 (1-p) \wL^2}{p } \rb{\fr{2}{p} - 1}  \leq \fr{4 \wL^2}{p^2}.
	$$
	\textit{Case IV:} let $0 < \alpha < 1$,and $ 0 < p < 1$.Then the choice of constants $\theta = 1 - (1-\al)\rb{1+s}$, $\beta = (1-\al)\rb{1+\fr{1}{s}}$, $\rho = \fr{p\al}{4(1-p)}$, $s = \fr{\al}{4(1-\al)}$ yields
	
	\begin{eqnarray*}
		p\rho+\theta p - \rho  = 	p (\rho +1 - (1-\al)\rb{1+s})-\rho  
		=  p \al - p(1-\al) s - (1-p) \rho  
		= \fr{1}{2} p \al .
	\end{eqnarray*}
	Also
	\[1+\fr{1}{s}= \fr{4-3\al}{\al}\leq \fr{4}{\al},\quad
	1+\fr{1}{\rho} = \fr{4(1-p) + \al p}{p\al} = \fr{4 - p (4- \al)}{p\al}\leq \fr{4}{p\al}.\]
	
	Thus, we can bound by a simple calculation
	\begin{eqnarray*}
		\fr{B}{\theta_p}= \fr{ p\beta +\rb{1-p}\rb{1+\fr{1}{\rho}}}{p(\rho+\theta)-\rho } \wL^2	
		&\le& \fr{16\wL^2	}{p^2\al^2}.
	\end{eqnarray*}
\end{proof}


\begin{theorem}\label{th:pp-ef21}
	Let Assumption~\ref{as:main} hold, and let the stepsize in Algorithm~\ref{alg:PP-EF21} be set as  
	\begin{equation} \label{eq:pp-ef21_stepsize}
		0<\gamma \leq \rb{L +  \sqrt{\fr{B}{\theta_p} } }^{-1}.
	\end{equation}
	Fix $T \geq 1$ and let $\hat{x}^{T}$ be chosen from the iterates $x^{0}, x^{1}, \ldots, x^{T-1}$  uniformly at random. Then
	\begin{equation} \label{eq:pp-ef21_rate}
		\Exp{\sqnorm{\nabla f(\hat{x}^{T})} } \leq \fr{2\left(f(x^{0})-\finf\right)}{\gamma T} + \frac{\Exp{G^0}}{\theta_p  T} 
	\end{equation}
	with $\theta_p =  {\rho p_{min} + \theta p_{max}  -\rho - \rb{p_{max} - p_{min}}}$, $B = {\suminn \rb{ \beta p_i  +  \rb{1+\rho^{-1}} (1-p_i) } L_i^2}   $, $p_{max} = \max_{1\leq i \leq n} p_i$,  $p_{min} = \min_{1\leq i \leq n} p_i$,   $\theta = 1 - (1+s) (1-\al)$, $\beta = \rb{1+\fr{1}{s}} (1-\al)$ and $\rho, s > 0$.
\end{theorem}

\begin{proof}
	
	By \eqref{eq:func_diff_distrib}, we have
	\begin{eqnarray}\label{eq:pp_ef21_0} 
		\Exp{\delta^{t+1}} &\leq& \Exp{\delta^t}-\frac{\gamma}{2} \Exp{\sqnorm{\nabla f(x^{t})}}  - \left(\frac{1}{2 \gamma}-\frac{L}{2}\right) \Exp{R^t}+ \frac{\gamma}{2}\Exp{G^t}.
	\end{eqnarray}
	Lemma~\ref{le:rec_2-pp} states that
	\begin{equation}\label{eq:pp_ef21_1}
		\Exp{G^{t+1}} \leq (1-\theta_p) \Exp{G^t} + B \Exp{R^t }
	\end{equation}
	with $\theta_p =  {\rho p_{min} + \theta p_{max}  -\rho - \rb{p_{max} - p_{min}}}$, $B = {\suminn \rb{ \beta p_i  +  \rb{1+\rho^{-1}} (1-p_i) } L_i^2}   $, $p_{max} = \max_{1\leq i \leq n} p_i$,  $p_{min} = \min_{1\leq i \leq n} p_i$,   $\theta = 1 - (1+s) (1-\al)$, $\beta = \rb{1+\fr{1}{s}} (1-\al)$ and small enough $\rho, s > 0$.
	
	Adding \eqref{eq:pp_ef21_0} with a $\frac{\gamma}{2\theta_2}$ multiple of \eqref{eq:pp_ef21_1} and rearranging terms in the right hand side, we have
	\begin{align*}
		\Exp{\delta^{t+1}}+\frac{\gamma}{2 \theta_p} \Exp{G^{t+1}} &\leq \Exp{ \delta^{t} } + \frac{\gamma}{2\theta_p} \Exp{ G^{t} }  - \frac{\gamma}{2} \Exp{ \sqnorm{\nabla f(x^{t}) }  }  - \rb{\frac{1}{2\g} -\frac{L}{2} - \frac{\g B }{2\theta} } \Exp{R^{t} } \\
		&\leq \Exp{ \delta^{t} } + \frac{\gamma}{2\theta_p} \Exp{ G^{t} }   - \frac{\gamma}{2} \Exp{ \sqnorm{\nabla f(x^{t}) }  }  .
	\end{align*}
	
	The last inequality follows from the bound $\g ^2\frac{B}{\theta_p} + L\g \leq 1,$ which holds because of Lemma \ref{le:stepsize_page_fact} and our assumption on the stepsize.
	By summing up and rearranging we get \eqref{eq:pp-ef21_rate}.
	
\end{proof}

\begin{corollary}\label{cor:pp-ef21}
	Let assumptions of Theorem~\ref{th:pp-ef21} hold, 
	\begin{eqnarray*}
		g_i^0 &=& \nabla f_i (x^0), \qquad i = 1,\ldots, n, \\
		\g &=&  \rb{L +  \sqrt{\fr{B}{\theta_p} } }^{-1},\\
		p_i &=& p , \qquad i = 1,\ldots, n,
	\end{eqnarray*}
	where $B$ and $\theta_p$ are given in Theorem~\ref{th:pp-ef21}. Then, after $T$ iterations/communication rounds of \algname{EF21-PP} we have $\Exp{\sqnorm{\nabla f(\hat{x}^{T})} } \leq \eps^2$. It requires 
	\begin{eqnarray}
		T = \# \text{grad} =   \cO\rb{    \fr{   \wL \delta^0}{ p \alpha \varepsilon^2} } \notag 
	\end{eqnarray}
	iterations/communications rounds/gradint computations at each node.
\end{corollary}

\begin{proof}
	Let $g_i^0 = \nabla f_i(x^0) $, $i = 1, \dots, n$ , then $G^0 = 0$ and by Theorem~\ref{th:pp-ef21}
	
	\begin{eqnarray*}
		\# \text{grad} &=& T \overset{(i)}{\leq }  \fr{  2\delta^0}{\g \varepsilon^2} \overset{(ii)}{\leq }  \fr{  2\delta^0}{ \varepsilon^2} \rb{L + \wL\sqrt{\frac{B}{\theta_p}}}
		\overset{(iii)}{\leq }  \fr{  2\delta^0}{ \varepsilon^2} \rb{L + \fr{4\wL}{p \alpha} }  \overset{(iv)}{\leq }   \fr{  2\delta^0}{ \varepsilon^2} \rb{\fr{\wL}{p\alpha} + \fr{4\wL  }{p\alpha} } =  \fr{  5 \wL \delta^0}{ p \alpha \varepsilon^2},
	\end{eqnarray*}
	where  $(i)$ is due to the rate \eqref{eq:pp-ef21_rate} given by Theorem~\ref{th:pp-ef21}. In two $(ii)$ we use the largest possible stepsize  \eqref{eq:pp-ef21_stepsize}, in $(iii)$ we utilize Lemma~\ref{le:stepsize_simplify}, and $(iv)$ follows by the  inequalities $\alpha \leq 1$, $p\leq 1$ and $L \leq \wL$.
\end{proof}

\newpage

\section{Bidirectional Compression}\label{sec:BC}


{\color{myblue} The main idea of the proof is to split the deviation error coming from worker's compressor and the server's compressor. That is we need to control the terms
	$$
	{\sqnorm{g^{t}- \wg^t }} \qquad \text{and} \quad  \sqnorm{ \wg^t - \nfxt } , 
	$$
	where 
	$$
	g^{t+1} = g^t + \cC_M( \wg^{t+1}  - g^t)\qquad \text{and} \quad  
	\wg_i^{t+1} = \wg_i^t + \cC_w( \nabla f_i(x^{t+1}) - \wg_i^t), \quad \wg^{t+1} =\frac{1}{n} \sum_{i=1}^n 	\wg_i^{t+1} .
	$$
	This is conceptually similar to the proof strategy for variance reduction extension, but the source of the second deviation error in this case is different and comes from server level compression rather than sampling stochastic gradients. }


\begin{lemma}\label{le:ineq_3_ef21_bc}
	Let Assumption~\ref{as:main} hold,  $\cC_w$ be a \textit{contractive compressor}, and $\wg_i^{t+1}$ be an EF21 estimator of $ \nfixtpo $, i. e. 
	\begin{equation}
		\wg_i^{t+1} = \wg_i^t + \cC_w( \nabla f_i(x^{t+1}) - \wg_i^t) \notag 
	\end{equation}
	for arbitrary $\wg_i^0$ and all all $i = 1, \dots, n$, $t\geq 0$.  Then 
	\begin{equation}\label{eq:ineq_3_ef21_bc}
		\Exp{P^{t+1}} \leq (1 - \theta_w ) \Exp{ P^t} + \beta_w \wL^2 \Exp{R^t },
	\end{equation}
	where $\theta_w \eqdef 1- (1- \alpha_w )(1+s), \quad \beta_w \eqdef (1- \alpha_w ) \left(1+ s^{-1} \right) \quad \text{for any } s > 0$, and $P_i^t =  \sqnorm{\wg_i^t - \nfixt} $, $P^t = \suminn P_i^t $.
\end{lemma}
\begin{proof}
	The proof is the same as for Lemma~\ref{lem:theta-beta}.
\end{proof}

\begin{lemma}\label{le:ineq_2_ef21_bc}
	Let Assumption~\ref{as:main} hold,  $\cC_M$, $\cC_w$ be \textit{contractive compressors}. Let $\wg_i^{t+1}$ be an EF21 estimator of $ \nfixtpo $, i. e.,$
		\wg_i^{t+1} = \wg_i^t + \cC_w( \nabla f_i(x^{t+1}) - \wg_i^t),
$
	and let $g^{t+1}$ be an EF21 estimator of $\wg^{t+1} = \suminn \wg_i^{t+1} $, i. e., 
$
		g^{t+1} = g^t + \cC_M( \wg^{t+1}  - g^t)
$
	for arbitrary $g^0$, $\wg_i^0$ and all $i = 1, \dots, n$, $t\geq 0$. Then 
	\begin{equation}\label{eq:ineq_2_ef21_bc}
		\Exp { \sqnorm{g^{t+1} - \wg^{t+1}} } \leq (1-\theta_M) \Exp { \sqnorm{g^{t} - \wg^{t}}  } + 8\beta_M \Exp {P^t } + 8 \beta_M  \wL^2  \Exp {R^t },
	\end{equation}
	where $g^{t} = \suminn g_i^{t}$, $\wg^{t} = \suminn \wg_i^t$, $\theta_M = 1- (1- \alpha_M )(1+\rho)$,  $\beta_M =  (1- \alpha_M ) \rb{ 1+ \rho^{-1} }$ for any $\rho > 0$ and $P_i^t =  \sqnorm{\wg_i^t - \nfixt} $, $P^t = \suminn P_i^t $.
\end{lemma}
\begin{proof}
	Similarly to the proof of Lemma~\ref{lem:theta-beta}, we derive
	\begin{equation}\label{eq:ef21-bc_ineq1}
		\Exp{  \sqnorm{g^{t+1} - \wg^{t+1}} } 
		 \leq   (1- \theta_M)\Exp{ \sqnorm{g^t - \wg^t}  } + \beta_M \Exp{\sqnorm{\wg^{t+1}  -  \wg^{t} }  } .
	\end{equation}
	
	Further we bound the last term in \eqref{eq:ef21-bc_ineq1}. Recall that
	\begin{eqnarray}\label{eq:bc-ef21_cit}
		\wg^{t+1} = \wg^t + \frac{1}{n} \sum_{i=1}^n c_i^t .
	\end{eqnarray}
	where $c_i^t = \cC_w(\nabla f_i(x^{t+1}) - \wg_i^t)$ and $\wg^t = \suminn \wg_i^t$. Then 
	\begin{eqnarray}\label{eq:ef21-bc_ineq2}
		\Exp{\sqnorm{\wgtpo - \wgt}} &\eqtext{\eqref{eq:bc-ef21_cit}}&  \Exp{\sqnorm{\wg^t + \frac{1}{n} \sum_{i=1}^n c_i^t  - \wgt} }
		=\Exp{\sqnorm{\frac{1}{n} \sum_{i=1}^n c_i^t }} 
		\overset{(i)}{\leq } \suminn \Exp{\sqnorm{ c_i^t } }\notag\\
		&=& \suminn \Exp{\sqnorm{c_i^t  -\rb{\nabla f_i(x^{t+1}) - \wg_i^t} + \rb{\nabla f_i(x^{t+1}) - \wg_i^t} } }\notag\\
		&\letext{\eqref{eq:facts:young_ineq2}}& 2\suminn \Exp{ \Exp{\sqnorm{\cC_w\rb{\nabla f_i(x^{t+1}) - \wg_i^t} -\rb{\nabla f_i(x^{t+1}) - \wg_i^t}}\mid W^t}  }\notag \\
		&& \qquad + 2\suminn \Exp{\sqnorm{\nabla f_i(x^{t+1}) - \wg_i^t } }\notag\\
		&\letext{\eqref{eq:b_compressor}}& 2(1 - \al_w)\suminn \Exp{\sqnorm{\nabla f_i(x^{t+1}) - \wg_i^t} } \notag \\
		&& \qquad + 2\suminn \Exp{\sqnorm{\nabla f_i(x^{t+1}) - \wg_i^t } }\notag\\
		&=& 2(2 - \al_w)\suminn \Exp{ \sqnorm{\nabla f_i(x^{t+1}) -   \wg_i^t } } \notag\\
		&\overset{(ii)}{< }& 4 \suminn \Exp{ \sqnorm{\nabla f_i(x^{t+1}) -   \wg_i^t } } \notag\\
		&=& 4\suminn \Exp{ \sqnorm{\nabla f_i(x^{t+1}) - \nfixt -  \rb{\wg_i^t - \nfixt}} } \notag\\
		&\letext{\eqref{eq:facts:young_ineq2}}& 8\suminn \sqnorm{\wg_i^t - \nfixt} + 8 \suminn \Exp{\sqnorm{\nabla f_i(x^{t+1}) - \nfixt }} \notag\\
		&\overset{(iii)}{\leq}& 8 \suminn \Exp{\sqnorm{\wg_i^t - \nfixt} } + 8 \wL^2 \Exp{\sqnorm{x^{t+1} - \xt }}\notag\\
		&{=}& 8 \Exp{P^t} + 8 \wL^2 \Exp{R^t}, 
	\end{eqnarray}
	where in $(i)$ we use \eqref{eq:facts:young_ineq3_avg}, $(ii)$ is due to $\al_w >0$,  $(iii)$ holds by Assumption~\ref{as:main}. In the last step we apply the definition of $P^t = \suminn \sqnorm{\wg_i^t - \nfixt} $, and $R^t = \sqnorm{x^{t+1} - \xt }$. Finally, plugging \eqref{eq:ef21-bc_ineq2} into \eqref{eq:ef21-bc_ineq1}, we conclude the proof. 
\end{proof}


\begin{theorem}\label{th:ef21-bc}
	Let Assumption~\ref{as:main} hold, and let the stepsize in Algorithm~\ref{alg:EF21-BC} be set as  
	\begin{equation}\label{eq:ef21-bc-stepsize}
		0<\gamma \leq \rb{L + \wL \sqrt{ \frac{16 \beta_M }{\theta_M} + \fr{2 \beta_w }{\theta_w} \rb{1+\fr{8\beta_M}{\theta_M}   }   }   }^{-1} 
	\end{equation}
	Fix $T \geq 1$ and let $\hat{x}^{T}$ be chosen from the iterates $x^{0}, x^{1}, \ldots, x^{T-1}$  uniformly at random. Then
	\begin{equation}\label{eq:ef21-bc_rate}
		\Exp{\sqnorm{\nabla f(\hat{x}^{T})} } \leq \fr{2 \Exp{\Psi^0 }}{\gamma T},
	\end{equation}
	where $\Psi^{t} \eqdef  f(x^t) - \finf + \frac{\gamma }{\theta_M} \sqnorm{g^t-\wg^t } + \fr{\g}{\theta_w} \rb{1+\fr{8\beta_M }{\theta_M}} P^{t}$,  $\wL = \sqrt{\frac{1}{n}\sum_{i=1}^n L_i^2}$, $\theta_w \eqdef 1- (1- \alpha_w )(1+s), \quad \beta_w \eqdef (1- \alpha_w ) \left(1+ s^{-1} \right) $, $\theta_M \eqdef 1- (1- \alpha_M )(1+\rho)$,  $\beta_M \eqdef  (1- \alpha_M ) \rb{ 1+ \rho^{-1} }$ for any $\rho, s > 0$.
\end{theorem}

\begin{proof}
	We apply Lemma \ref{le:aux_smooth_lemma} and split the error $\sqnorm{g^{t}-\nabla f (x^{t})}$ in two parts
	\begin{eqnarray}\label{eq:ef21-bc-ineq_1}
		f(x^{t+1}) 	&\leq&  f(x^{t})  -\fr{\g}{2} \sqnorm{\nf{\xt}} -\left(\fr{1}{2 \gamma}-\fr{L}{2}\right) R^t + \fr{\gamma}{2}\sqnorm{g^{t}-\nabla f(x^{t})}\notag \\
		&\leq&{ f(x^{t}) } -\fr{\g}{2}\sqnorm{\nf{\xt}} -\left(\fr{1}{2 \gamma}-\fr{L}{2}\right) {R^t} + \g {\sqnorm{g^{t}- \wg^t }} + \g\sqnorm{ \wg^t - \nfxt } \notag \\
		&\leq&{ f(x^{t}) } -\fr{\g}{2}{\sqnorm{\nf{\xt}}} -\left(\fr{1}{2 \gamma}-\fr{L}{2}\right) {R^t} \notag \\ 
		&& \qquad + \g \suminn{\sqnorm{g^{t}- \wg^t }} + \g \suminn{\sqnorm{ \wg_i^t - \nfixt }} \notag \\
		&=&{ f(x^{t}) } -\fr{\g}{2}{\sqnorm{\nf{\xt}}} -\left(\fr{1}{2 \gamma}-\fr{L}{2}\right) {R^t}  + \g {\sqnorm{g^{t}- \wg^t }} + \g {P^t} ,
	\end{eqnarray}
	where we used notation $R^t = \sqnorm{\g g^t} = \sqnorm{\xtpo - \xt}$, $P^t = \suminn \sqnorm{\wg_i^t - \nfixt}$ and applied \eqref{eq:facts:young_ineq2} and \eqref{eq:facts:young_ineq3_avg}.
	
	Subtracting $f^{\text {inf }}$ from both sides of the above inequality, taking expectation and using the notation $\delta^t = f(x^{t+1})-\finf$, we get
	\begin{eqnarray}\label{eq:ef21_bc_0} 
		\Exp{\delta^{t+1}} &\leq&\Exp{\delta^{t}} -\fr{\g}{2}\Exp{\sqnorm{\nf{\xt}}} -\left(\fr{1}{2 \gamma}-\fr{L}{2}\right) \Exp{R^t} + \g \Exp{\sqnorm{g^{t}- \wg^t }} + \g \Exp{P^t} . \notag \\
	\end{eqnarray}
	
	Further, Lemma~\ref{le:ineq_3_ef21_bc} and \ref{le:ineq_2_ef21_bc} provide the recursive bounds for the last two terms of (\ref{eq:ef21_bc_0})
	
	\begin{eqnarray}\label{eq:ef21_bc_2}
		\Exp{ P^{t+1}  } &\le& 	(1 - \theta_w) \Exp{P^t} + { \beta_w \wL^2} \Exp{R_t},
	\end{eqnarray}
	
	\begin{eqnarray}\label{eq:ef21_bc_1}
		\Exp{\sqnorm{g^{t+1}- \wg^{t+1} } }&\le& 	(1 - \theta_M) \Exp{\sqnorm{g^{t}- \wg^t }} + 8 \beta_M  \wL^2 \Exp{R_t} + 8\beta_M \Exp{P^t}  .
	\end{eqnarray}
	
	Summing up (\ref{eq:ef21_bc_0}) with a $\fr{\g}{\theta_M}$ multiple of (\ref{eq:ef21_bc_1}) we obtain
	\begin{eqnarray}
		\Exp{\delta^{t+1}}+\frac{\gamma }{\theta_M} \Exp{\sqnorm{g^{t+1}- \wg^{t+1} } } & \leq& \Exp{\delta^{t}}-\frac{\gamma}{2} \mathbb{E}\left[\left\|\nabla f\left(x^{t}\right)\right\|^{2}\right]-\left(\frac{1}{2 \gamma}-\frac{L}{2}\right) \Exp{R^{t}}   \notag \\
		&& \qquad +\g \Exp{\sqnorm{g^{t}- \wg^{t} }} +\gamma \Exp{P^{t}}  \notag \\
		&& \qquad +\frac{\g}{\theta_M}\rb{ \left(1-\theta_M \right) \Exp{\sqnorm{g^{t}- \wg^{t} }} } \notag \\
		&& \qquad + \frac{\g}{\theta_M}\rb {8\beta_M\wL^2 \Exp{R^{t}} + 8\beta_M \Exp{P^{t}}  }  \notag \\ 
		& \leq & \Exp{\delta^{t}} +\frac{\gamma }{\theta_M} \Exp{\sqnorm{g^{t}- \wg^{t} }}-\frac{\gamma}{2} \mathbb{E}\left[\left\|\nabla f\left(x^{t}\right)\right\|^{2}\right] \notag \\
		&& \qquad  -\left(\frac{1}{2 \gamma}-\frac{L}{2}-\frac{8\gamma  \beta_M\wL^2}{\theta_M}\right) \Exp{R^{t}}\notag \\
		&& \qquad+\gamma \rb{1+\fr{8\beta_M}{\theta_M}} \Exp{P^{t}}. \notag
	\end{eqnarray}
	
	Then adding the above inequality with a $\fr{\g}{\theta_w} \rb{1+\fr{8\beta_M}{\theta_M}} $ multiple of (\ref{eq:ef21_bc_2}), we get
	
	\begin{eqnarray}\label{eq:ef21_bc_summed_all}
		\Exp{\Psi^{t+1}} &=& \Exp{\delta^{t+1}}+\frac{\gamma }{\theta_M} \Exp{\sqnorm{g^{t+1}- \wg^{t+1} }} + \fr{\g}{\theta_w} \rb{1+\fr{8\beta_M}{\theta_M}} \Exp{P^{t+1}} \notag \\
		& \leq & \Exp{\delta^{t}}+\frac{\gamma }{\theta_M} \Exp{\sqnorm{g^{t}- \wg^{t} }}-\frac{\gamma}{2} \mathbb{E}\left[\left\|\nabla f\left(x^{t}\right)\right\|^{2}\right]-\left(\frac{1}{2 \gamma}-\frac{L}{2}-\frac{8 \gamma \beta_M \wL^2}{\theta_M}\right) \Exp{R^{t}}\notag \\
		&& \qquad+\gamma \rb{1+\fr{8\beta_M}{\theta_M}} \Exp{P^{t}} \notag \\
		&& \qquad + \fr{\g}{\theta_w} \rb{1+\fr{8\beta_M}{\theta_M}} \rb{ (1-\theta_w) \Exp{P^t} + \beta_w \wL^2 \Exp{R^t} } \notag \\
		& \leq& \Exp{\delta^{t}}+\frac{\gamma }{\theta_M} \Exp{\sqnorm{g^{t}- \wg^{t} }} + \fr{\g}{\theta_w} \rb{1+\fr{8\beta_M}{\theta_M}} \Exp{P^t} - \frac{\gamma}{2} \mathbb{E}\left[\left\|\nabla f\left(x^{t}\right)\right\|^{2}\right] \notag \\
		&& \quad -\left(\frac{1}{2 \gamma}-\frac{L}{2}-\frac{8 \gamma \beta_M \wL^2}{\theta_M} - \fr{\g}{\theta_w} \rb{1+\fr{8\beta_M}{\theta_M}} {\beta_w \wL^2} \right) \Exp{R^{t}} \notag \\
		& =& \Exp{\Psi^{t}}  - \frac{\gamma}{2} \mathbb{E}\left[\left\|\nabla f\left(x^{t}\right)\right\|^{2}\right] \notag \\
		&& \quad -\left(\frac{1}{2 \gamma}-\frac{L}{2}-\frac{8\gamma  \beta_M \wL^2}{\theta_M} - \fr{\g \beta_w \wL^2}{\theta_w} \rb{1+\fr{8\beta_M}{\theta_M}}  \right) \Exp{R^{t}} .
	\end{eqnarray}
	Thus by Lemma \ref{le:stepsize_page_fact} and the stepsize choice,
	the last term in \eqref{eq:ef21_bc_summed_all} is not positive. By summing up inequalities for $t =0, \ldots, T-1,$ and rearranging we get \eqref{eq:ef21-bc_rate}.
	
\end{proof}

\begin{corollary}\label{cor:ef21_bc}
	Let assumption of Theorem~\ref{th:ef21-bc} hold, 
	\begin{eqnarray*}
		g^0 &=& \nabla f(x^0),  \qquad \wg_i^0 = \nabla f_i (x^0), \qquad i = 1,\ldots, n, \\
		\g &=&  \rb{L + \wL \sqrt{ \frac{16 \beta_M }{\theta_M} + \fr{2 \beta_w }{\theta_w} \rb{1+\fr{8\beta_M}{\theta_M}   }   }   }^{-1} ,
	\end{eqnarray*}
	Then, after $T$ iterations/communication rounds of \algname{EF21-BC} we have $\Exp{\sqnorm{\nabla f(\hat{x}^{T})} } \leq \eps^2$. It requires 
	\begin{eqnarray*}
		T = \# \text{grad} =  \cO\rb{    \fr{   \wL \delta^0}{ \alpha_w \alpha_M \varepsilon^2} } 
	\end{eqnarray*}
	iterations/communications rounds/gradint computations at each node.
\end{corollary}

\begin{proof}
	Note that by Lemma~\ref{le:optimal_t-Peter} and $\al_M, \al_w \leq 1$, we have 
	\begin{eqnarray*}
		\frac{16 \beta_M }{\theta_M} + \fr{2 \beta_w }{\theta_w} \rb{1+\fr{8\beta_M}{\theta_M}   }  &\leq& 16\frac{4  }{\al_M^2} + 2\fr{ 4 }{\al_w^2} \rb{1+8\fr{4}{\al_M^2}   } 
		\leq \frac{64  }{\al_M^2} + \fr{ 8 }{\al_w^2} \fr{33}{\al_M^2}    
		\leq \frac{64  + 8 \cdot 33}{\al_w^2 \al_M^2}  .
	\end{eqnarray*}
	It remains to apply the steps similar to those in the proof of Corollary~\ref{cor:ef21}. 
\end{proof}

\newpage

\section{Heavy Ball Momentum}\label{sec:HB}

In this section, we study the momentum version of \algname{EF21}. In particular, we focus on Polyak style momentum \citep{Heavy-ball,Unified_momentum}. Let $g^t$ be a gradient estimator at iteration $t$ and $v^t$ is some vector, then the update rule of \textit{heavy ball} (HB) can be written as
$$
\left\{\begin{array}{l}
	x^{t+1} = x^t - \gamma v^t \\
	v^{t+1} = \eta v^t + g^{t+1},
\end{array}\right.
$$
where $\eta \in [0, 1)$ is the \textit{momentum parameter}, and $\g > 0$ is the stepsize. {\color{myblue} To combine this algorithm with \algname{EF21}, we use EF21 estimator to approximate $g^t$. The formal pseudocode in distributed setting is presented in Algorithm~\ref{alg:EF21_HB}. }




We present the convergence analysis results for this algorithm in Theorem~\ref{thm:HB} and Corollary~\ref{cor:ef21-hb}. We recall the notations used in this section: $R^t = \sqnorm{\g g^t} = (1-\eta)^2 \sqnorm{z^{t+1} - z^t}${\color{myblue}, $\delta^t = f(z^t)  - \finf$}. In the analysis of \algname{EF21-HB}, we assume by default that $v^{-1} = 0$.

\begin{lemma}\label{le:virtual_z_sequence}
	Let sequences  $\cb{x^t}_{t\geq 0}$ , and $\cb{v^t}_{t\geq 0}$ be generated by Algorithm~\ref{alg:EF21_HB} and let the sequence $\cb{z^t}_{t\geq 0}$ be defined as $z^{t+1} \eqdef x^{t+1} - \fr{\g \eta }{1-\eta} v^t$ with $0 \leq \eta < 1$. Then for all $t \geq 0$
	$$
	z^{t+1}	 = z^t - \fr{\g}{1-\eta} g^t.
	$$
\end{lemma}
\begin{proof}
	\begin{eqnarray*}
		z^{t+1} &\overset{(i)}{=} & x^{t+1} - \fr{\g \eta }{1-\eta} v^t 
		\overset{(ii)}{=}  x^{t} - \g v^t - \fr{\g \eta }{1-\eta} v^t  
		\overset{(iii)}{=}  z^{t} + \fr{\g \eta }{1-\eta} v^{t-1}  - \fr{\g  }{1-\eta} v^t  \\
		&= & z^{t} - \fr{\g  }{1-\eta} \rb{ v^{t}  - \eta v^{t-1}}  
		= z^{t} - \fr{\g  }{1-\eta} g^t ,
	\end{eqnarray*}
	where in $(i)$ and $(iii)$ we use the definition of $z^{t+1}$ and $z^t$, in $(ii)$ we use the step $x^{t+1} = x^t - \gamma v^t$ (line \ref{line:EF21_HB_master_step} of Algorithm~\ref{alg:EF21_HB}). Finally, the last equality follows by the update $v^{t+1} = \eta v^t + g^{t+1}$ (line \ref{line:EF21_HB_momentum_upd} of Algorithm~\ref{alg:EF21_HB}). 
	
\end{proof}

\begin{lemma}\label{le:bound_v_by_g}
	Let the sequence $\cb{v^t}_{t\geq 0}$  be defined as $v^{t+1} = \eta v^t + g^{t+1}$ with $0 \leq \eta < 1$. Then
	$$
	\sum_{t=0}^{T-1} \sqnorm{v^t} \leq \fr{1}{(1-\eta)^2}\sum_{t=0}^{T-1} \sqnorm{g^t} .
	$$
\end{lemma}
\begin{proof}
	Unrolling the given recurrence and noticing that $v^{-1} = 0$, we have $v^t = \sum_{l = 0}^{t} \eta^{t-l} g^l$. Define $H \eqdef \sum_{l = 0}^{t} \eta^{l} \leq \fr{1}{1-\eta} $. Then by Jensen's inequality
	\begin{eqnarray*}
		\sum_{t=0}^{T-1} \sqnorm{v^t}  &=& H^2 \sum_{t=0}^{T-1} \sqnorm{\sum_{l = 0}^{t} \fr{\eta^{t-l}}{H} g^l} 
		\leq H^2 \sum_{t=0}^{T-1} \sum_{l = 0}^{t} \fr{\eta^{t-l}}{H} \sqnorm{g^l} 
		= H \sum_{t=0}^{T-1} \sum_{l = 0}^{t} \eta^{t-l} \sqnorm{g^l} \\
		&\leq& \fr{1}{1-\eta} \sum_{t=0}^{T-1} \sum_{l = 0}^{t} \eta^{t-l} \sqnorm{g^l} 
		= \fr{1}{1-\eta} \sum_{l=0}^{T-1} \sqnorm{g^l}  \sum_{t = l}^{T-1} \eta^{t-l}  
		\leq \fr{1}{(1-\eta)^2} \sum_{t=0}^{T-1} \sqnorm{g^t}  .
	\end{eqnarray*}
	
\end{proof}

\begin{lemma}\label{le:rec_HB}
	Let the sequence $\cb{z^t}_{t\geq 0}$ be defined as $z^{t+1} \eqdef x^{t+1} - \fr{\g \eta }{1-\eta} v^t$ with $0 \leq \eta < 1$. Then 
	$$
	\sum_{t=0}^{T-1} \Exp{G^{t+1}}	\leq (1-\theta ) 	\sum_{t=0}^{T-1} \Exp{G^{t}}	+ 2 \beta \wL^2 (1+ 4\eta^2) 	\sum_{t=0}^{T-1} \Exp{\sqnorm{z^{t+1} - z^{t}}}	,
	$$
	where  $\theta = 1- (1- \alpha )(1+s), \quad \beta = (1- \alpha ) \left(1+ s^{-1} \right) \quad \text{for any } s > 0$.
\end{lemma}
\begin{proof}
	Summing up the inequality in Lemma~\ref{lem:theta-beta} (for EF21 estimator) for $t = 0, \dots, T-1$, we have 
	\begin{equation}\label{eq:summed_rec}
		\sum_{t=0}^{T-1} \Exp{G^{t+1}}	\leq (1-\theta ) 	\sum_{t=0}^{T-1} \Exp{G^{t}}	+ \beta \wL^2 \sum_{t=0}^{T-1} \Exp{\sqnorm{x^{t+1} - x^{t}}}	.
	\end{equation}
	It remains to bound $\sum_{t=0}^{T-1} \Exp{\sqnorm{x^{t+1} - x^{t}}}$. Notice that by definition of $\cb{z^t}_{t\geq 0}$, we have
	$$
	x^{t+1} - x^{t} = z^{t+1} - z^{t} + \fr{\g \eta }{1-\eta} \rb{ v^t - v^{t-1} }.
	$$
	Thus 
	\begin{eqnarray*}
		\sum_{t=0}^{T-1} \Exp{\sqnorm{x^{t+1} - x^{t}}} &\leq& 2 \sum_{t=0}^{T-1} \Exp{\sqnorm{z^{t+1} - z^{t}}}  +  \fr{2\g^2 \eta^2 }{(1-\eta)^2} \sum_{t=0}^{T-1} \Exp{\sqnorm{v^t - v^{t-1}}} \\
		&=& 2 \sum_{t=0}^{T-1} \Exp{\sqnorm{z^{t+1} - z^{t}}}  +  \fr{2\g^2 \eta^2 }{(1-\eta)^2} \sum_{t=0}^{T-1} \Exp{\sqnorm{g^t - (1-\eta ) v^{t-1}}} \\
		&\leq& 2 \sum_{t=0}^{T-1} \Exp{\sqnorm{z^{t+1} - z^{t}}}  +  \fr{4\g^2 \eta^2 }{(1-\eta)^2} \sum_{t=0}^{T-1} \Exp{\sqnorm{g^t }} \\
		&& \qquad+  \fr{4\g^2 \eta^2 }{(1-\eta)^2} \sum_{t=0}^{T-1} (1-\eta)^2 \Exp{\|v^{t-1}\|^2} .
		\end{eqnarray*}
		{\color{myblue} Next, using Lemma~\ref{le:bound_v_by_g} we continue bounding the last term above}
		\begin{eqnarray*}
		\sum_{t=0}^{T-1} \Exp{\sqnorm{x^{t+1} - x^{t}}}	&{\leq}& 2 \sum_{t=0}^{T-1} \Exp{\sqnorm{z^{t+1} - z^{t}}}  +  \fr{8\g^2 \eta^2 }{(1-\eta)^2} \sum_{t=0}^{T-1} \Exp{\sqnorm{g^t }} \\
		&\overset{(i)}{=}& 2 \sum_{t=0}^{T-1} \Exp{\sqnorm{z^{t+1} - z^{t}}}  +  8 \eta^2  \sum_{t=0}^{T-1} \Exp{\sqnorm{z^{t+1} - z^t }} \\
		&=& 2 (1 + 4\eta^2) \sum_{t=0}^{T-1} \Exp{\sqnorm{z^{t+1} - z^{t}}} ,
	\end{eqnarray*}
	where in $(i)$ we used Lemma~\ref{le:virtual_z_sequence}. It remains to plug in the above inequality into \eqref{eq:summed_rec}.
\end{proof}

\begin{lemma}\label{le:gradient_bound_HB}
	Let the sequence $\cb{z^t}_{t\geq 0}$ be generated as in Lemma~\ref{le:virtual_z_sequence}, i.e., $z^{t+1}	 = z^t - \fr{\g}{1-\eta} g^t$, then for all $t \geq 0$
	$$
	\sqnorm{\nfxt} \leq 2 G^t + \fr{2(1-\eta)^2}{\g^2}\sqnorm{z^{t+1} - z^t}
	$$	
	with $G^t =  \suminn \sqnorm{\nfixt - g_i^t} $.
\end{lemma}
\begin{proof}
	Notice that for $\g > 0$ we have $\nfxt  = \nfxt - g^t - \fr{1-\eta}{\g} (z^{t+1} - z^t )$. Then 
	
	\begin{eqnarray*}
		\sqnorm{\nfxt  }&\leq& 2 \sqnorm{\nfxt - g^t}  + 2 \fr{(1-\eta)^2}{\g^2} \sqnorm{z^{t+1} - z^t } \\
		&\leq & \frac{2}{n}\sum\limits_{i=1}^n \sqnorm{\nfixt - g_i^t} +  \fr{2(1-\eta)^2}{\g^2} \sqnorm{z^{t+1} - z^t },
	\end{eqnarray*}
where the inequalities hold due to \eqref{eq:facts:young_ineq2} with $s=1$, and \eqref{eq:facts:young_ineq3_avg}.
\end{proof}


\begin{theorem}\label{thm:HB}
	Let Assumption~\ref{as:main} hold, and let the stepsize in Algorithm~\ref{alg:EF21_HB} be set as
	\begin{equation} \label{eq:HB-stepsize}
		0<\gamma < \rb{ \fr{(1+ \eta)L}{2(1-\eta)^2}  + \fr{\wL}{1-\eta}\sqrt{\frac{2\beta}{\theta} \rb{1+ 4\eta^2}}}^{-1} \eqdef \g_0 ,
	\end{equation}
	where $0 \leq \eta < 1$, $\theta = 1- (1- \alpha )(1+s)$, $ \beta = (1- \alpha ) \left(1+ s^{-1} \right)$, and $s > 0$. 
	
	Fix $T \geq 1$ and let $\hat{x}^{T}$ be chosen from the iterates $x^{0}, x^{1}, \ldots, x^{T-1}$  uniformly at random. Then
	\begin{eqnarray}\label{eq:HB}		
 \Exp{\sqnorm{\nabla f(\hat x^T) } }	&\leq& \fr{3{\delta^0 }(1-\eta)}{T \g\rb{1- \fr{\g}{\g_0}}} + \fr{\Exp{G^0}}{\theta T } \rb{ 2 + \fr{1}{2 \lambda_1}\fr{3 (1-\eta)}{\g\rb{1- \fr{\g}{\g_0}}} }  ,
	\end{eqnarray}
	where  $\lambda_1 \eqdef \wL \sqrt{\frac{2\beta}{\theta} \rb{1+ 4\eta^2}}$.
	If the stepsize is set to $0 < \g \leq \nfr{\g_0}{2}, $ then
	\begin{eqnarray}\label{eq:HB_g_half}
		\sum_{t=0}^{T-1} \Exp{\sqnorm{\nfxt} }	&\leq& \fr{6{\delta^0 }(1-\eta)}{\g T} + \fr{\Exp{G^0}}{T \theta} \rb{ 2 + \fr{3 (1-\eta)}{\g \wL \sqrt{\fr{2\beta}{\theta} \rb{1 + 4 \eta^2}}} }  .
	\end{eqnarray}
\end{theorem}

\begin{proof}
	Consider the sequence $z^{t+1} \eqdef x^{t+1} - \fr{\g \eta }{1-\eta} v^t$ with $0 \leq \eta < 1$. Then Lemma~\ref{le:virtual_z_sequence} states that $z^{t+1}	 = z^t - \fr{\g}{1-\eta} g^t$. By $L$-smoothness of $f(\cdot)$
	\begin{eqnarray*}
		f(z^{t+1}) - f (z^{t}) &\leq& \la\nabla f(z^{t}), z^{t+1}-z^{t} \ra+\frac{L}{2}\sqnorm{z^{t+1}-z^{t}} \\
		&=& \la\nabla f(z^{t}) - g^t, z^{t+1}-z^{t} \ra + \la g^t, z^{t+1}-z^{t} \ra + \frac{L}{2}\sqnorm{z^{t+1}-z^{t}} \\ 
		&\overset{(i)}{=}& \la\nabla f(z^{t}) - g^t, z^{t+1}-z^{t} \ra - \fr{1-\eta}{\g} \sqnorm{z^{t+1}-z^{t}} + \frac{L}{2}\sqnorm{z^{t+1}-z^{t}} \\ 
		&{=}& \la\nabla f(z^{t}) - g^t, z^{t+1}-z^{t} \ra - \rb{\fr{1-\eta}{\g} - \frac{L}{2} } \sqnorm{z^{t+1}-z^{t}} \\ 
		&&\qquad - \rb{\fr{1-\eta}{\g} - \frac{L}{2} } \sqnorm{z^{t+1}-z^{t}} ,
			\end{eqnarray*}
			{\color{myblue} where in $(i)$ Lemma~\ref{le:virtual_z_sequence} is applied. Next, using Young's inequality \eqref{eq:facts:young_ineq} twice with $\lambda_1, \lambda_2 > 0$, we have }
			\begin{eqnarray*}
			f(z^{t+1}) - f (z^{t}) &{\leq}& \fr{1}{2\lambda_1}\sqnorm{\nabla f(x^{t}) - g^t} + \fr{\lambda_1}{2} \sqnorm{z^{t+1}-z^{t} }  + \fr{1}{2\lambda_2}\sqnorm{\nabla f(z^{t}) - \nabla f(x^{t})} \\
		&&\qquad  + \fr{\lambda_2}{2} \sqnorm{z^{t+1}-z^{t} }  - \rb{\fr{1-\eta}{\g} - \frac{L}{2} } \sqnorm{z^{t+1}-z^{t}} \\  
		&{=}& \fr{1}{2\lambda_1}\sqnorm{\nabla f(x^{t}) - g^t} + \fr{1}{2\lambda_2}\sqnorm{\nabla f(z^{t}) - \nabla f(x^{t})} \\
		&&\qquad   - \rb{\fr{1-\eta}{\g} - \frac{L}{2} - \fr{\lambda_1}{2}  - \fr{\lambda_2}{2}} \sqnorm{z^{t+1}-z^{t} } \\  
		&\overset{(i)}{\leq}& \fr{1}{2\lambda_1}\sqnorm{\nabla f(x^{t}) - g^t} + \fr{L^2}{2\lambda_2}\sqnorm{z^{t} - x^{t}} \\
		&&\qquad   - \rb{\fr{1-\eta}{\g} - \frac{L}{2} - \fr{\lambda_1}{2}  - \fr{\lambda_2}{2}} \sqnorm{z^{t+1}-z^{t} } \\  
		&\overset{(ii)}{\leq}& \fr{1}{2\lambda_1}\sqnorm{\nabla f(x^{t}) - g^t} + \fr{\g^2 \eta^2 L^2}{2\lambda_2 (1-\eta)^2}\sqnorm{v^{t-1}} \\
		&&\qquad   - \rb{\fr{1-\eta}{\g} - \frac{L}{2} - \fr{\lambda_1}{2}  - \fr{\lambda_2}{2}} \sqnorm{z^{t+1}-z^{t} } ,
	\end{eqnarray*}
	{\color{myblue} where $(i)$ holds by smoothness (Assumption~\ref{as:main}),} and $(ii)$ holds by  definition of $z^{t} = x^t - \fr{\g \eta }{1-\eta} v^{t-1}$. Summing up the above inequalities for $t = 0, \dots, T-1$ (assuming $v^{-1} = 0$), we have
	\begin{eqnarray*}
		f(z^{T})  &\leq&  f (z^{0}) + \fr{1}{2\lambda_1} \sum_{t=0}^{T-1}\sqnorm{\nabla f(x^{t}) - g^t} + \fr{\g^2 \eta^2 L^2}{2\lambda_2 (1-\eta)^2} \sum_{t=0}^{T-1} \sqnorm{v^{t}} \\
		&&\qquad   - \rb{\fr{1-\eta}{\g} - \frac{L}{2} - \fr{\lambda_1}{2}  - \fr{\lambda_2}{2}} \sum_{t=0}^{T-1} \sqnorm{z^{t+1}-z^{t} } \\
		&\overset{(i)}{\leq}&  f (z^{0}) + \fr{1}{2\lambda_1} \sum_{t=0}^{T-1}\sqnorm{\nabla f(x^{t}) - g^t} + \fr{\g^2 \eta^2 L^2}{2\lambda_2 (1-\eta)^4} \sum_{t=0}^{T-1} \sqnorm{g^{t}} \\
		&&\qquad   - \rb{\fr{1-\eta}{\g} - \frac{L}{2} - \fr{\lambda_1}{2}  - \fr{\lambda_2}{2}} \sum_{t=0}^{T-1} \sqnorm{z^{t+1}-z^{t} } \eqdef (*) ,
			\end{eqnarray*}
		{\color{myblue} where $(i)$ holds due to Lemma~\ref{le:bound_v_by_g}. Further, using Lemma~\ref{le:virtual_z_sequence} in $(ii)$, we get }
		\begin{eqnarray*}
		(*) &\overset{(ii)}{=}&  f (z^{0}) + \fr{1}{2\lambda_1} \sum_{t=0}^{T-1}\sqnorm{\nabla f(x^{t}) - g^t} + \fr{\g^2 \eta^2 L^2}{2\lambda_2 (1-\eta)^4} \sum_{t=0}^{T-1} \fr{(1-\eta)^2}{\g^2}\sqnorm{ z^{t+1}-z^{t} } \\
		&&\qquad   - \rb{\fr{1-\eta}{\g} - \frac{L}{2} - \fr{\lambda_1}{2}  - \fr{\lambda_2}{2}} \sum_{t=0}^{T-1} \sqnorm{z^{t+1}-z^{t} } \\
		&{=}&  f (z^{0}) + \fr{1}{2\lambda_1} \sum_{t=0}^{T-1}\sqnorm{\nabla f(x^{t}) - g^t}  \\
		&&\qquad   - \rb{\fr{1-\eta}{\g} - \frac{L}{2} - \fr{\lambda_1}{2}  - \fr{\lambda_2}{2} - \fr{ \eta^2 L^2}{2\lambda_2 (1-\eta)^2} } \sum_{t=0}^{T-1} \sqnorm{z^{t+1}-z^{t} } \eqdef (**).
			\end{eqnarray*}
				{\color{myblue} 	Finally, in $(iii)$ we use that $\sqnorm{\nabla f(x^{t}) - g^t} \leq G^t$ and further derive }
			\begin{eqnarray*}
		(**) &\overset{(iii)}{\leq}&  f (z^{0}) + \fr{1}{2\lambda_1} \sum_{t=0}^{T-1} G^t    - \rb{\fr{1-\eta}{\g} - \frac{L}{2} - \fr{\lambda_1}{2}  - \fr{\lambda_2}{2} - \fr{ \eta^2 L^2}{2\lambda_2 (1-\eta)^2} } \sum_{t=0}^{T-1} \sqnorm{z^{t+1}-z^{t} }\\
		&{=}&  f (z^{0}) + \fr{1}{2\lambda_1} \sum_{t=0}^{T-1}G^t   - \rb{\fr{1-\eta}{\g} - \frac{L}{2} - \fr{\lambda_1}{2}   - \fr{ \eta L}{ (1-\eta)} } \sum_{t=0}^{T-1} \sqnorm{z^{t+1}-z^{t} } \\
		&{=}&  f (z^{0}) + \fr{1}{2\lambda_1} \sum_{t=0}^{T-1}G^t  - \rb{\fr{1-\eta}{\g} - \frac{L}{2} - \fr{\lambda_1}{2}   - \fr{ \eta L}{ (1-\eta)} } \fr{1}{(1-\eta)^2}\sum_{t=0}^{T-1} R^t,
	\end{eqnarray*}
	where in the last two steps we choose $\lambda_2 = \fr{\eta L}{1-\eta}$, and recall the definition $R^t = \sqnorm{\g g^t} = (1-\eta)^2\sqnorm{z^{t+1} - z^t}$.
	
	Subtracting $f^{\text {inf }}$ from both sides of the above inequality, taking expectation and using the notation $\delta^t = f(z^{t})-\finf$, we get
	\begin{eqnarray}\label{eq:func_diff_HB}
		\Exp{\delta^{T}} &\leq& \Exp{\delta^0} + \fr{1}{2\lambda_1} \sum_{t=0}^{T-1}\Exp{G^t }   - \rb{\fr{1-\eta}{\g} - \frac{L}{2} - \fr{\lambda_1}{2}   - \fr{ \eta L}{ (1-\eta)} }  \fr{1}{(1-\eta)^2}\sum_{t=0}^{T-1} \Exp{ R^t }. \notag\\
	\end{eqnarray}
	By Lemma~\ref{le:rec_HB}, we have
	\begin{eqnarray}\label{eq:rec_HB}
		\sum_{t=0}^{T-1} \Exp{G^{t+1}}	\leq (1-\theta ) 	\sum_{t=0}^{T-1} \Exp{G^{t}}	+ \fr{2 \beta \wL^2 (1+ 4\eta^2) }{(1-\eta)^2}	\sum_{t=0}^{T-1} \Exp{R^t}	.
	\end{eqnarray}
	Next, we are going to add \eqref{eq:func_diff_HB} with a $\fr{1}{2\theta\lambda_1}$ multiple of \eqref{eq:rec_HB}. First, let us "forget", for a moment, about all the terms involving $R^t$ and denote their sum appearing on the right hand side by $\mathcal{R}$, then
	\begin{eqnarray*}
		\Exp{\delta^{T}} + \fr{1}{2\theta\lambda_1} \sum_{t=0}^{T-1} \Exp{G^{t+1}} &\leq& \Exp{\delta^0} + \fr{1}{2\lambda_1} \sum_{t=0}^{T-1}\Exp{G^t } + (1-\theta)		\fr{1}{2\lambda_1}	\sum_{t=0}^{T-1} \Exp{G^{t}} + \mathcal{R}\\
		&=&  \Exp{\delta^0} + \fr{1}{2\theta\lambda_1} \sum_{t=0}^{T-1}\Exp{G^t } + \mathcal{R}.
	\end{eqnarray*}
	Canceling out the same terms in both sides of the above inequality, we get
	\begin{eqnarray*}
		\Exp{\delta^{T}} + \fr{1}{2\theta\lambda_1} \Exp{G^{T}} &\leq&  \Exp{\delta^0} + \fr{1}{2\theta\lambda_1}\Exp{G^0 } + \mathcal{R},
	\end{eqnarray*}
	where $\mathcal{R} \eqdef - \rb{\fr{1-\eta}{\g} - \frac{L}{2} \rb{1+ \fr{2\eta}{1-\eta}}-  \fr{\lambda_1}{2}   - \fr{\beta \wL^2 (1+4\eta^2)}{\theta \lambda_1} }  \fr{1}{(1-\eta)^2}\sum_{t=0}^{T-1} \Exp{ R^t }$. 
	
	Now choosing $\lambda_1 = \wL \sqrt{\fr{2\beta}{\theta} (1+4\eta^2)}$ and using the definition of $\g_0$ given by \eqref{eq:HB-stepsize}, i.e., $
	\g_0 \eqdef \rb{ \fr{(1+ \eta)L}{2(1-\eta)^2}  + \fr{\wL}{1-\eta}\sqrt{\frac{2\beta}{\theta} \rb{1+ 4\eta^2}}}^{-1} 
	$, we can compute 
	
	\begin{eqnarray*}
		\rb{\fr{1-\eta}{\g} - \frac{L}{2} \rb{1+ \fr{2\eta}{1-\eta}}-  \fr{\lambda_1}{2}   - \fr{\beta \wL^2 (1+4\eta^2)}{\theta \lambda_1} }  \fr{1}{(1-\eta)^2} 
		 = \rb{\fr{1}{\g} - \fr{1}{\g_0}}\fr{1}{1-\eta}.
	\end{eqnarray*}
	Then
	\begin{eqnarray*}
		0 \leq \Exp{\Phi^T} & \eqdef &\Exp{\delta^{T} + \fr{1}{2\theta\lambda_1} G^T}  
		\leq \Exp{\delta^0 + \fr{1}{2\theta\lambda_1} G^0} - \rb{\fr{1}{\g} - \fr{1}{\g_0}}\fr{1}{1-\eta}   \sum_{t=0}^{T-1} \Exp{ R^t }\\
		&=& \Exp{\Phi^0} - \rb{\fr{1}{\g} - \fr{1}{\g_0}}\fr{1}{1-\eta}   \sum_{t=0}^{T-1} \Exp{ R^t }.
	\end{eqnarray*}
	After rearranging, we get
	
	\begin{eqnarray*}
		\fr{1}{\g^2}\sum_{t=0}^{T-1}\Exp{ R^t } \leq \fr{\Exp{\Phi^0} (1-\eta)}{\g\rb{1- \fr{\g}{\g_0}}}.
	\end{eqnarray*}
	Summing the result of Lemma~\ref{le:gradient_bound_HB} over $t = 0, \dots, T-1$ and applying expectation, we get
	\begin{eqnarray*}
		\sum_{t=0}^{T-1} \Exp{\sqnorm{\nfxt} } \leq 2 \sum_{t=0}^{T-1}\Exp{G^t} + \fr{2}{\g^2}\sum_{t=0}^{T-1} \Exp{R^t}.
	\end{eqnarray*}
	Due to Lemma~\ref{le:rec_HB}, the conditions of Lemma~\ref{le:psi_1_cumul_bound} hold with $C \eqdef 2\beta \wL^2 \fr{1+ 4 \eta^2}{(1-\eta)^2}$, $s^t = \Exp{G^t} $, $r^t = \Exp{R^t} $, thus
	
	\begin{eqnarray*}
		\sum_{t=0}^{T-1}\Exp{G^t} \leq \fr{\Exp{G^0}}{\theta}  + \fr{C}{\theta}\sum_{t=0}^{T-1} \Exp{R^t}.
	\end{eqnarray*}
	Combining the above inequalities, we can continue with
	\begin{eqnarray*}
		\sum_{t=0}^{T-1} \Exp{\sqnorm{\nfxt} } &\leq& 2 \sum_{t=0}^{T-1}\Exp{G^t} + \fr{2}{\g^2}\sum_{t=0}^{T-1} \Exp{R^t} \\ 
		&\leq& \fr{2\Exp{G^0}}{\theta} + \rb{ 2 + \fr{\g^2 C}{\theta}  }\fr{1}{\g^2}\sum_{t=0}^{T-1} \Exp{R^t} \\ 
		&\leq& \fr{2\Exp{G^0}}{\theta} + \rb{ 2 + \fr{\g^2 C}{\theta}  } \fr{\Exp{\Phi^0 }(1-\eta)}{\g\rb{1- \fr{\g}{\g_0}}}.
	\end{eqnarray*}
	Note that for $\g < \g_0 =  \rb{\fr{(1+ \eta)L}{2(1-\eta)^2}  +\sqrt{\fr{C}{\theta}}}^{-1} $, we have $\fr{\g^2 C}{\theta} < 1$. Thus 
	\begin{eqnarray*}
		\sum_{t=0}^{T-1} \Exp{\sqnorm{\nfxt} }	&\leq& \fr{2\Exp{G^0}}{\theta} +  \fr{3\Exp{\Phi^0 }(1-\eta)}{\g\rb{1- \fr{\g}{\g_0}}} \\
		&=& \fr{3{\delta^0 }(1-\eta)}{\g\rb{1- \fr{\g}{\g_0}}} + \fr{\Exp{G^0}}{\theta} \rb{ 2 + \fr{1}{2 \lambda_1}\fr{3 (1-\eta)}{\g\rb{1- \fr{\g}{\g_0}}} }  ,
	\end{eqnarray*}
	where $\lambda_1 = \wL \sqrt{\fr{2\beta}{\theta} (1+4\eta^2)}$.
\end{proof}

\begin{corollary}\label{cor:ef21-hb}
	Let assumptions of Theorem~\ref{thm:HB} hold, 
	\begin{eqnarray*}
		g_i^0 &=& \nabla f_i (x^0), \qquad i = 1,\ldots, n, \\
		\g &=&  \rb{ \fr{(1+ \eta)L}{2(1-\eta)^2}  + \fr{\wL}{1-\eta}\sqrt{\frac{2\beta}{\theta} \rb{1+ 4\eta^2}}}^{-1}  .
	\end{eqnarray*}
	Then, after $T$ iterations/communication rounds of \algname{EF21-HB} we have $\Exp{\sqnorm{\nabla f(\hat{x}^{T})} } \leq \eps^2$. It requires 
	\begin{eqnarray}
		T = \# \text{grad} =   \cO\rb{    \fr{   \wL \delta^0}{  \varepsilon^2}  \rb{\fr{1}{\al} + \fr{1}{1-\eta}}}  \notag 
	\end{eqnarray}
	iterations/communications rounds/gradint computations at each node.
\end{corollary}

\begin{proof}
	Notice that by using $L\leq \wL$,  $\eta<1$ and Lemma~\ref{le:optimal_t-Peter}, we have
	\begin{eqnarray*}
		\fr{(1+ \eta)L}{2(1-\eta)^2}  + \fr{\wL}{1-\eta}\sqrt{\frac{2\beta}{\theta} \rb{1+ 4\eta^2}} 	&\leq& \fr{\wL }{(1-\eta)^2}  + \fr{\wL}{1-\eta}\sqrt{\frac{10 \beta}{\theta} } 
		\leq \fr{\wL}{1-\eta}\rb{  \fr{1}{1-\eta} + \fr{2 \sqrt{10}}{\al}}.
	\end{eqnarray*}
	Using the above inequality, \eqref{eq:HB_g_half}, and \eqref{eq:HB-stepsize}, we get
	
	\begin{eqnarray*}
		\# \text{grad} = T \leq   \fr{ 6 \delta^0 (1-\eta)}{\g \varepsilon^2} 
		\leq   \fr{  6 \delta^0 (1-\eta)}{ \varepsilon^2} \fr{\wL}{1-\eta}\rb{  \fr{1}{1-\eta} + \fr{2 \sqrt{10}}{\al}} 		
		\leq  \fr{  6 \wL \delta^0 }{ \varepsilon^2} \rb{  \fr{1}{1-\eta} + \fr{2 \sqrt{10}}{\al}} .
	\end{eqnarray*}
	
\end{proof}

\newpage

\section{Composite Setting}\label{sec:prox_setting}
Now we focus on solving a composite optimization problem 
\begin{equation}\label{eq:composite_optimization}
	\min_{x\in \R^d}\Phi(x) \eqdef \frac{1}{n} \sum \limits_{i=1}^n f_i(x) + r(x),
\end{equation}
where each $f_i(\cdot)$ is $L_i$-smooth (possibly non-convex), $r(\cdot)$ is convex, and $\Phi^{\inf} = \inf_{x\in\R^d} \Phi(x) > -\infty$. This is a standard and important generalization of problem \eqref{eq:finit_sum}. In particular, it includes optimization problems over convex compact sets and $l_1$-regularization (LASSO).

For any $\g>0$, $x\in\R^d$, recall that the proximal mapping of function $r(\cdot)$ (prox-operator) is defined as 
\begin{equation}\label{eq:def_prox}
	\opn{prox}_{\g r}(x) = \underset{y \in \mathbb{R}^{d}}{\arg \min }\left\{r(y)+\frac{1}{2 \g}\|y-x\|^{2}\right\}.
\end{equation}

To evaluate convergence in composite case, we define the \textit{generalized gradient mapping} at a point $x \in \R^d$ with a parameter $\g$ 
\begin{equation*}
	\mathcal{G}_{\g}(x) \eqdef \fr{1}{\g} \rb{ x-\opn{prox}_{\g r}(x-\g \nabla f(x)) }.
\end{equation*}
One can verify that the above quantity is a well-defined evaluation metric~\citep{BeckBook}. Namely, for any $x^* \in \R^d$, it holds that $\mathcal{G}_{\g}(x) = 0 $ if and only if $x^*$ is a stationary point of (\ref{eq:composite_optimization}), and in a special case when $r \equiv 0$, we have $\mathcal{G}_{\g}(x) = \nfx$.


{\color{myblue} When there is no compression, the convergence analysis of proximal gradient descent (see, e.g., Section 10.3 in \citep{BeckBook}) consists in showing a descent lemma with repect to the squared norm of gradient mapping, i.e., for any $x^t$
$$
\Phi(x^t) - \Phi(x^{t+1}) \geq  \gamma \rb{1-\gamma L /2 } \sqnorm{\mathcal{G}_{\g}(x)} .
$$
However, when there is a non-trivial compression, such inequality may not hold. The main idea of the analysis below is to upper bound the squared norm of gradient mapping $\sqnorm{\mathcal{G}_{\g}(x)}$ with certain error terms proportional to $\sqnorm{\xtpo - \xt}$ and $\sqnorm{g^t - \nfxt}$, which can be controlled using recursions from EF21 analysis (Lemma~\ref{lem:theta-beta}).
}

\begin{lemma}[Gradient mapping bound]\label{le:gradient_mapping_bound}
	Let $\xtpo \eqdef \opn{prox}_{\g r}(\xt - \g v^t)$, then 
	\begin{equation*}
		\Exp{\sqnorm{\cG_{\g}\rb{x^t}}} \leq \fr{2}{\g^2 } \Exp{\sqnorm{\xtpo - \xt}} + 2  \Exp{\sqnorm{v^t - \nfxt}}.
	\end{equation*}	
\end{lemma}
\begin{proof}
	\begin{eqnarray}
		\Exp{\sqnorm{\cG_{\g}\rb{x^t}}} &=& \fr{1}{\g^2 }  \Exp{\sqnorm{ \xt - \opn{prox}_{\g r}(\xt - \g \nfxt)}} \notag \\
		& \leq & \fr{2}{\g^2 }  \Exp{\sqnorm{ \xtpo - \xt}} + \fr{2}{\g^2 }  \Exp{\sqnorm{ \xtpo - \opn{prox}_{\g r}(\xt - \g \nfxt)}}  \notag \\
		& \leq & \fr{2}{\g^2 }  \Exp{\sqnorm{ \xtpo - \xt}}  + \fr{2}{\g^2 }  \Exp{\sqnorm{ (\xt - \g v^t) - (\xt - \g \nfxt)}}  \notag \\
		& = & \fr{2}{\g^2 } \Exp{\sqnorm{ \xtpo - \xt}} + 2  \Exp{\sqnorm{ v^t - \nfxt)}}  , \notag 
	\end{eqnarray}
	where in the last inequality we apply non-expansiveness of prox-operator.
\end{proof}

\begin{lemma}\label{le:prox-le}
	Let $\xtpo \eqdef \opn{prox}_{\g r}(\xt - \g v^t)$, then for any $\lambda > 0$,
	\begin{eqnarray*}
		\Phi \rb{\xtpo } &\leq& 	\Phi \rb{\xt }  + \fr{1}{2\lambda} \sqnorm{ v^t - \nfxt  } - \rb{ \fr{1}{\g} - \fr{L}{2} - \fr{\lambda}{2} } \sqnorm{\xtpo - \xt}. 
	\end{eqnarray*}
\end{lemma}
\begin{proof}
	Define $\tilde{r}(x) \eqdef r(x) + \fr{1}{2\g} \sqnorm{x - \xt + \g v^t}$, and note that $\xtpo = \argmin_{x\in \R^d} \cb{\tilde{r}(x)}$. Since $\tilde{r}(\cdot)$ is $\nfr{1}{\g}$ - strongly convex, we have
	\begin{equation}
		\tilde{r}(\xt) \geq \tilde{r}(\xtpo)  + \fr{1}{2\g} \sqnorm{\xtpo - \xt}, \notag
	\end{equation}
	\begin{equation}
		r(\xt)  + \fr{1}{2\g}\sqnorm{\g v^t}\geq r(\xtpo)  + \fr{1}{2\g} \sqnorm{\xtpo - \xt + \g v^t} + \fr{1}{2\g} \sqnorm{\xtpo - \xt}. \notag
	\end{equation}
	Thus 
	\begin{equation}\label{eq:prox-le-inter}
		r(\xtpo) - r(\xt) \leq -\fr{1}{\g}\sqnorm{\xtpo - \xt} - \la v^t, \xtpo - \xt\ra.
	\end{equation}
	
	By $L$ smoothness of $f(\cdot)$, 
	\begin{equation}\label{eq:L_smooth1}
		f\left(x^{t+1}\right)-f\left(x^{t}\right) \leq\left\langle\nabla f\left(x^{t}\right), x^{t+1}-x^{t}\right\rangle+\frac{L}{2}\left\|x^{t+1}-x^{t}\right\|^{2}.
	\end{equation}
	
	Summing up (\ref{eq:L_smooth1}) with (\ref{eq:prox-le-inter}) we obtain
	\begin{eqnarray}
		\Phi \rb{\xtpo } -	\Phi \rb{\xt } &\leq&\la \nfxt - v^t , \xtpo - \xt \ra - \rb{ \fr{1}{\g} - \fr{L}{2} } \sqnorm{\xtpo - \xt}  \notag \\
		&\leq& \fr{1}{2\lambda} \sqnorm{\nfxt - v^t } - \rb{ \fr{1}{\g} - \fr{L}{2} - \fr{\lambda}{2} } \sqnorm{\xtpo - \xt}.  \notag 
	\end{eqnarray}
	
\end{proof}



\begin{theorem}\label{th:prox-ef21}
	Let Assumption~\ref{as:main} hold, $r(\cdot)$ be convex and  $\Phi^{\inf} = \inf_{x\in\R^d} \Phi(x) > -\infty$. Set the stepsize in Algorithm~\ref{alg:Prox-ef21} as 
	\begin{equation*}
		0 < \g < \rb{\fr{ L}{2}   + \wL \sqrt{\fr{\beta}{\theta} }}^{-1} \eqdef \g_0,
	\end{equation*}
	where $\wL = \sqrt{\frac{1}{n}\sum_{i=1}^n L_i^2}$, $\theta = 1- (1- \alpha )(1+s)$,  $\beta =  (1- \alpha ) \left(1+ s^{-1} \right)$ for any $s > 0$.
	
	Fix $T \geq 1$ and let $\hat{x}^{T}$ be chosen from the iterates $x^{0}, x^{1}, \ldots, x^{T-1}$  uniformly at random. Then
	\begin{eqnarray*} 	
		\Exp{\sqnorm{\mathcal{G}_{\g}(\hat{x}^T) }} &\leq& \fr{4 \rb{\Phi^0 - \Phi^{\inf}}}{ T \g \rb{1-\fr{\g}{\g_0}}}  + \fr{2 \Exp{G^0}}{\theta T}  \rb{1  + \fr{1 }{  \g \rb{1-\fr{\g}{\g_0}} }  \fr{1}{ \wL }\sqrt{\fr{\theta}{\beta}} } .
	\end{eqnarray*}
	If the stepsize is set to $0 < \g \leq \nfr{\g_0}{2}, $ then
	\begin{equation*} 	
		\Exp{\sqnorm{\mathcal{G}_{\g}(\hat{x}^T) }} \leq 
		\fr{8 \rb{\Phi^0 - \Phi^{\inf}}}{ \g T}  + \fr{2  \Exp{G^0}}{\theta T} \rb{1 + \fr{2}{\g \wL} \sqrt{\fr{\theta }{\beta}}} .
	\end{equation*}
\end{theorem}

\begin{proof}
	First, let us apply  Lemma~\ref{le:prox-le} with $v^t = g^t, \lambda >0$
	\begin{eqnarray*}
		\Phi\left(x^{t+1}\right) \leq \Phi\left(x^{t}\right)+\frac{1}{2 \lambda}\left\|g^{t}-\nabla f\left(x^{t}\right)\right\|^{2}-\left(\frac{1}{\gamma}-\frac{L}{2}-\frac{\lambda}{2}\right)\left\|x^{t+1}-x^{t}\right\|^{2}.
	\end{eqnarray*}
	Subtract $\Phi^{\inf}$ from both sides, take expectation, and define $\delta^t = {\Phi \rb{\xt } -	\Phi^{\inf} }$ , \\$G^t = \suminn { \sqnorm{ g_i^t  - \nfixt  } }$, $R^t =  {\sqnorm{\xtpo - \xt } }$, then
	\begin{eqnarray}\label{eq:ineq_1_prox_ef21}
		\Exp{\delta^{t+1}} &\leq& \Exp{\delta^{t}} - \rb{ \fr{1}{\g} - \fr{L}{2} - \fr{\lambda}{2} }\Exp{R^t} + \fr{1}{2\lambda} \Exp{G^t} .   
	\end{eqnarray}
	Note that the proof of Lemma~\ref{lem:theta-beta} does not rely on the update rule for  $x^{t+1}$, but only on the way the estimator $g_i^{t+1}$ is constructed. Therefore, \eqref{eq:rec_1_avg} also holds for the composite case 
	\begin{eqnarray}\label{eq:ineq_2_prox_ef21}
		\Exp{G^{t+1}}&\leq& (1-\theta) \Exp{G^t} + \beta \wL^2 \Exp{R^t}.
	\end{eqnarray}
	Adding  (\ref{eq:ineq_1_prox_ef21}) with a $\fr{1}{2 \theta \lambda }$ multiple of (\ref{eq:ineq_2_prox_ef21}) , we obtain
	\begin{eqnarray}
		\Exp{\delta^{t+1}} + \fr{1}{2 \theta \lambda } \Exp{G^{t+1}} &\leq& \Exp{ \delta^{t}  } + \fr{1}{2 \theta \lambda } \Exp{ G^{t}}   - \rb{ \fr{1}{\g} - \fr{L}{2} - \fr{\lambda}{2} - \fr{\beta}{2 \theta \lambda } \wL^2} \Exp{R^t}  .\notag
	\end{eqnarray}	
	By summing up the inequalities for $t =0, \ldots, T-1,$ and rearranging, we get
	\begin{eqnarray}\label{eq:sumof_rt_bound_ef21}
		\sum_{t=0}^{T-1} \Exp{ R^t} &\leq& \rb{\delta^{0}  + \fr{1}{2 \theta \lambda } \Exp{  G^0 }} \rb{\fr{1}{\g} - \fr{L}{2} - \fr{\lambda}{2}  - \fr{\beta}{2 \theta \lambda } \wL^2 }^{-1} \notag \\
		&=& \rb{\delta^{0}  + \fr{1}{2 \theta}\sqrt{\fr{\theta}{\beta \wL^2 }} \Exp{ G^0 } } \rb{\fr{1}{\g} - \fr{L}{2} - \sqrt{\fr{\beta}{\theta} \wL^2 } }^{-1} 
		= \g^2 F^0 B . 
	\end{eqnarray}
	where in the first equality we choose $\lambda = \sqrt{\fr{\beta}{\theta}  \wL^2}$, and in the second we define $F^0 \eqdef \delta^{0}  + \fr{1}{2 \theta}\sqrt{\fr{\theta}{\beta \wL^2 }} \Exp{ G^0} $,  $B \eqdef \rb{\g - \fr{L\g^2}{2} - \sqrt{\fr{\beta}{\theta}  \wL^2} \g^2 }^{-1} = \rb{ \g - \fr{\g^2}{\g_0}}^{-1}$.
	
	By Lemma~\ref{le:gradient_mapping_bound} with $v^t = g^t$ we have 
	\begin{eqnarray}
\fr{1}{T} \sum_{t=0}^{T-1} \Exp{\sqnorm{\cG_{\g}\rb{x^t}}} 
		&\leq& \fr{2}{\g^2 T} \sum_{t=0}^{T-1} \Exp{ R^t} + \fr{2}{T} \sum_{t=0}^{T-1} \Exp{ G^t } \notag \\
		&\overset{(i)}{\leq}& \fr{2}{\g^2 T} \sum_{t=0}^{T-1} \Exp{ R^t} + \fr{2}{T} \fr{\Exp{ G^0}}{\theta} + \fr{2}{T} \fr{\beta \wL^2}{\theta} \sum_{t=0}^{T-1} \Exp{ R^t} \notag \\
		&\overset{(ii)}{\leq}& \fr{2 F^0 B}{ T} + \fr{2}{T} \fr{\Exp{ G^0}}{\theta} + \fr{2}{T} \fr{\beta \wL^2}{\theta} \g^2 F^0 B  \notag \\
		&=& \fr{2 F^0}{ T \g \rb{ 1 - \fr{\g}{\g_0}}}  \rb{1 + \fr{\g^2 \beta \wL^2}{\theta}} + \fr{2}{T} \fr{\Exp{ G^0}}{\theta}, \notag
	\end{eqnarray}
	where in $(i)$ we apply Lemma~\ref{le:psi_1_cumul_bound} with $C \eqdef \beta \wL^2$, $s^t \eqdef \Exp{ G^t}$, $r^t \eqdef \Exp{ R^t}$. $(ii)$ is due to (\ref{eq:sumof_rt_bound_ef21}).
	
	Note that for $\g < \rb{\fr{ L}{2}   + \sqrt{\fr{\beta}{\theta} }\wL}^{-1} $, we have $\fr{\g^2 \beta \wL^2}{\theta} < 1$. Thus
	\begin{eqnarray*}
		\Exp{\sqnorm{\mathcal{G}_{\g}(\hat{x}^T) }} &\leq &   \fr{4 \delta^0}{ T \g \rb{ 1 - \fr{\g}{\g_0}} }   + \fr{2 \Exp{ G^0}}{\theta T}   + \fr{2 \Exp{ G^0}}{T \g \rb{ 1 - \fr{\g}{\g_0}} }  \fr{1}{ \theta}\sqrt{\fr{\theta}{\beta \wL^2}} .  
	\end{eqnarray*}
	Set $\g \leq \nfr{\g_0}{2} $, then the bound simplifies to 
	\begin{eqnarray*}
		\Exp{\sqnorm{\mathcal{G}_{\g}(\hat{x}^T) }} &\leq &
		\fr{8  \delta^0 }{ \g T}  + \fr{2 \Exp{ G^0}}{\theta T} \rb{1 + \fr{2}{\g  } \sqrt{\fr{\theta }{\beta \wL^2}}} .
	\end{eqnarray*}
	
\end{proof}

\newpage
\section{Useful Lemma}

\begin{lemma}[Basic Facts]
For all $a, b, x_{1}, \ldots, x_{n} \in \mathbb{R}^{d}, s>0$ and $p \in(0,1]$ the following inequalities hold
\begin{eqnarray}
	\langle a, b\rangle &\leq& \frac{\|a\|^{2}}{2 s}+\frac{s\|b\|^{2}}{2}, \label{eq:facts:young_ineq}\\
	\|a+b\|^{2} &\leq& (1+s)\|a\|^{2}+(1+1 / s)\|b\|^{2}, \label{eq:facts:young_ineq2}\\
	\sqnorm{\suminn x_i } &\leq& \suminn \sqnorm{x_{i}},\label{eq:facts:young_ineq3_avg}\\
	\left(1-\frac{p}{2}\right)^{-1} &\leq& 1+p, \label{eq:1-p_2_in_denominator}\\
	\left(1+\frac{p}{2}\right)(1-p) &\leq& 1-\frac{p}{2}, \label{eq:1+p_2_1-p}\\
	\log\rb{1 - p } &\leq&  -  p.  \label{eq:facts:log_bound}
\end{eqnarray}
\end{lemma}

\begin{lemma}[Lemma 5 of \citep{EF21}]\label{le:stepsize_page_fact}
	If $0 \leq \gamma \leq \fr{1}{\sqrt{a}+b}$, then $a \gamma^{2}+b \gamma \leq 1$. Moreover, the bound is tight up to the factor of 2 since $\fr{1}{\sqrt{a}+b} \leq \min \left\{\fr{1}{\sqrt{a}}, \fr{1}{b}\right\} \leq \fr{2}{\sqrt{a}+b}$.
\end{lemma}
\begin{lemma}[Lemma 2 of \citep{PAGE2021}]\label{le:aux_smooth_lemma}
	Suppose that function $f$ is $L$-smooth and let $x^{t+1}\eqdef x^{t}-\g g^{t} ,$ where $g^t\in \R^d$ is any vector, and $\g>0$ any scalar. Then we have
	\begin{eqnarray*}
		f(x^{t+1}) \leq f(x^{t})-\fr{\g}{2}\sqnorm{\nabla f(x^{t})}-\left(\fr{1}{2 \g}-\fr{L}{2}\right)\sqnorm{x^{t+1}-x^{t}}+\fr{\g}{2}\sqnorm{g^{t}-\nabla f(x^{t})}.
	\end{eqnarray*}
\end{lemma}

\begin{lemma}[Lemma 3 of \citep{EF21}]\label{le:optimal_t-Peter} Let $0<\alpha< 1$ and for $s>0$ let $\theta(s)$ and $\beta(s)$ be defined as 
	\begin{eqnarray*}
		\theta(s) &\eqdef& 1 - (1 - \al)(1 + s), \qquad
		\beta(s)  \eqdef (1 - \al)(1 + s^{-1}).\notag
	\end{eqnarray*} Then the solution of the optimization problem
	\begin{equation*}
		 \min_{s} \left\{ \frac{\beta(s)}{\theta(s)} \;:\; 0<s<\frac{\alpha}{1-\alpha}\right\}
	\end{equation*}
	is given by $s^* = \frac{1}{\sqrt{1-\alpha}}-1$. Furthermore, $\theta(s^*) = 1-\sqrt{1-\alpha}$, $\beta(s^*) = \frac{1-\alpha}{1-\sqrt{1-\alpha}}$ and
	\begin{equation*}
		\sqrt{\frac{\beta(s^*)}{\theta(s^*)}} = \frac{1}{\sqrt{1-\alpha}} -1 = \frac{1}{\alpha} + \frac{\sqrt{1-\alpha}}{\alpha} - 1 \leq \frac{2}{\alpha}-1.
	\end{equation*}
	In the trivial case $\al = 1$, we have $\frac{\beta(s)}{\theta(s)} = 0$ for any $s > 0$, and above inequality is satisfied.
\end{lemma}

\begin{lemma}\label{le:psi_1_cumul_bound}
	Let (arbitrary scalar) non-negative sequences $\{s^t\}_{t\geq 0}$, and $\{r^t\}_{t\geq 0}$ satisfy 
	$$\sum_{t=0}^{T-1} s^{t+1} \leq (1-\theta ) \sum_{t=0}^{T-1} s^t + C  \sum_{t=0}^{T-1} r^t $$
	for some parameters $\theta\in (0, 1]$, $C > 0$. Then for all $T\geq 0$
	\begin{equation*}
		\sum_{t=0}^{T-1} s^t \leq \fr{s^0}{\theta}  + \fr{C}{\theta}  \sum_{t=0}^{T-1} r^t.
	\end{equation*}
\end{lemma}
{\color{myblue} \begin{proof}
	The proof follows immediately by canceling out the common terms on both sides and then dividing by $\theta > 0$.
\end{proof}
}

\newpage

\section{Extra Experiments}\label{sec:exp_extra}

In this section, we give missing details on the experiments from Section~\ref{sec:exp}, and  provide additional experiments. 

\subsection{Non-Convex Logistic Regression: Additional Experiments and Details} \label{sec:exp:logistic}

{\bf Data sets, hardware and implementation.}	We use standard LibSVM data sets \citep{chang2011libsvm}, and  split each data set among $n$ clients. 
For experiments $1$, $3$, $4$ and $5$, we chose $n = 20$ whereas for the experiment $2$ we consider $n=100$.
The first $n-1$ clients own equal parts, and the remaining part, of size $N - n \cdot\lfloor\nicefrac{N}{n}\rfloor$, is assigned to the last client. We consider the heterogeneous data distribution regime (i.e. we do not make any additional assumptions on data similarity between workers).  A summary of data sets and details of splitting data among workers can be found in Tables~\ref{tbl:datasets_summary_1} and ~\ref{tbl:datasets_summary_2}.   The algorithms are implemented in Python 3.8; we use 3 different CPU cluster node types in all experiments: 1) AMD EPYC 7702 64-Core; 2) Intel(R) Xeon(R) Gold 6148 CPU @ 2.40GHz; 3) Intel(R) Xeon(R) Gold 6248 CPU @ 2.50GHz. In all algorithms involving compression, we use Top-$k$ \citep{alistarh2017qsgd} as a canonical example of contractive compressor $\cC$, and fix the compression ratio $\nfr{k}{d} \approx 0.01$, where $d$ is the number of features in the data set. For all algorithms, at each iteration we compute the squared norm of the exact/full gradient for comparison of the methods performance. We terminate our algorithms either if they reach the certain number of iterations or the following stopping criterion is satisfied:  $\sqnorm{\nf{\xt}} \le 10^{-7}$. 

In all experiments, the stepsize is set to the largest stepsize predicted by theory for \algname{EF21} multiplied by some constant multiplier which was individually tuned in all cases. 

\begin{table}[h]
	\caption{Summary of the data sets and splitting of the data among clients for Experiments $1$, $3$, $4$, and $5$. Here $N_i$ denotes the number of datapoints per client.} 
	\label{tbl:datasets_summary_1}
	\centering
	\begin{tabular}{l l r r r r}
		\toprule
		Data set  & $n$ & $N$ (total \# of datapoints) & $d$ (\# of features)  &k & $N_i$ \\
		\midrule 			
		\texttt{mushrooms} & 20  & 8,120    & 112  &  2& 406		\\  
		\texttt{w8a}  				 & 20  &49,749  & 300  & 2 & 2,487		\\
		\texttt{a9a} 				  & 20  &32,560  & 123  &  2 &1,628 		\\
		\texttt{phishing}       & 20   &11,055  & 68    &   1&552		\\	
		\texttt{{real-sim}}       & {20}   &{72,309}  & {20,958}    &   {210}& {3615}		\\		
		\bottomrule
	\end{tabular}
\end{table}

\paragraph{Experiment 1: Fast convergence with variance reductions (extra details).}\label{subsecexp:ef21_page-ef21_sgd_ext}

The parameters $p_i$ of the PAGE estimator are set to $p_i = p \eqdef  \suminn\fr{\tau_i}{\tau_i + N_i}$, where $\tau_i$ is the batchsize for clients $i=1,\dots,n$ (see Table~\ref{tbl:page_p} for details). In our experiments, we assume that the sampling of Bernoulli random variable is performed on server side (which means that at each iteration for all clients $b_i^t=1$ or $b_i^t=0$). And if $b_i^t=0$, then in line $5$ of Algorithm~\ref{alg:EF21-PAGE} $I_i^t$ is sampled without replacement uniformly at random. Table~\ref{tbl:page_p} shows the selection of parameter $p$ for each experiment. 

For each batchsize from the set\footnote{By $50\%, 25\% $ (and so on) we refer to a batchsize, which is equals to $\lfloor0.5 N_i\rfloor$, $\lfloor0.25N_i\rfloor$ (and so on) for all clients $i=1,\dots,n$.} $$\cb{95 \%, 50\%,	25\%,12.5\%, 6.5\%,3\%},$$ we tune the stepsize multiplier for \algname{EF21-PAGE} within the set
$$\cb{0.25, 0.5, 1, 2, 4, 8, 16, 32, 64, 128, 256, 512, 1024, 2048}.$$

The best pair (batchsize, stepsize multiplier) is chosen in such a way that it gives the best convergence in terms of ${\rm \# bits}/n (C\rightarrow S)$. In the rest of the experiments, fine tuning is performed in a similar fashion.

\begin{table}[H]
	\caption{Summary of the parameter choice of  $p$. } 
	\label{tbl:page_p}
	\centering
	\begin{tabular}{l l l l }
		\toprule
		Data set  & $25\%$ & $12.5\%$ & $1.5\%$  \\
		\midrule 			
		\texttt{mushrooms} & 0.1992 &  0.1097   & 0.0146	\\ 
		\texttt{w8a}  				 & 0.1998  & 0.1108 &	0.0147\\
		\texttt{a9a} 				  &  0.2 &0.1109 &	0.0145\\
		\texttt{phishing}       &   0.2 & 0.1111& 0.0143 \\ 
		\texttt{real-sim}       & 0.1999   &  0.1109&   0.0147\\  			
		\bottomrule
	\end{tabular}
\end{table}

\begin{figure*}
	\centering
		\includegraphics[width=0.9\linewidth]{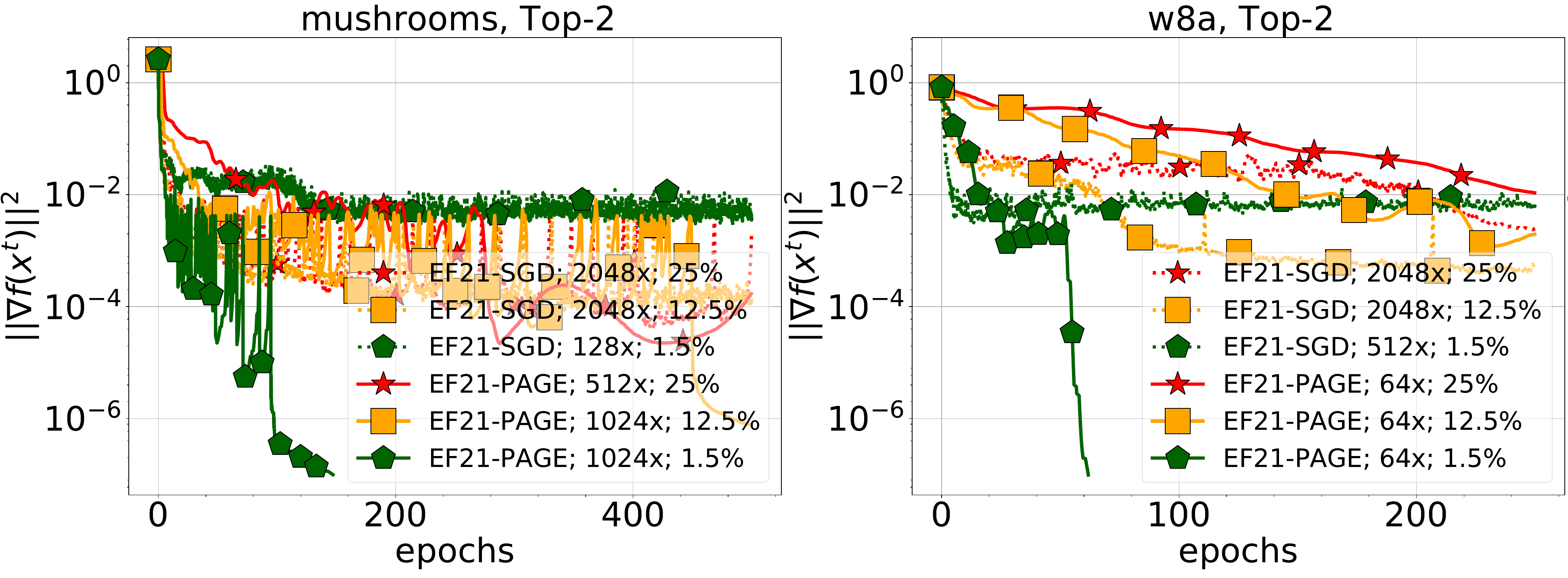}

\includegraphics[width=0.9\linewidth]{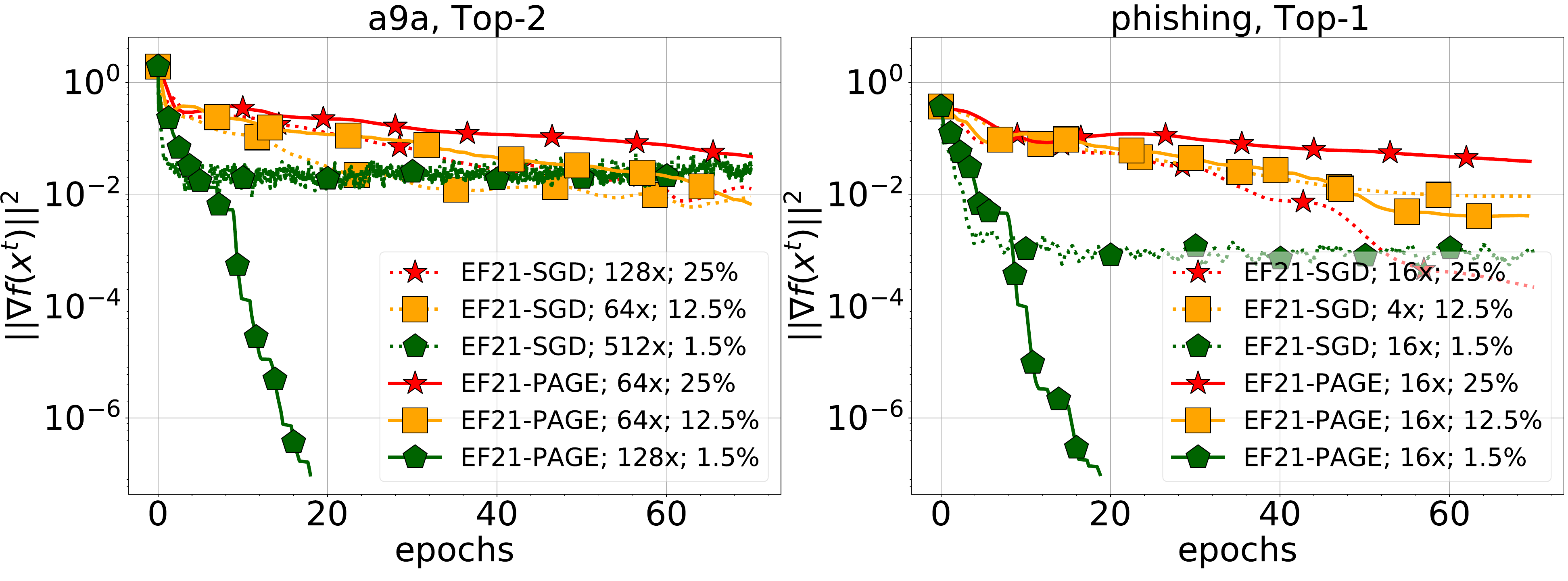}
	\caption{
		Comparison of \algname{EF21-PAGE} and \algname{EF21-SGD} with tuned step-sizes. By $1\times, 2\times, 4\times$ (and so on) we indicate that the stepsize was set to a multiple of  the largest stepsize predicted by theory for \algname{EF21}. By $25\%$, $12.5\% $ and $1.5\% $ we refer to batch-sizes equal $\lfloor0.25 N_i\rfloor$, $\lfloor0.125N_i\rfloor$ and $\lfloor0.015N_i\rfloor$ for all clients $i=1,\dots,n$, where $N_i$ denotes the size of local data set. }\label{fig:ef21_page_epochs}%
\end{figure*}

\paragraph{Experiment 2: On the effect of partial participation of clients (extra details).}\label{subsecexp:ef21_pp-ef21_fg_ext}
In this experiment, we consider $n=100$ and, therefore, a different data partitioning, see Table~\ref{tbl:datasets_summary_2} for the summary.
\begin{table}[h]
	\caption{Summary of the data sets and splitting of the data among clients for  Experiment 5. Here $N_i$ denotes the number of datapoints per client.} 
	\label{tbl:datasets_summary_2}
	\centering
	\begin{tabular}{l l r r r r}
		\toprule
		Data set  & $n$ & $N$ (total \# of datapoints) & $d$ (\# of features)  & k& $N_i$ \\
		\midrule 			
		\texttt{mushrooms} & 100  & 8,120    & 112  & 2&  81		\\ 
		\texttt{w8a}  				 & 100 &49,749  & 300  &   2&497		\\
		\texttt{a9a} 				  & 100  &32,560  & 123  &   2& 325 		\\
		\texttt{phishing}       & 100   &11,055  & 68    &  1 &110  \\  			
		\bottomrule
	\end{tabular}
\end{table}

We tune the stepsize multiplier for \algname{EF21-PP} within the following set:
$$\cb{0.125,0.25, 0.5, 1, 2, 4, 8, 16, 32, 64, 128, 256, 512, 1024, 2048, 4096}.$$

\begin{figure*}
	\centering
		\includegraphics[width=0.9\linewidth]{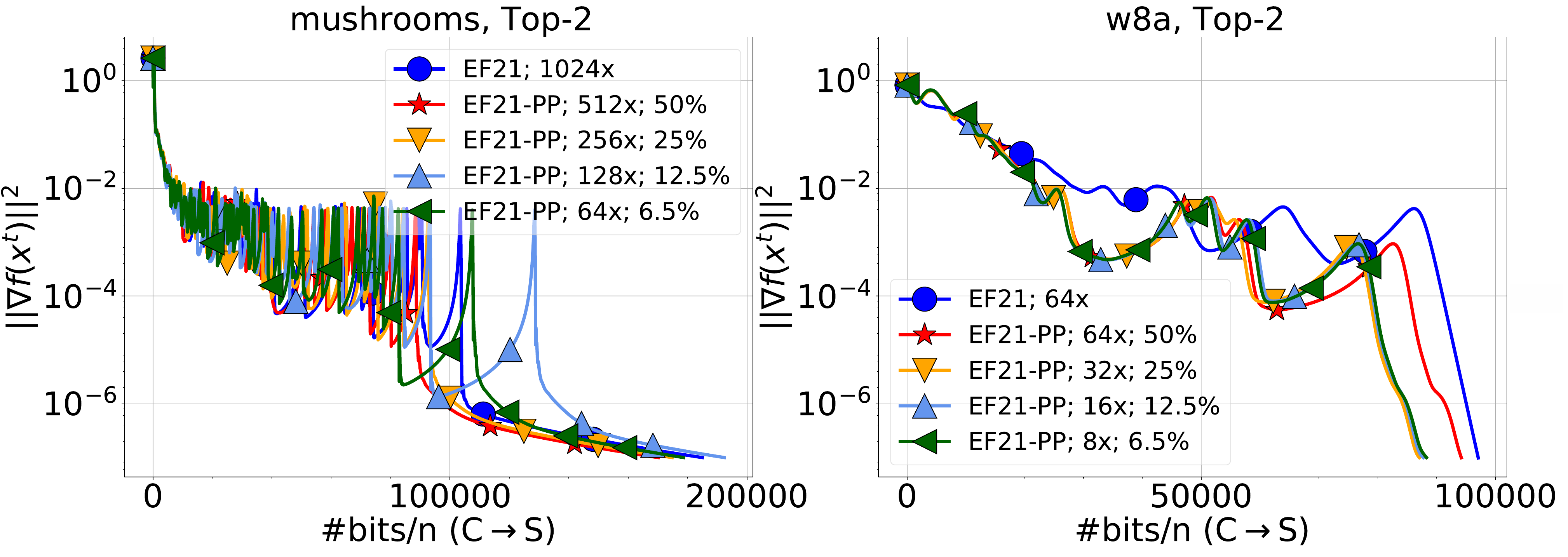}

		\includegraphics[width=0.9\linewidth]{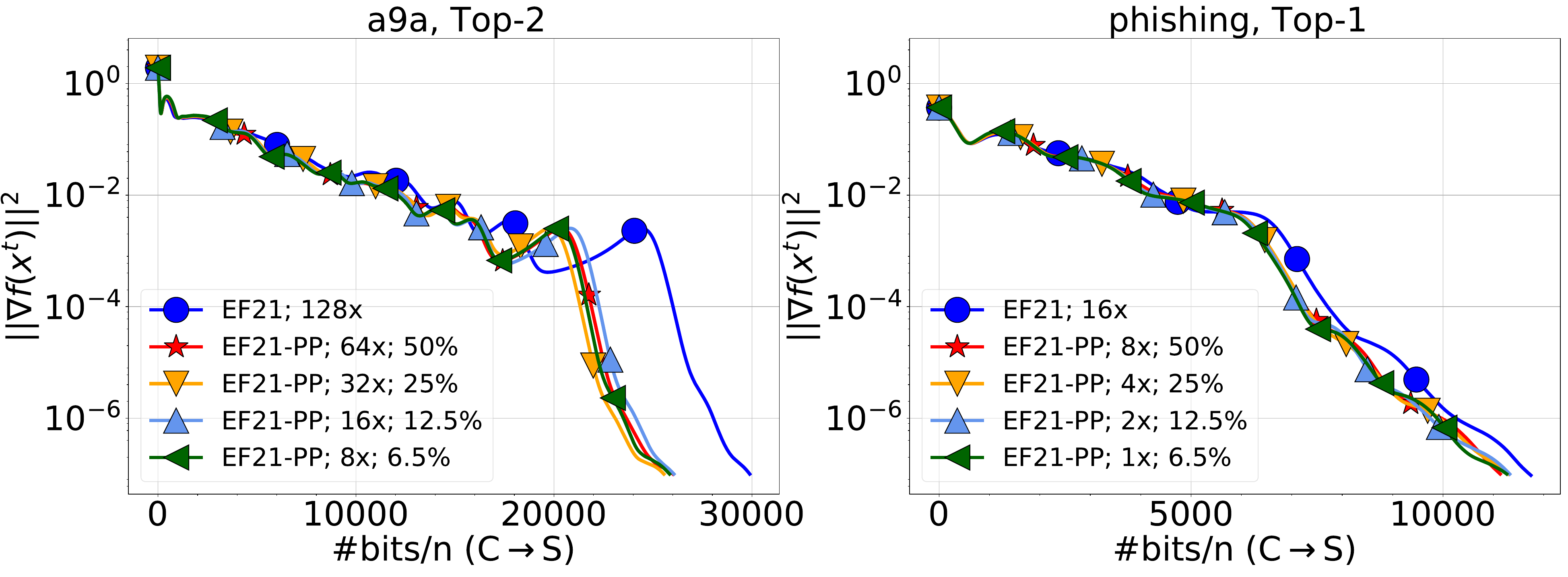} 
	\caption{
		Comparison of \algname{EF21-PP} and \algname{EF21} with tuned step-sizes. By $1\times, 2\times, 4\times$ (and so on) we indicate that the stepsize was set to a multiple of  the largest stepsize predicted by theory for \algname{EF21}. By $50\%$, $25\%$ , $12.5\%$ and $6.5\%$ we refer to a number of participating clients equal to $\lfloor0.5 n\rfloor$, $\lfloor0.25 n\rfloor$, $\lfloor0.125 n\rfloor$ and $\lfloor0.065 n\rfloor$. }\label{fig:ef21_pp_bits}%
\end{figure*}

\paragraph{Experiment 3: On the advantages of bidirectional biased compression (extra details).}\label{subsecexp:ef21_bc-ef21_fg}
Our next experiment demonstrates that the application of the \textbf{S}erver $\rightarrow$ \textbf{C}lients  compression in \algname{EF21-BC} (Alg.~\ref{alg:EF21-BC}) does not significantly slow down the convergence in terms of the communication rounds but requires much less bits to be transmitted. Indeed, Figure \ref{fig:ef21_bc_comm_realsim}, \ref {fig:ef21_bc_comm} illustrates that that it is sufficient to communicate only $5\% - 15\%$ of data to perform similarly to \algname{EF21} (Alg.~\ref{alg:EF21}).\footnote{The range $5\% - 15\%$  comes from the fractions $\nicefrac{k}{d}$ for each data set.} Note that \algname{EF21} communicates full vectors from the \textbf{S}erver $\rightarrow$ \textbf{C}lients, and, therefore, may have slower communication at each round. In Figure \ref{fig:ef21_bc_od_bits_realsim}, \ref{fig:ef21_bc_od_bits} we take into account only the number of bits sent from clients to the server, and therefore we observe the same behavior as in Figure \ref{fig:ef21_bc_comm}. However, if we consider the total number of bits (see Figure~\ref{fig:ef21_bc_bd_bits_realsim}, \ref{fig:ef21_bc_bd_bits}), then \algname{EF21-BC} considerably outperforms \algname{EF21} in all cases.

\begin{figure}[H]
	\centering
	\begin{subfigure}{0.48\textwidth}
		\includegraphics[width=0.9\linewidth]{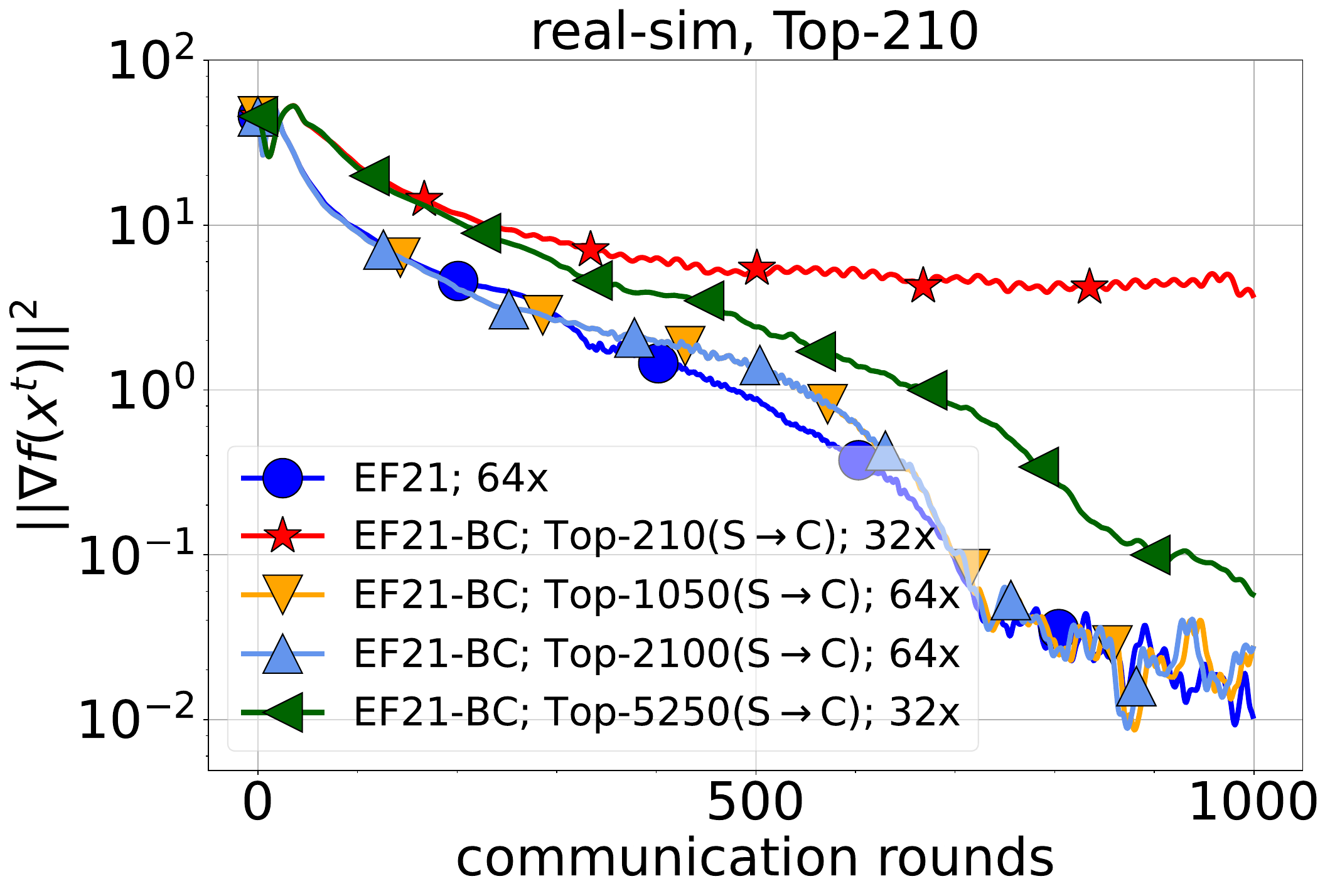}
		\caption{Convergence in communication rounds.\\${}$\\${}$\\${}$}
		\label{fig:ef21_bc_comm_realsim}
	\end{subfigure}
	\begin{subfigure}{0.48\textwidth}
		\includegraphics[width=0.9\linewidth]{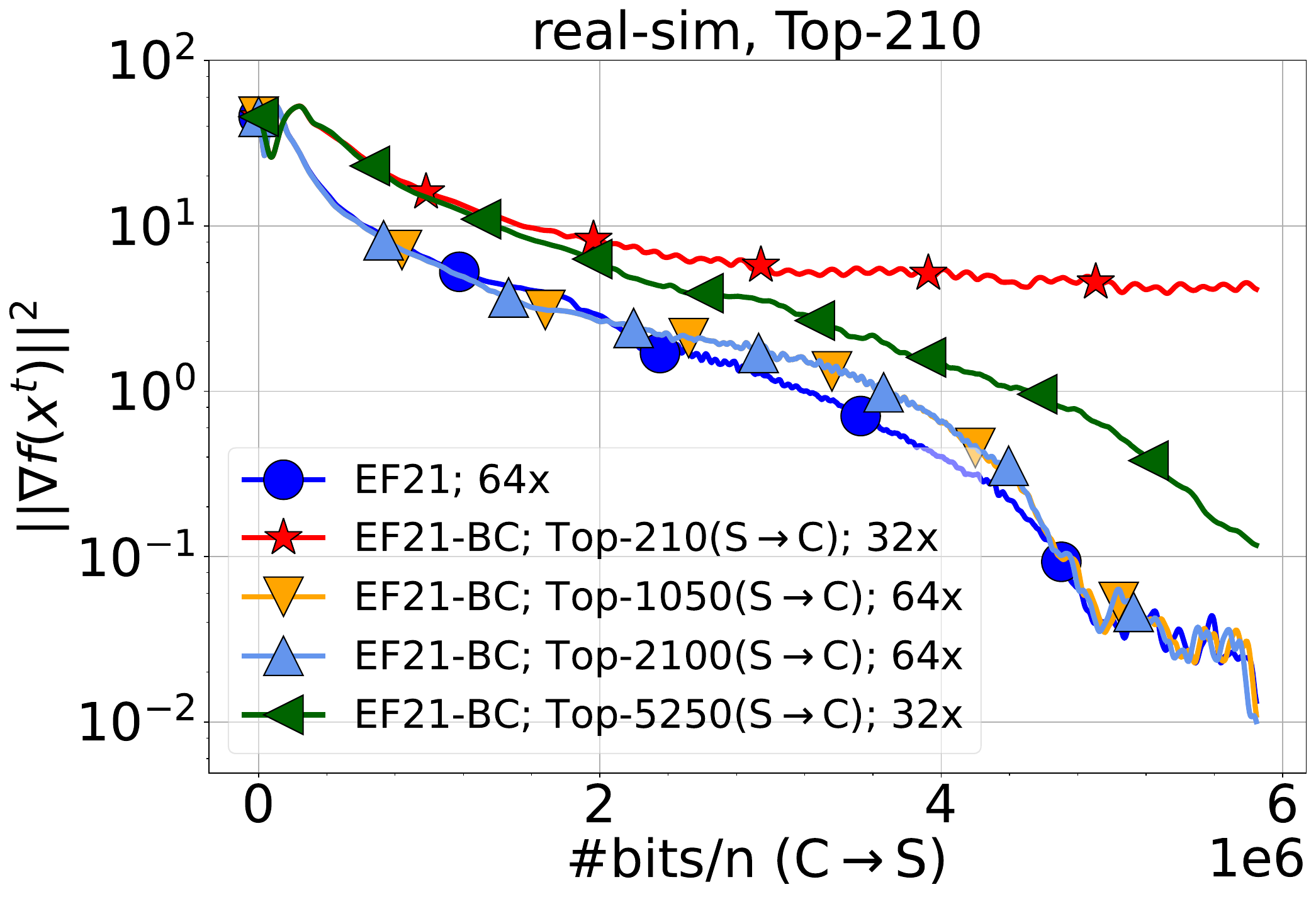} 
		\caption{Convergence in terms of  total number of bits sent from \textbf{C}lients to the \textbf{S}erver divided by $n$.\\${}$\\${}$}
		\label{fig:ef21_bc_od_bits_realsim}
	\end{subfigure}
	\caption{{Comparison of \algname{EF21-BC} and \algname{EF21} with tuned stepsizes .  By $1\times, 2\times, 4\times$ (and so on) we indicate that the stepsize was set to a multiple of  the largest stepsize predicted by theory for \algname{EF21} (see  the Theorem \ref{thm:main-distrib}).}}\label{fig:ef21_bc_0}
\end{figure}

\begin{figure}[H]
	\centering
	\begin{subfigure}{0.9\textwidth}
		\includegraphics[width=\linewidth]{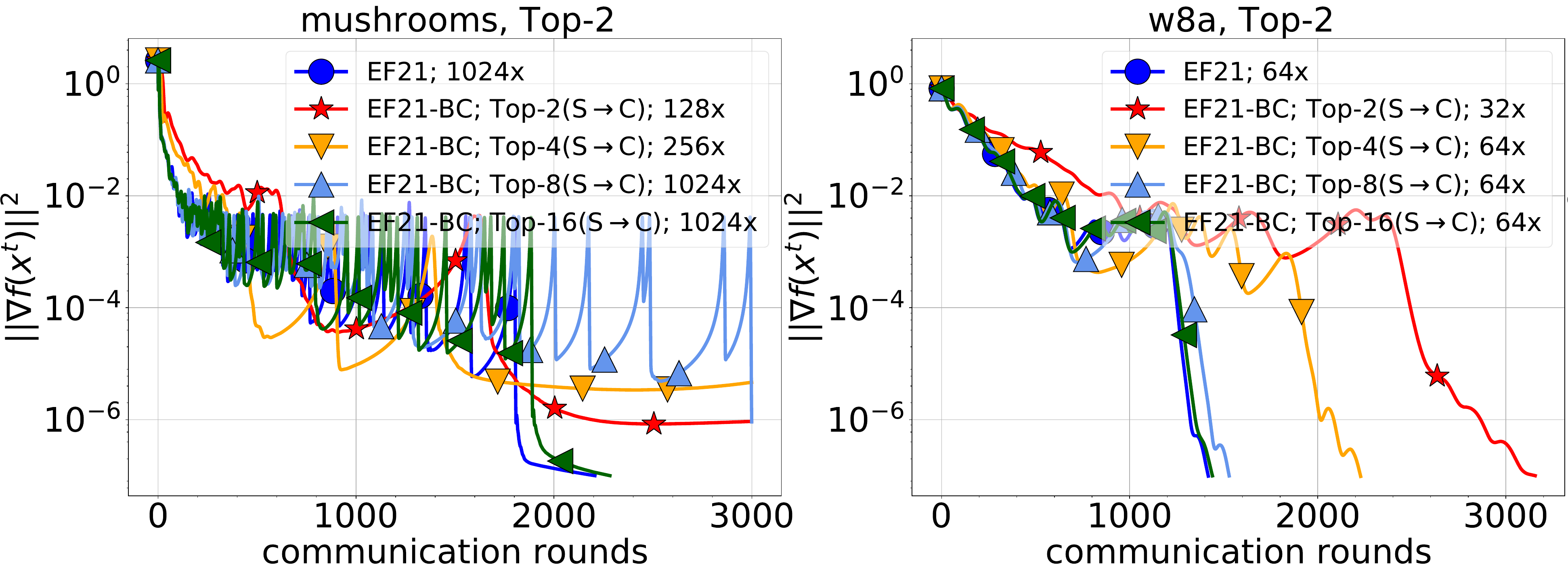}

		\includegraphics[width=\linewidth]{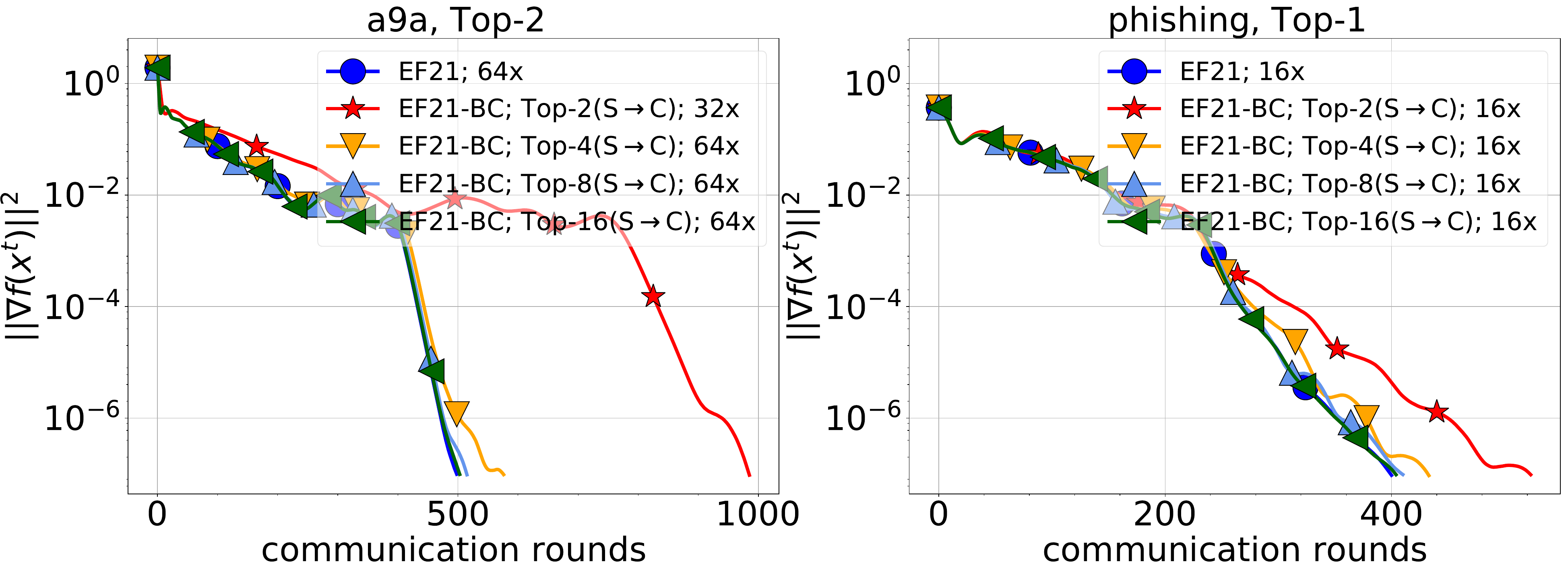}
		\caption{Convergence in communication rounds.}
		\label{fig:ef21_bc_comm}
	\end{subfigure}
	\hfill
	\begin{subfigure}{0.9\textwidth}
		\includegraphics[width=\linewidth]{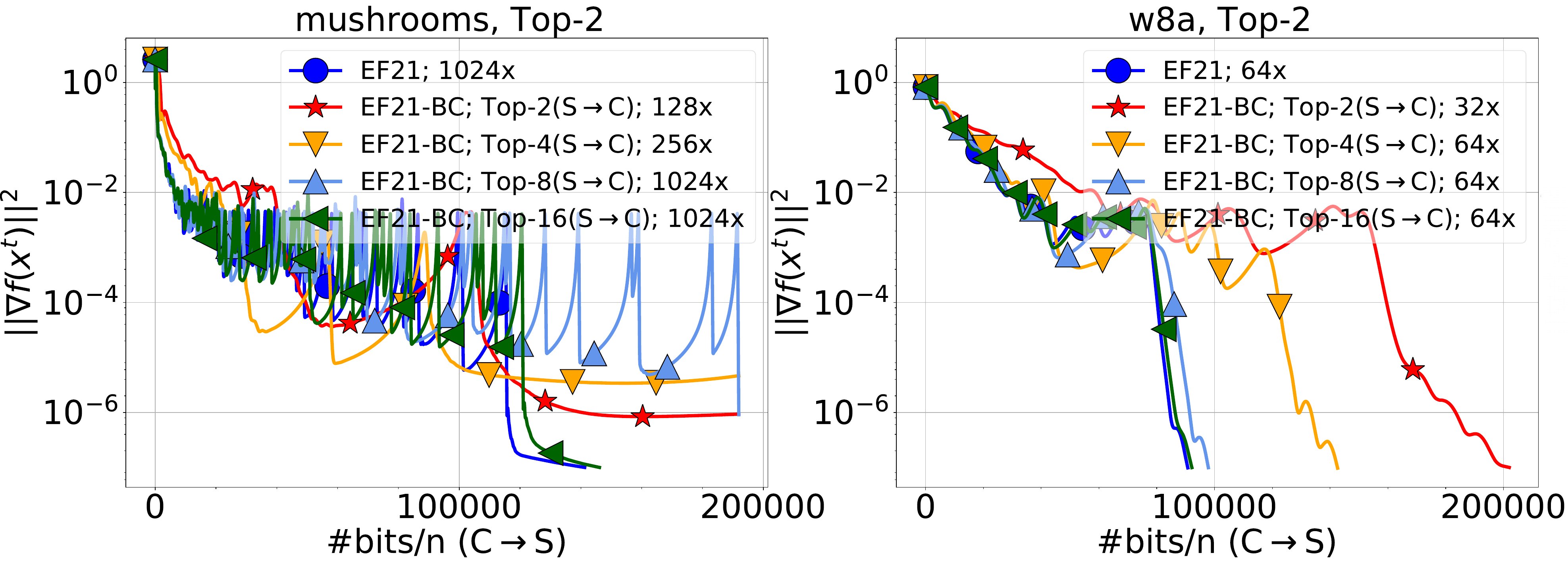}

		\includegraphics[width=\linewidth]{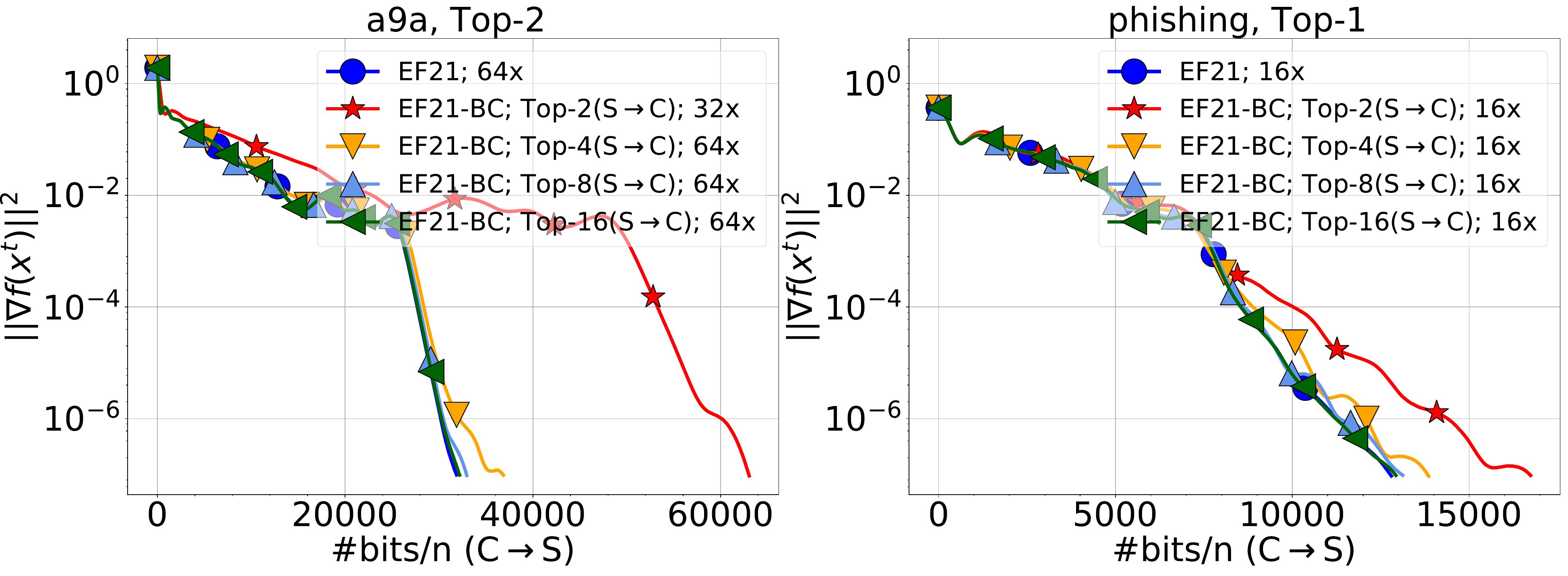}
		\caption{Convergence in terms of  total number of bits sent from \textbf{C}lients to the \textbf{S}erver divided by $n$.}
		\label{fig:ef21_bc_od_bits}
	\end{subfigure}
	\hfill
		\caption{Comparison of \algname{EF21-BC} and \algname{EF21} with tuned stepsizes .  By $1\times, 2\times, 4\times$ (and so on) we indicate that the stepsize was set to a multiple of  the largest stepsize predicted by theory for \algname{EF21} (see  the Theorem \ref{thm:main-distrib}) .}\label{fig:ef21_bc}
\end{figure}

\begin{figure}[H]
	\centering
	\begin{subfigure}{0.9\textwidth}
		\includegraphics[width=\linewidth]{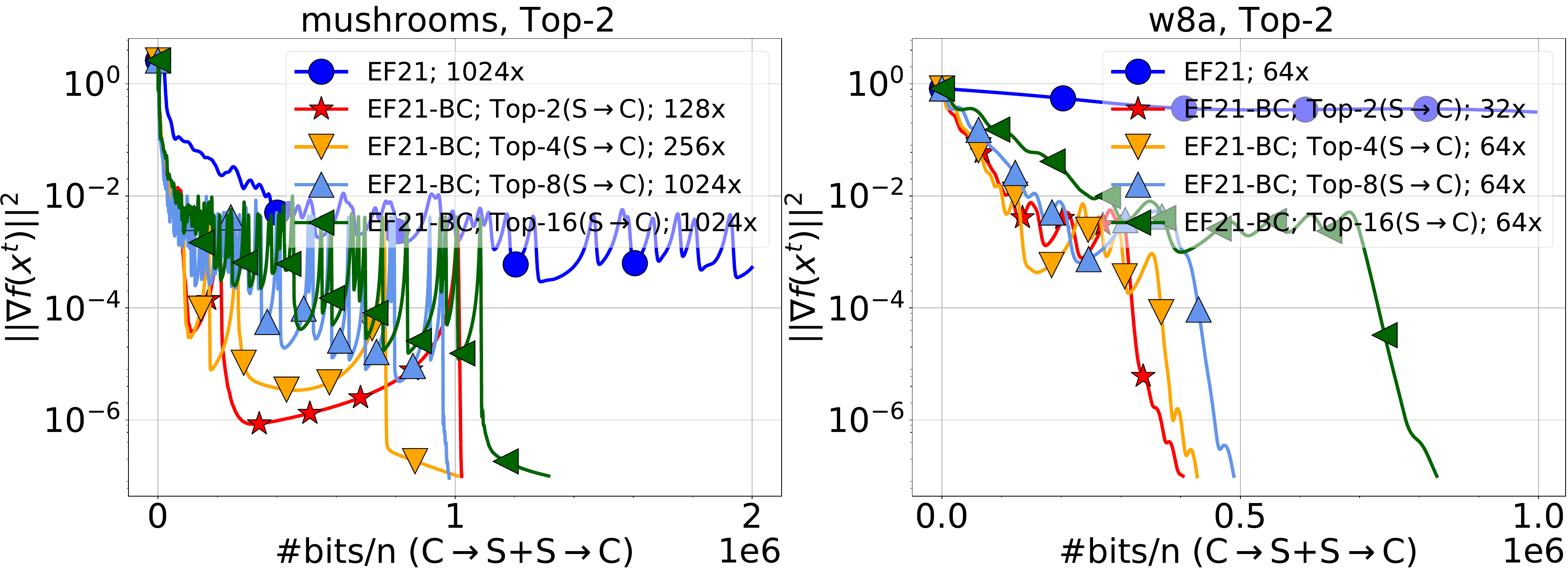}

		\includegraphics[width=\linewidth]{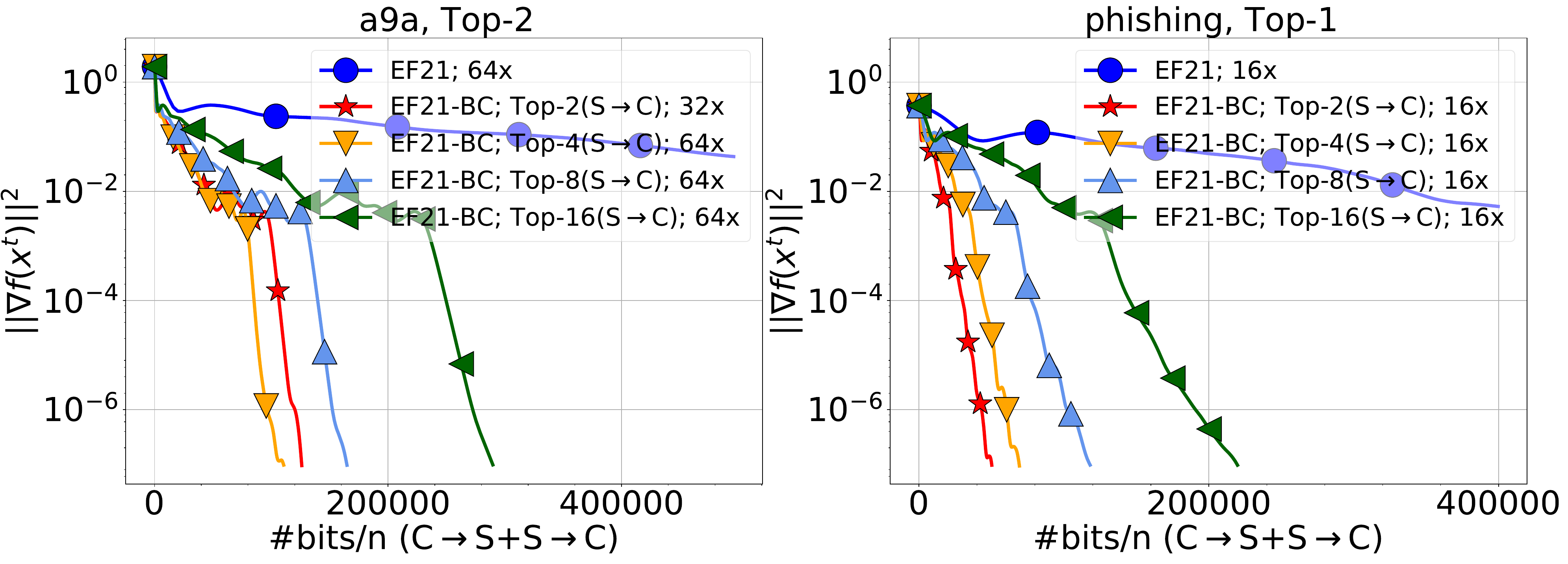}
		\caption{Convergence in terms of  total number of bits sent from \textbf{C}lients to the \textbf{S}erver plus the total number of bits broadcasted from \textbf{S}erver to \textbf{C}lients divided by $n$.}
		\label{fig:ef21_bc_bd_bits}
	\end{subfigure}
\caption{Comparison of \algname{EF21-BC} and \algname{EF21} with tuned stepsizes .  By $1\times, 2\times, 4\times$ (and so on) we indicate that the stepsize was set to a multiple of  the largest stepsize predicted by theory for \algname{EF21} (see  the Theorem \ref{thm:main-distrib}) .}\label{fig:ef21_bc_2}
\end{figure}

For each parameter $k$ in Server-Clients compression, we tune the stepsize multiplier for \algname{EF21-BC} within the following set:
$$\cb{0.125,0.25, 0.5, 1, 2, 4, 8, 16, 32, 64, 128, 256, 512, 1024, 2048}.$$

\paragraph{Experiment 4: On the cheaper computations via \algname{EF21-SGD}. }\label{subsecexp:ef21_sgd-ef21_fg} 
The fourth experiment (see Figure~\ref{fig:ef21_sgd_epochs}) illustrates that \algname{EF21-SGD}(Alg. \ref{alg:EF21-online}) is the more preferable choice than \algname{EF21}for the cases when full gradient computations are costly.

For each batchsize from the set\footnote{By $50\%, 25\% $ (and so on) we refer to a batchsize, which is equals to $\lfloor0.5 N_i\rfloor$, $\lfloor0.25N_i\rfloor$ (and so on) for all clients $i=1,\dots,n$.} $$\cb{95 \%, 50\%,	25\%,12.5\%,	6.5\%,3\%},$$ we tune the stepsize multiplier for \algname{EF21-SGD} within the following set:
$$\cb{0.25, 0.5, 1, 2, 4, 8, 16, 32, 64, 128, 256, 512, 1024, 2048}.$$

Figure \ref{fig:ef21_sgd_epochs} illustrates that \algname{EF21-SGD} is able to reach a moderate tolerance in $5-10$ epochs.
\begin{figure}[H]
	\centering
	\begin{subfigure}{0.9\textwidth}
		\includegraphics[width=\linewidth]{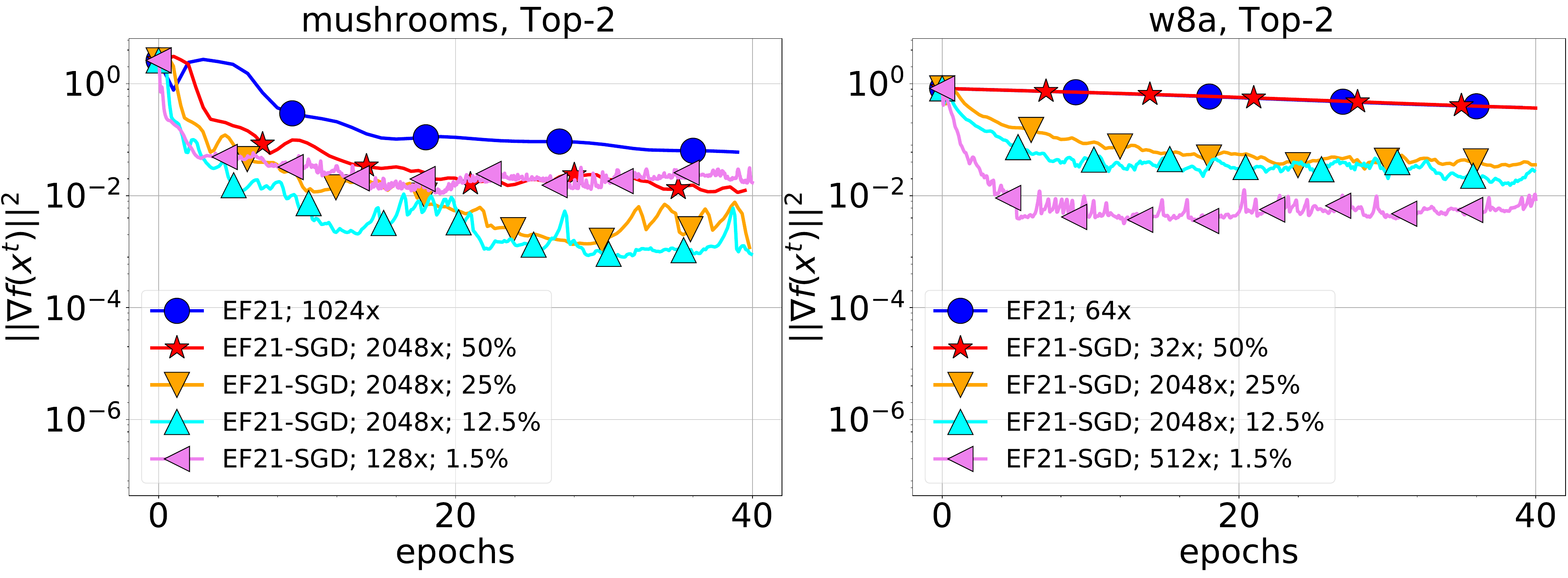}

		\includegraphics[width=\linewidth]{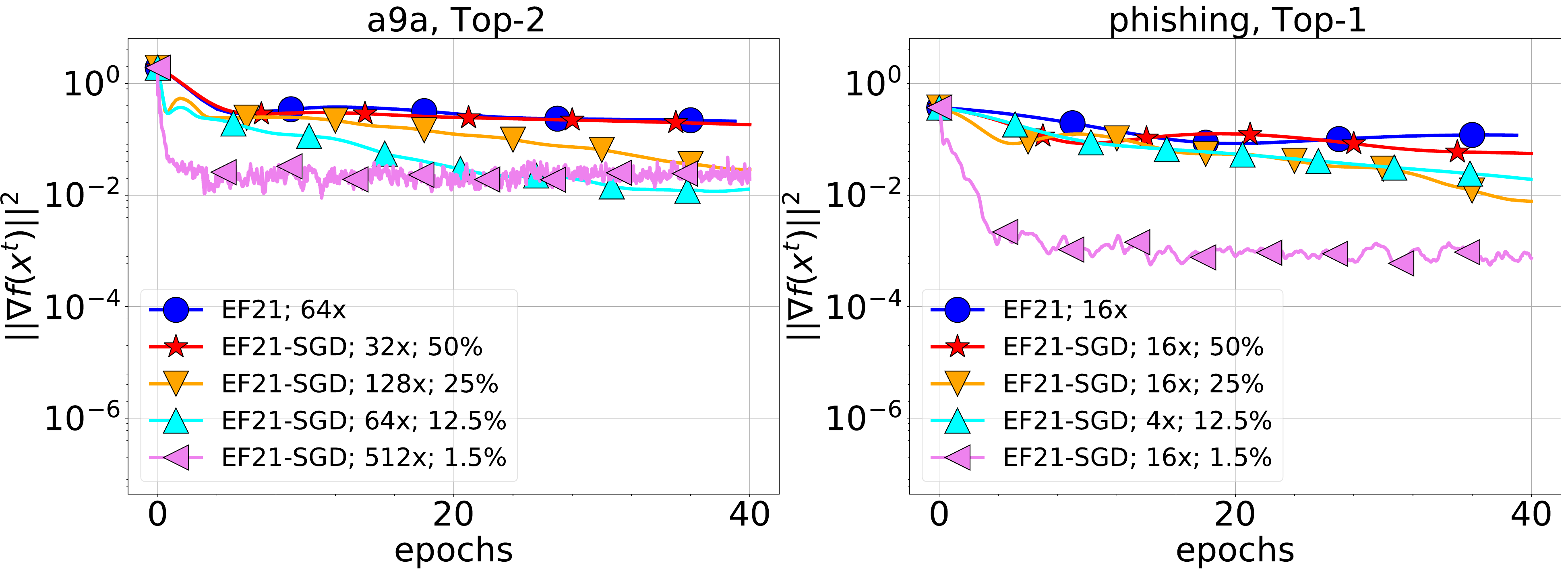}
		\caption{Convergence in epochs.}
		\label{fig:ef21_sgd_epochs}
	\end{subfigure}
	\hfill
	\begin{subfigure}{0.9\textwidth}
		\includegraphics[width=\linewidth]{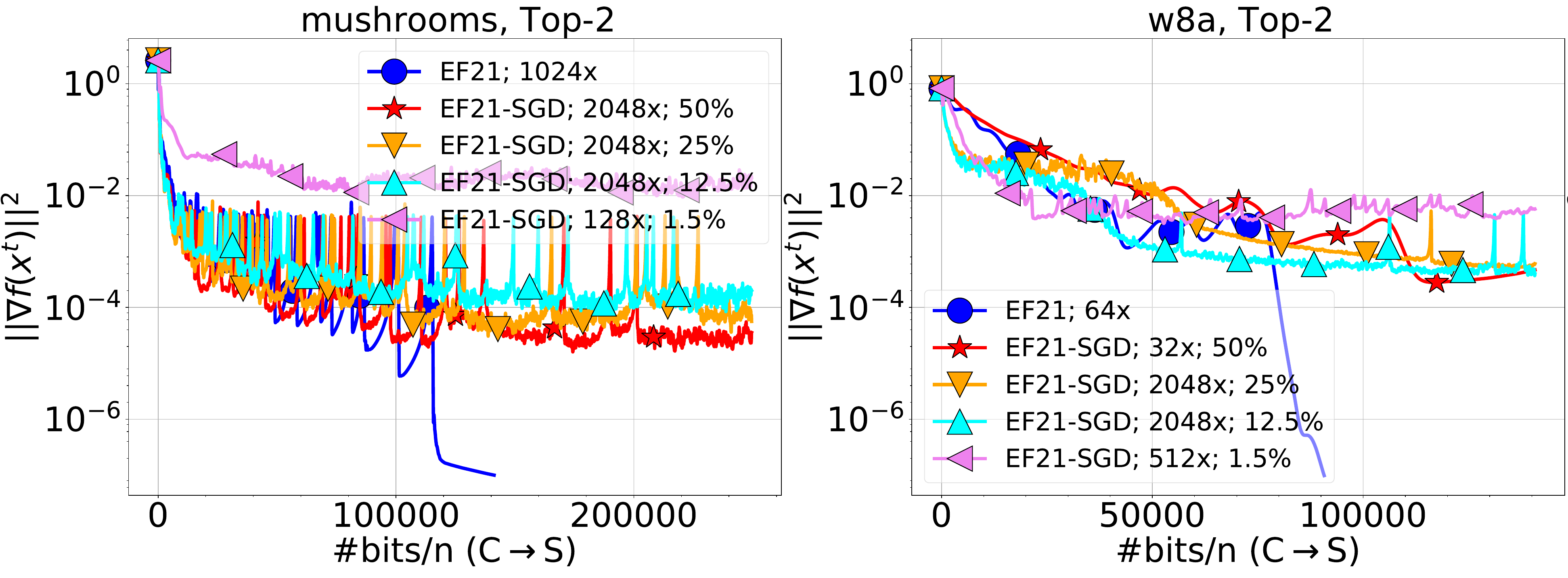}

		\includegraphics[width=\linewidth]{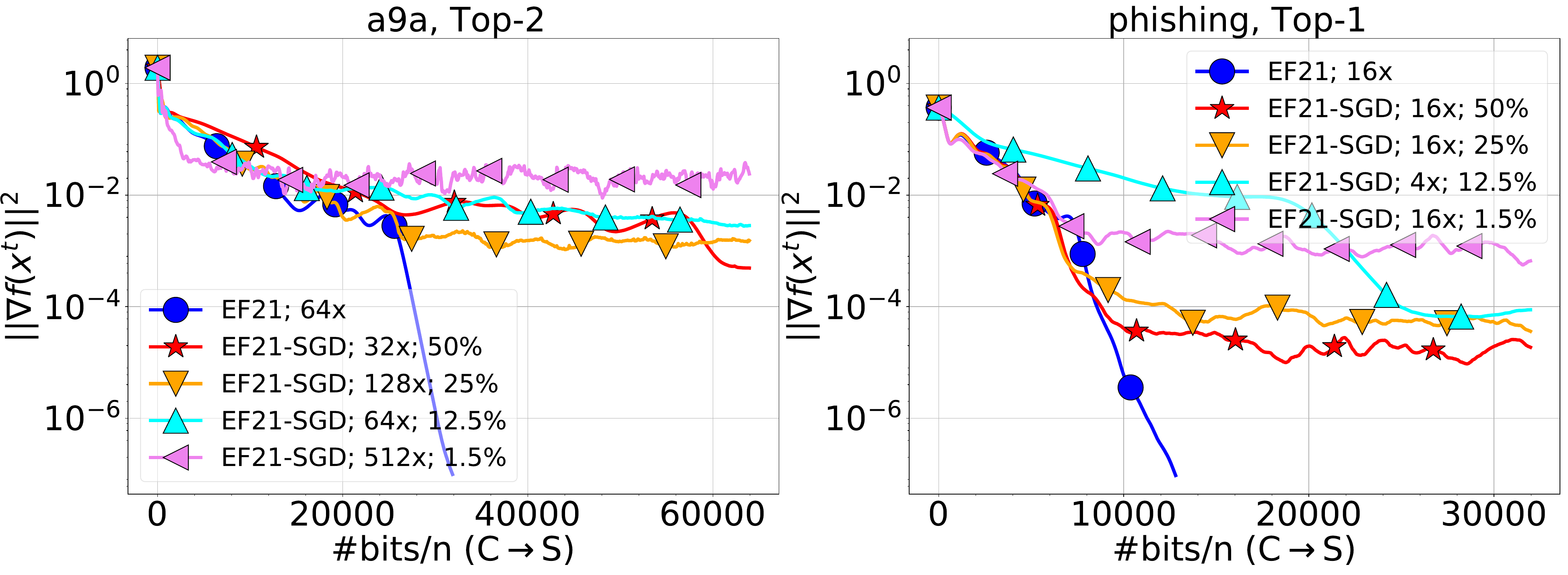} 
		\caption{Convergence in terms of the number of bits sent from \textbf{C}lients to the \textbf{S}erver by each client.}
		\label{fig:ef21_sgd_bits}
	\end{subfigure}
	\caption{
		Comparison of \algname{EF21-SGD} and \algname{EF21} with tuned stepsizes. By $1\times, 2\times, 4\times$ (and so on) we indicate that the stepsize was set to a multiple of  the largest stepsize predicted by theory for \algname{EF21}. By $50\%, 25\% $ (and so on) we refer to a batchsize, which is equals to $\lfloor0.5 N_i\rfloor$, $\lfloor0.25N_i\rfloor$ (and so on) for all clients $i=1,\dots,n$.}\label{fig:ef21_sgd}%
\end{figure}
However, due to the accumulated variance introduced by SGD, estimator \algname{EF21-SGD} is stuck at some accuracy level (see Figure \ref{fig:ef21_sgd_bits}), showing the usual behavior of the SGD observed in practice.

\paragraph{Experiment 5: On the effect of heavy ball momentum.}\label{subsecexp:ef21_hb-ef21_fg_ext}
In this experiment (see Figure~\ref{fig:ef21_hb}), we show that for the majority of the considered data sets heavy ball acceleration used in \algname{EF21-HB} (Alg. \ref{alg:EF21_HB}) improves the convergence of \algname{EF21} method. For every data set (and correspondingly chosen parameter $k$) we tune momentum parameter  $\eta$ in \algname{EF21-HB} by making a grid search over all possible parameter values  from $0.05$ to $0.99$ with the step $0.05$. Finally, for our plots we pick $\eta \in \cb{0.05, 0.2, 0.25, 0.4, 0.9}$ since the first four values shows the best performance and $\eta =0.9$ is a popular choice in practice.

For each parameter $\eta$ from the set 
$$\cb{0.05,0.1,0.15,0.2,0.25,0.3,0.35,0.4,0.45,0.5,0.55,0.6,0.65,0.7,0.75,0.8,0.85,0.9,0.95,0.99}.$$ 
we perform a grid search of stepsize multiplier within the powers of $2$:
$$\cb{0.125, 0.25, 0.5, 1, 2, 4, 8, 16, 32, 64, 128, 256, 512, 1024, 2048}.$$

\begin{figure}[H]
	\includegraphics[width=0.9\linewidth]{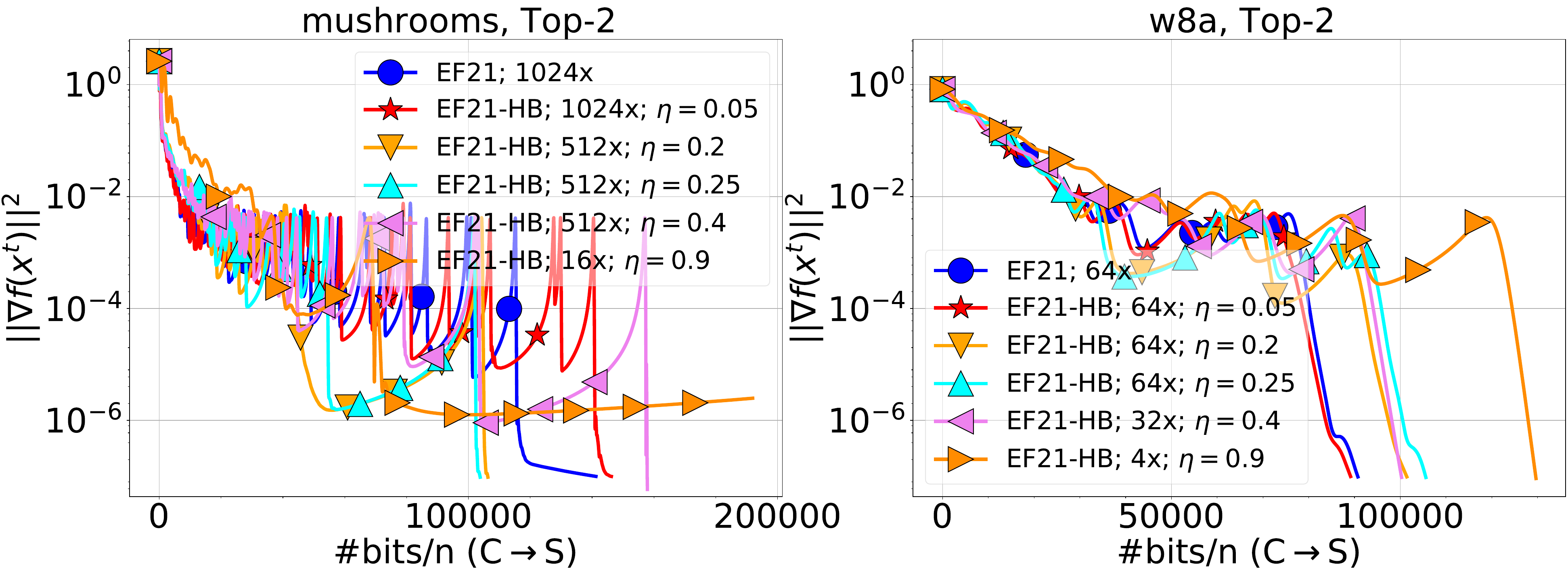}

	\includegraphics[width=0.95\linewidth]{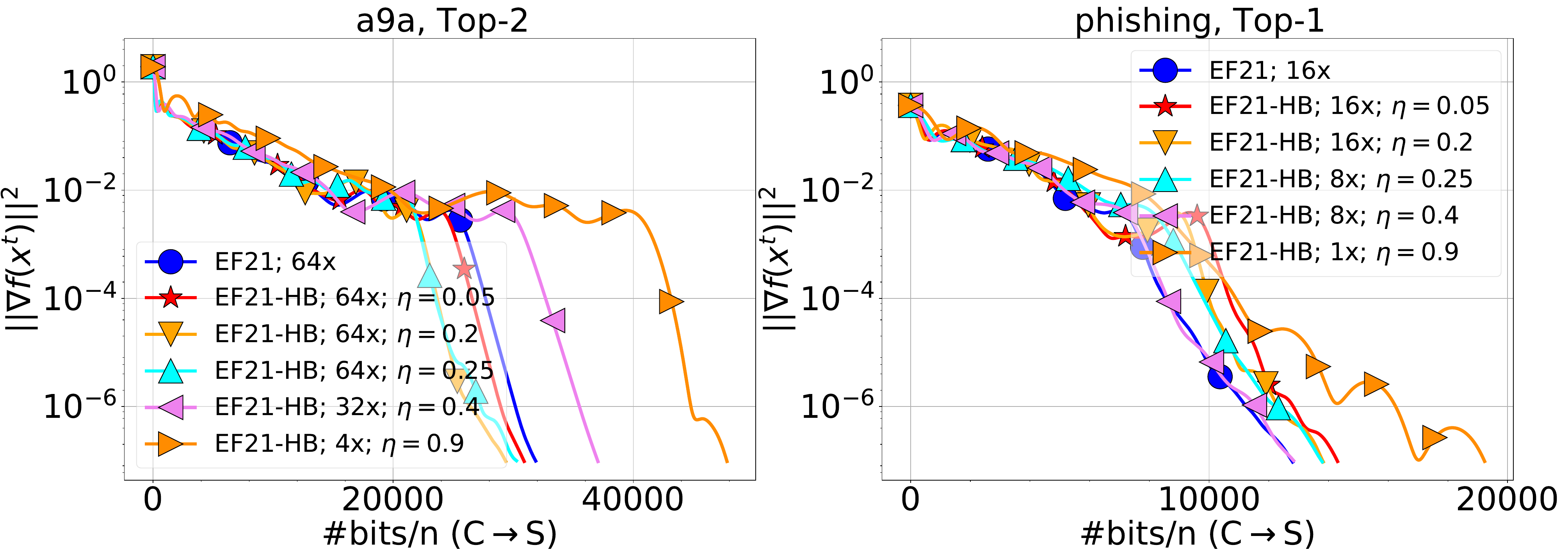} 
	\caption{Comparison of \algname{EF21-HB} and \algname{EF21} with  tuned parameters in terms of  total number of bits sent from \textbf{C}lients to the \textbf{S}erver divided by $n$.  By $1\times, 2\times, 4\times$ (and so on) we indicate that the stepsize was set to a multiple of  the largest stepsize predicted by theory for \algname{EF21} (see  the Theorem \ref{thm:main-distrib}) .}\label{fig:ef21_hb}
\end{figure}

\paragraph{{Comparison to non-compressed methods.}}  {In addition, we compare  \algname{EF21-PAGE} and \algname{EF21-SGD} to the baseline methods without compression: \algname {PAGE} (Figure \ref{fig:page_bits}) and \algname {SGD} (Figure \ref{fig:sgd_bits}). In these experiments, we observe that \algname{EF21-PAGE} and \algname{EF21-SGD} require much less information to transmit in order to achieve the same accuracy of the solution as the methods without compression (\algname{PAGE}, \algname{SGD}).}
\begin{figure}[H]
	\begin{subfigure}{0.9\textwidth}
		\includegraphics[width=\linewidth]{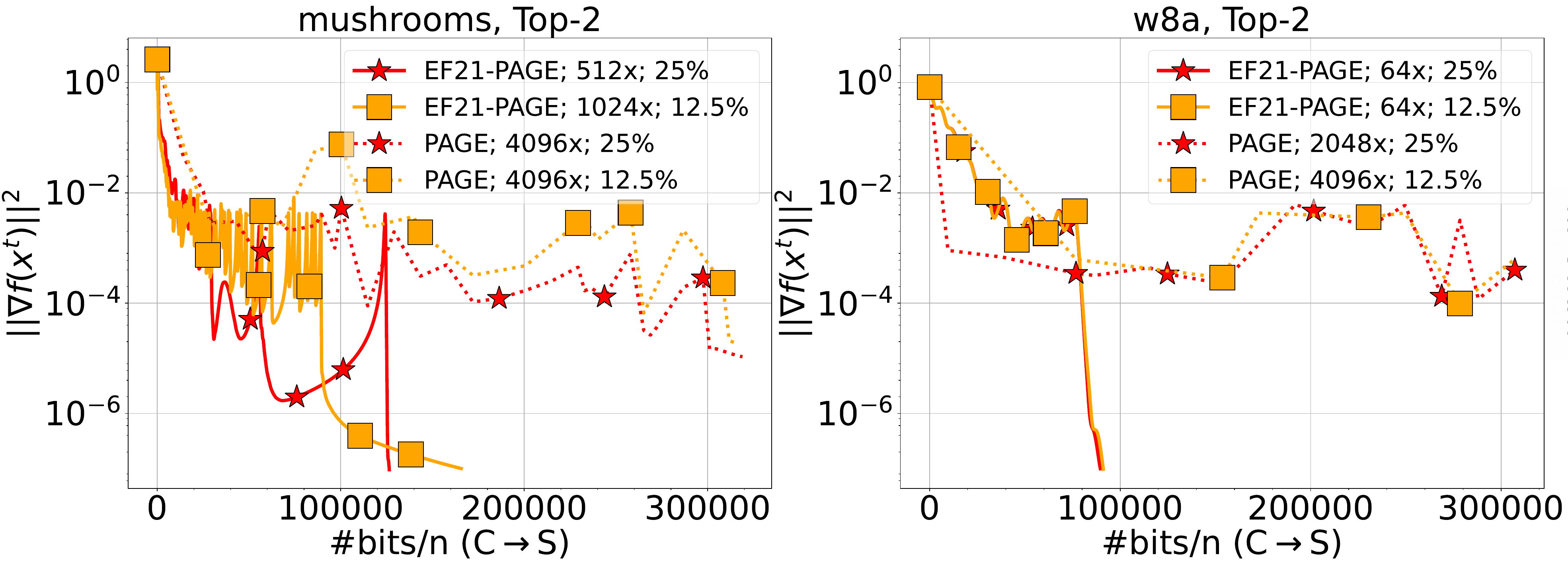}

		\includegraphics[width=\linewidth]{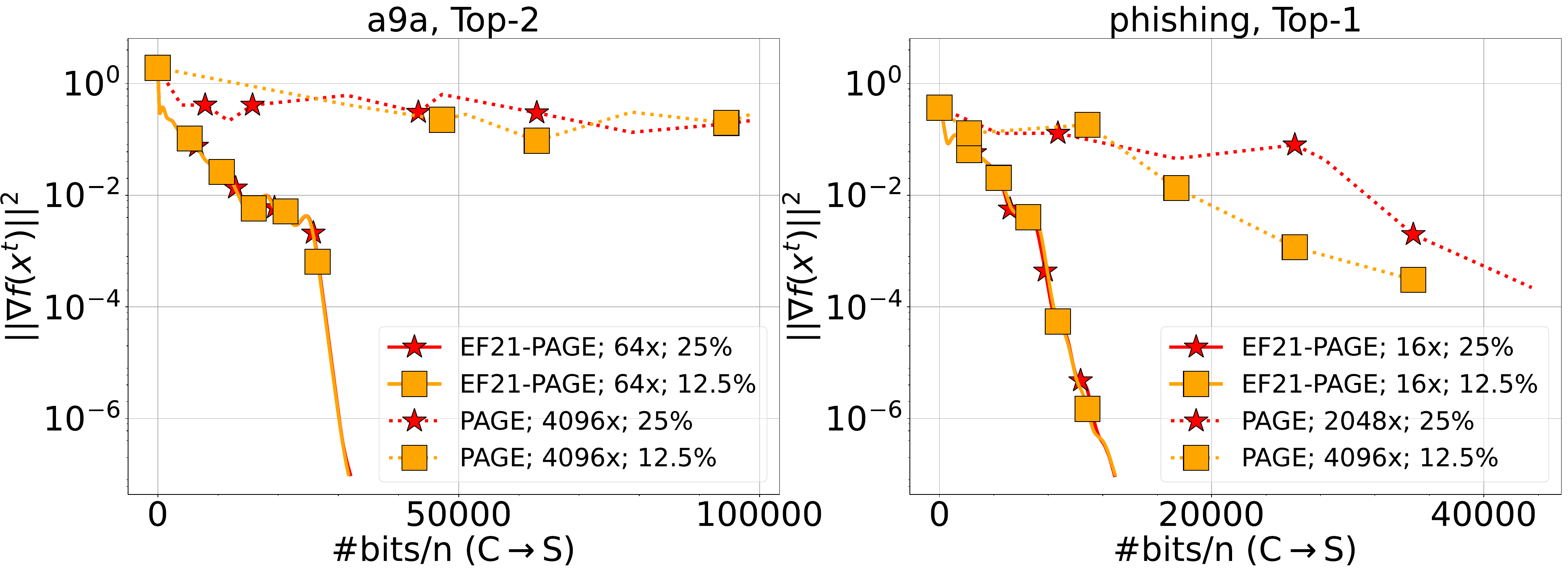} 
		\caption{Convergence in terms of  total number of bits sent from \textbf{C}lients to the \textbf{S}erver divided by $n$.}
		\label{fig:page_bits}
	\end{subfigure}
	\caption{
		{Comparison of \algname{EF21-PAGE} and \algname{PAGE} with tuned parameters. By $1\times, 2\times, 4\times$ (and so on) we indicate that the stepsize was set to a multiple of  the largest stepsize predicted by theory for \algname{EF21}. By $25\%$, $12.5\% $ and $1.5\% $ we refer to batchsizes equal $\lfloor0.25 N_i\rfloor$, $\lfloor0.125N_i\rfloor$ and $\lfloor0.015N_i\rfloor$ for all clients $i=1,\dots,n$, where $N_i$ denotes the size of local data set.}}\label{fig:ef21_page}%
\end{figure}

\begin{figure}[H]
	\begin{subfigure}{0.9\textwidth}
		\includegraphics[width=\linewidth]{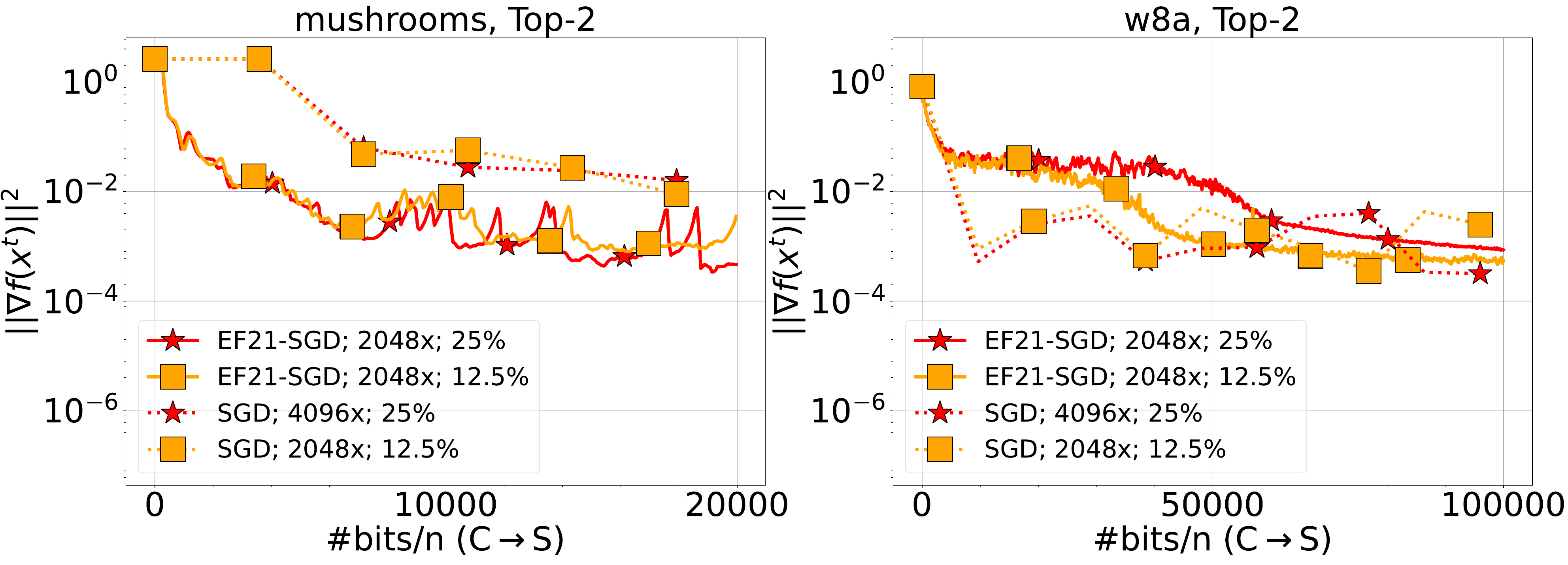} 
		
		\includegraphics[width=\linewidth]{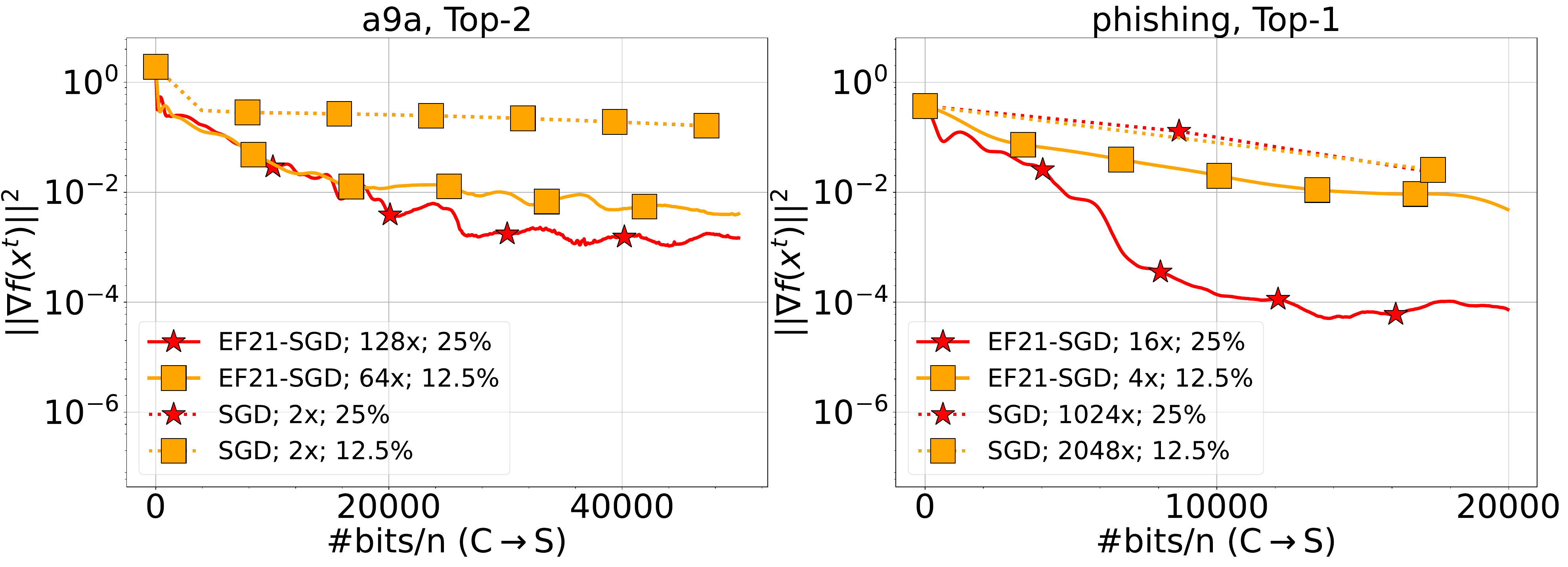}
		\caption{Convergence in terms of  total number of bits sent from \textbf{C}lients to the \textbf{S}erver divided by $n$.}
		\label{fig:sgd_bits}
	\end{subfigure}
	\caption{
		{Comparison of \algname{EF21-SGD} and \algname{SGD} with tuned parameters. By $1\times, 2\times, 4\times$ (and so on) we indicate that the stepsize was set to a multiple of  the largest stepsize predicted by theory for \algname{EF21}. By $25\%$, $12.5\% $ and $1.5\% $ we refer to batchsizes equal $\lfloor0.25 N_i\rfloor$, $\lfloor0.125N_i\rfloor$ and $\lfloor0.015N_i\rfloor$ for all clients $i=1,\dots,n$, where $N_i$ denotes the size of local data set.}}\label{fig:sgd_page}%
\end{figure}

\subsection{Deep Learning Experiments}\label{subsec:DL_exp}

\begin{figure}[H]
	\includegraphics[width=1\linewidth]{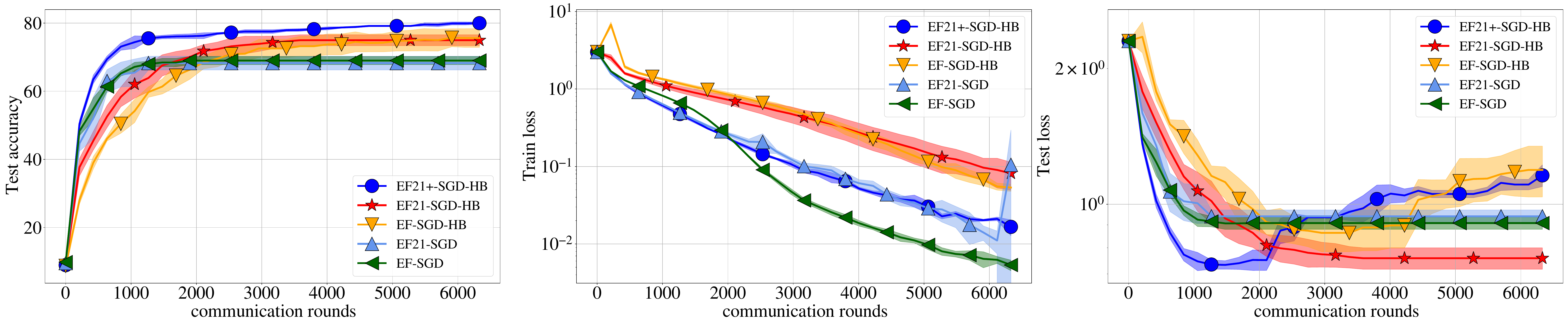} 
	\caption{Comparison of \algname{EF-SGD} and \algname{EF21-SGD} with  \algname{EF-SGD-HB}, \algname{EF21-SGD-HB}, and \algname{EF21+-SGD-HB} with tuned stepsizes applied to train ResNet18 on CIFAR10.}\label{fig:DL_HB}
\end{figure}

The main goal of this section is to compare the behavior of \algname{EF} based methods on a larger optimization problem: training a standard image classification model. 
In this set of experiments, the exact/full gradient $\nabla f_i(x^{k+1})$ in the algorithm \algname{EF21-HB} is replaced by its stochastic estimator (we later refer to this method as \algname{EF21-SGD-HB}). We compare the resulting method with some existing baselines on a popular deep learning multi-class image classification task. In particular, we compare our \algname{EF21-SGD-HB} method to 
\algname{EF21+-SGD-HB}\footnote {\algname{EF21+-SGD-HB} is the method obtained from \algname{EF21-SGD-HB} via replacing \algname{EF21} by \algname{EF21+} compressor} , 
\algname{EF-SGD-HB}\footnote {\algname{EF-SGD-HB} is the method obtained from \algname{EF21-SGD-HB} via replacing \algname{EF21} by \algname{EF} compressor}, 
\algname{EF21-SGD} and \algname{EF-SGD} on the problem of training ResNet18 \citep{he2016deep} model on CIFAR-10 \citep{krizhevsky2009learning} data set. For more details about the \algname{EF21+} and \algname{EF} type methods and their applications in deep learning we refer reader to \citep{EF21}. 
We implement the algorithms in PyTorch \citep{paszke2019pytorch} and run the experiments on a single GPU NVIDIA GeForce RTX 2080 Ti.
The data set is split into $n = 8$ equal parts. Total train set size for CIFAR-10 is  $50,000$. The test set for evaluation has $10,000$ data points. The train set is split into batches of size $\tau = 32$. The first seven workers own an equal number of batches of data, while the last worker gets the rest. 
In our experiments, we fix $k \approx 0.05d $, $\tau = 32$ and momentum parameter $\eta = 0.9$.\footnote{Here, $d$ is the number of model parameters. For ResNet18, $d = 11,511,784$.} We tune the stepsize $\g$ within the range $\cb{0.0625, 0.125, 0.25, 0.5, 1}$ and for each method we individually chose the one $\g$ giving the highest accuracy score on test. For momentum methods, the best stepsize was $0.5$, whereas for the non-momentum ones it was $0.125$. Note that in this experiment, we fix constant step-size $\gamma$ for all methods during training in order to focus on the effect that different \algname{EF} methods bring. However, in order to acheive even better performance in practice, these methods should be combined with appropriate step-size scheduling or adaptive step-size scheme. 

The experiments show (see Figure~\ref{fig:DL_HB}) that the train loss for momentum methods decreases slower than for the non-momentum ones, whereas for the test loss situation is the opposite. Finally, momentum methods show a considerable improvement in the accuracy score on the test set over the existing \algname{EF21-SGD} and \algname{EF-SGD}. Note that the achieved accuracies are below the highest current standards since we deactivate all augmentations and regularizations during training.

\subsection{Verifying tightness of rates for \algname{EF21-SGD} and \algname{EF21-PAGE} w.r.t. $n$}\label{subsec:tightness_n}
Consider the following toy problem $\min_x f(x) := \frac{1}{n} \sum_{i=1}^{n} f_i(x)$, where all local functions are the same and are defined by $f_i(x) = \frac{1}{2} \sqnorm{x} + \sum_{j=1}^{3} \langle z_j, x \rangle $, $x \in \R^2$, where
$
z_1 = \begin{pmatrix} 2 \\  0 \end{pmatrix} \sqrt{\frac{3 \sigma^2}{10}}, \quad z_2 = \begin{pmatrix} 0 \\ 1 \end{pmatrix} \sqrt{\frac{3 \sigma^2}{10}} , \quad  z_3 = \begin{pmatrix} -2 \\  -1 \end{pmatrix} \sqrt{\frac{3 \sigma^2}{10}} 
$
for some $\sigma > 0$. The stochastic gradients at each node $i = 1, \ldots, n$ are $\nabla f_i(x^{t}, \xi^{t}) = x + z_j$, where $z_j$ are sampled uniformly at random (and independently for each node) from the three datapoints $z_1$, $z_2$ and $z_3$. Notice that $\Exp{\nabla f_i(x^{t}, \xi_i^{t})} = \nabla f_i(x^{t})$, and  $\Exp{\sqnorm{\nabla f_i(x^{t}, \xi_i^{t}) - \nabla f_i(x^t)} } = \sigma^2 $. We select $\sigma = 1$ in our experiments and use Top-$1$ compressor. We run the algorithms with batch-size $\tau = 1$ and the same small constant step-sizes $\gamma = \frac{0.1}{\sqrt{T}}$ or $\gamma = \frac{0.9}{\sqrt{T}}$, where $T = 10000$. Here we select the same step-size across all algorithms for a fair comparison and to demonstrate the absence of improvement over $n$. \footnote{In fact, in \algname{EF21-PAGE} the step-size should be selected much larger, i.e. of order $\cO(1)$, to achieve faster convergence. See Section~\ref{sec:exp} for experiments with tuned step-sizes.} The presented plots show the median performance alongside the $25\%$ and $75\%$ quantiles over $10$ independent runs. 

We observe that \algname{EF21-SGD} and \algname{EF21-PAGE} \textit{do not have improvement} when $n$ is increased, while in the same setup (as it is expected) \algname{SGD} (\algname{EF21-SGD} without compression) does improve with $n$, see Figure~\ref{fig:sgd_toy}. Interestingly, we notice that for \algname{EF21-SGD} with $n \leq 5000$, increasing $n$ even hurts the convergence, however, when using larger $n$ the convergence rate almost does not change. 

These observations imply that our theoretical sample complexities for \algname{EF21-SGD} and \algname{EF21-PAGE} summarized in Corollary~\ref{cor:monster_corollary} \textbf{are tight} in terms of the dependence on $n$. 

\begin{figure}[H] 
	\begin{minipage}[t]{.95\linewidth}
		\centering
		\begin{subfigure}{.49\textwidth}
			\centering
			\includegraphics[width=1.0\textwidth]{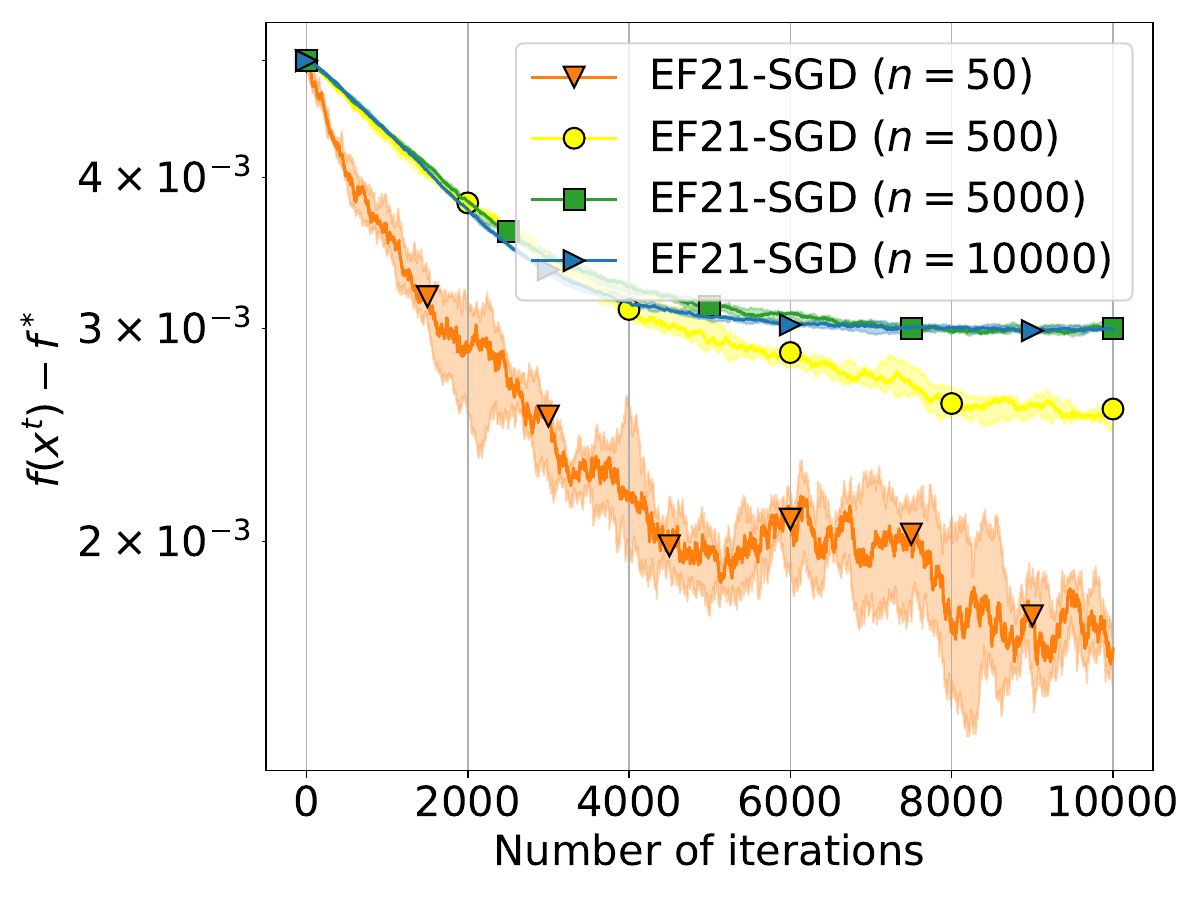}
			\caption{Step-size $\gamma = \frac{0.1}{\sqrt{T}}$.}
		\end{subfigure}
		\begin{subfigure}{.49\textwidth}
			\centering
			\includegraphics[width=1.0\textwidth]{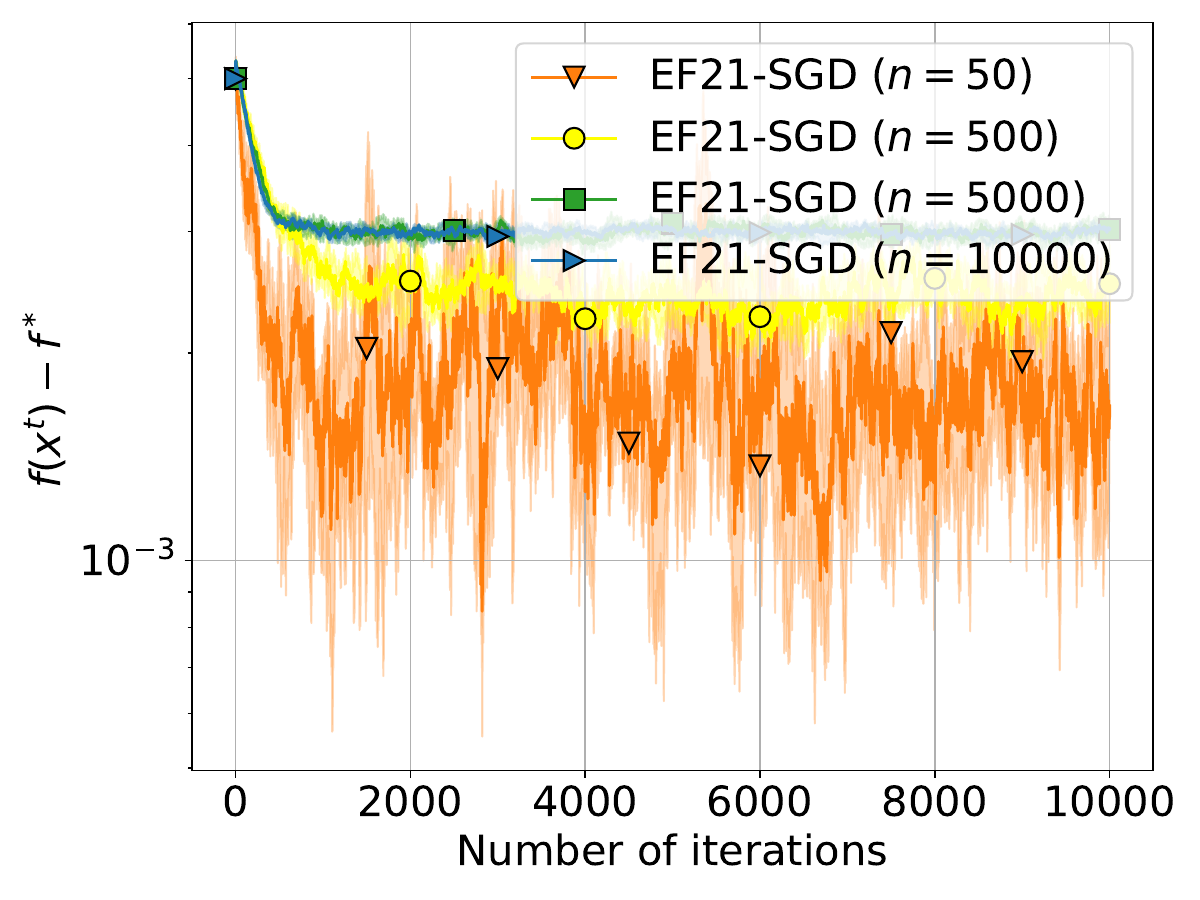}
			\caption{Step-size $\gamma = \frac{0.9}{\sqrt{T}}$.}
			
		\end{subfigure}
		\caption{{No improvement} with $n$ for \algname{EF21-SGD} in terms of the number of iterations. Note that by increasing $n$, the number of data samples used per iteration increases, and therefore, the method expected to have faster convergence. The absence of such improvement is in line with our theory for \algname{EF21-SGD} in Corollary~\ref{cor:monster_corollary}.}
		\label{fig:ef21_sgd_toy}
	\end{minipage}\hfill
	
\end{figure}

\begin{figure}[H]
	\begin{minipage}[t]{.95\linewidth}
		\centering
		\begin{subfigure}{.49\textwidth}
			\centering
			\includegraphics[width=1.0\textwidth]{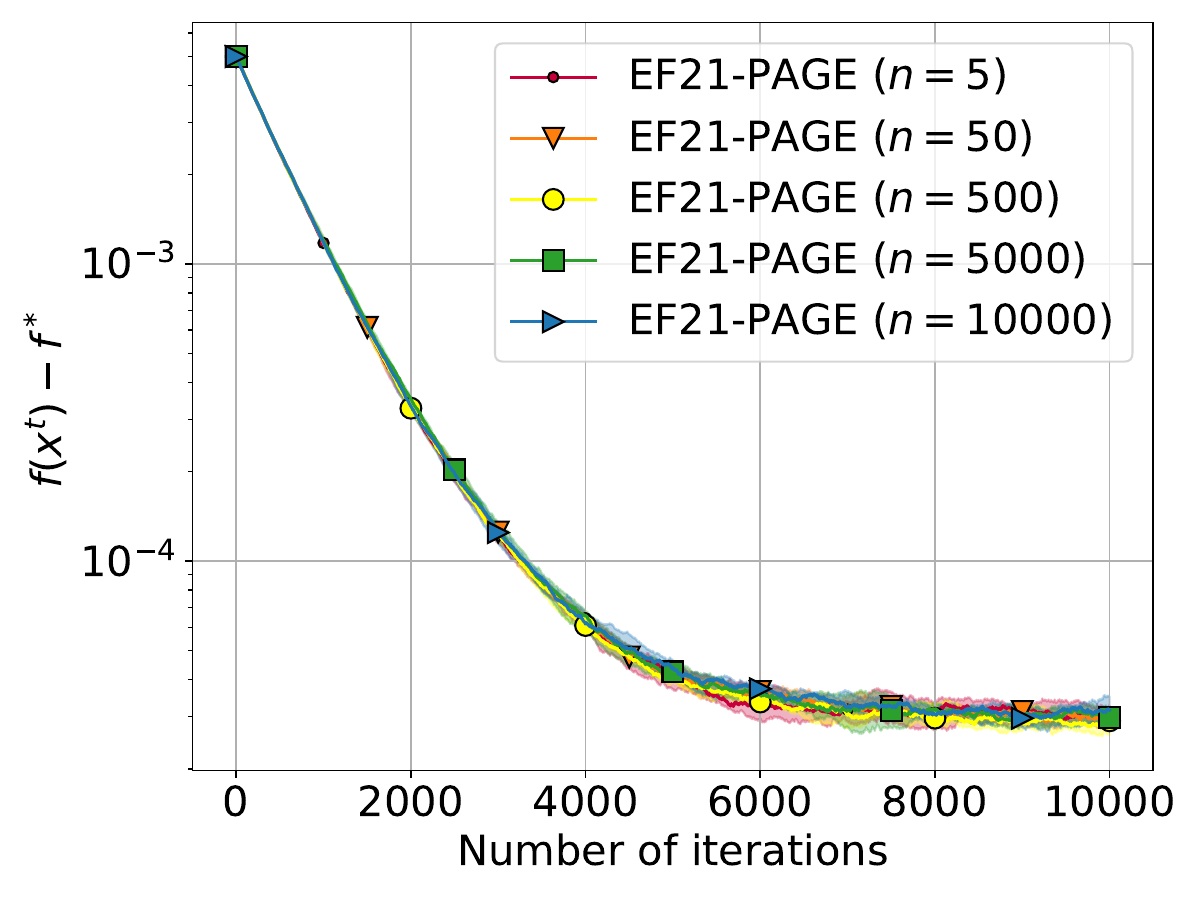}
			\caption{Step-size $\gamma = \frac{0.1}{\sqrt{T}}$.}
		\end{subfigure}
		\begin{subfigure}{.49\textwidth}
			\centering
			\includegraphics[width=1.0\textwidth]{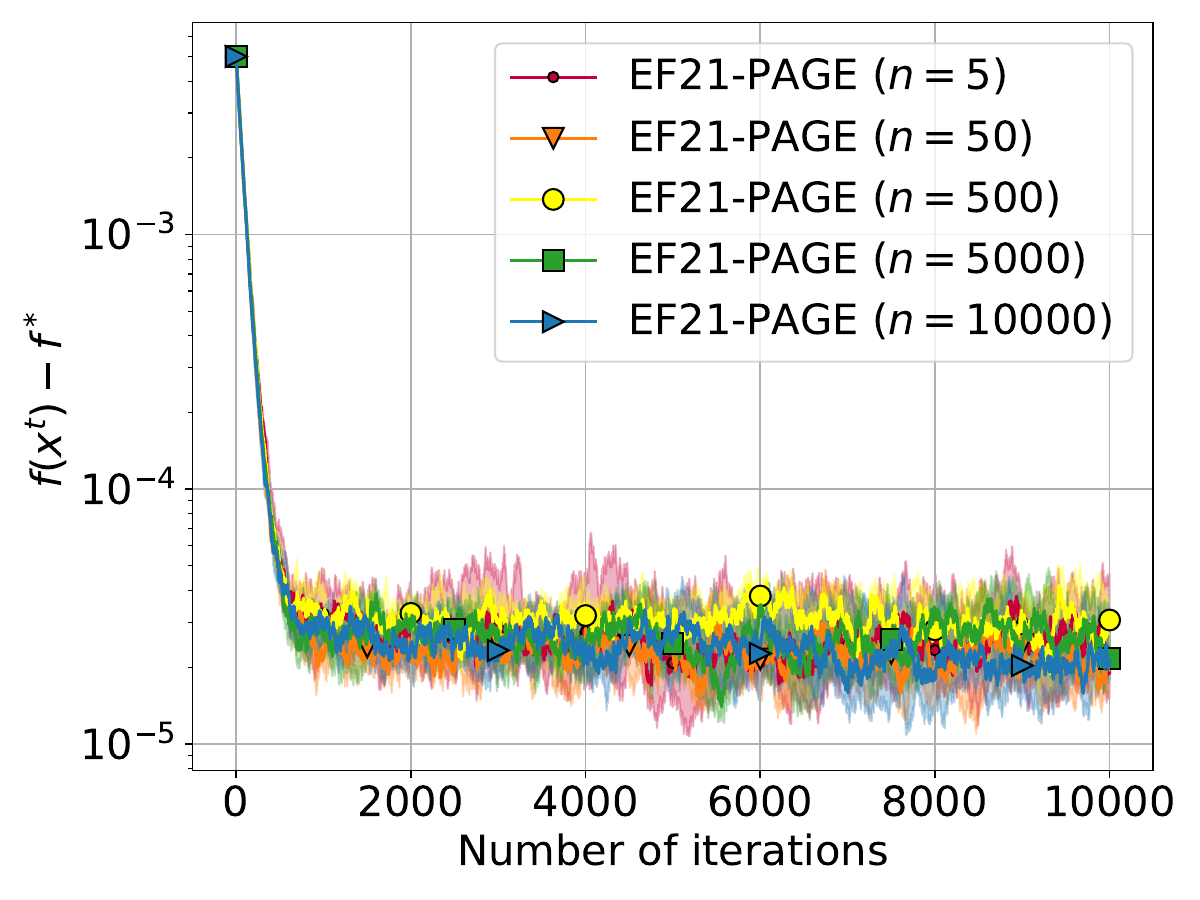}
			\caption{Step-size $\gamma = \frac{0.9}{\sqrt{T}}$.}
			
		\end{subfigure}
		\caption{{No improvement} with $n$ for \algname{EF21-PAGE} in terms of the number of iterations. Note that by increasing $n$, the number of data samples used per iteration increases, and therefore, the method is expected to have faster convergence. The absence of such improvement is in line with our theory for \algname{EF21-PAGE} in Corollary~\ref{cor:monster_corollary}. }
		\label{fig:ef21_page_toy}
	\end{minipage}\hfill
\end{figure}

\begin{figure}[H]
	\begin{minipage}[t]{0.95\linewidth}
		\centering
		\begin{subfigure}{.49\textwidth}
			\centering
			\includegraphics[width=1.0\textwidth]{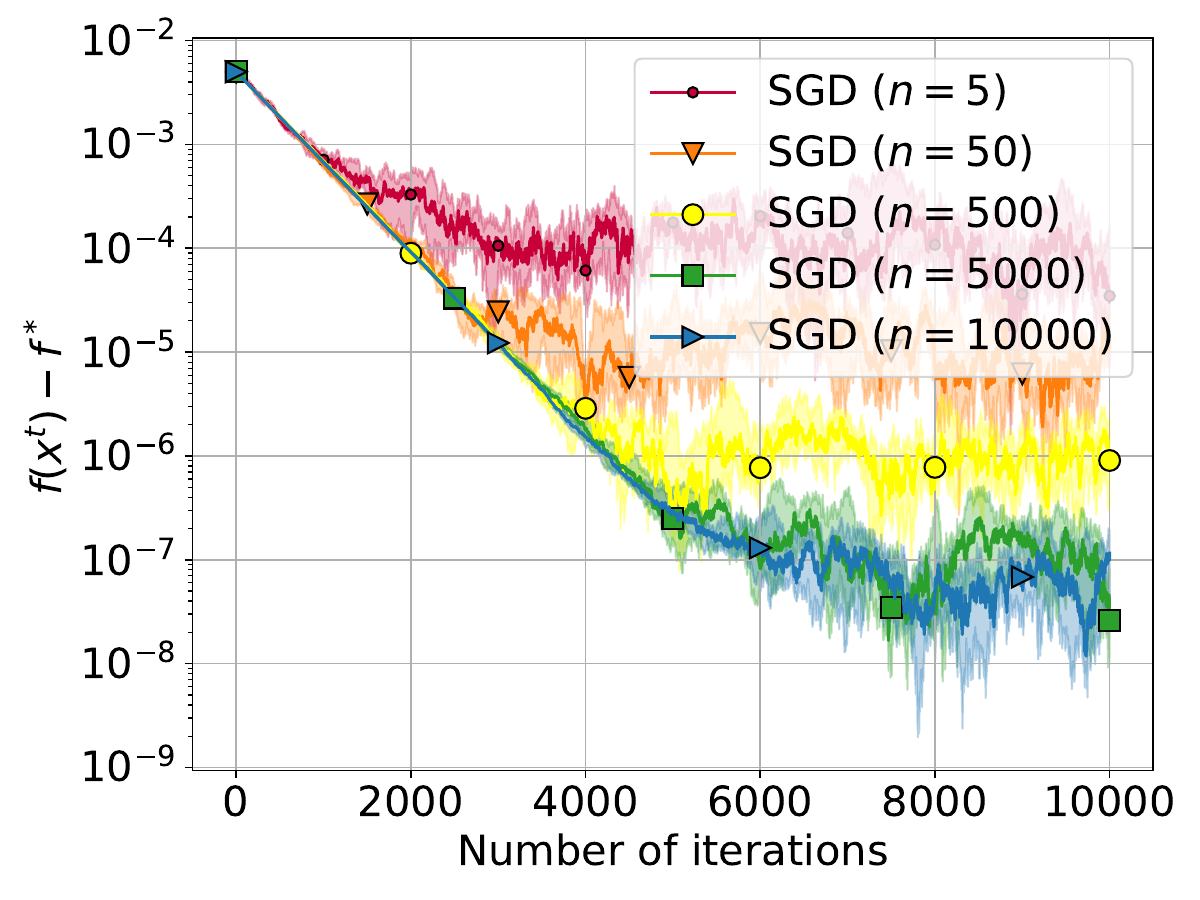}
			\caption{Step-size $\gamma = \frac{0.1}{\sqrt{T}}$.}
		\end{subfigure}
		\begin{subfigure}{.49\textwidth}
			\centering
			\includegraphics[width=1.0\textwidth]{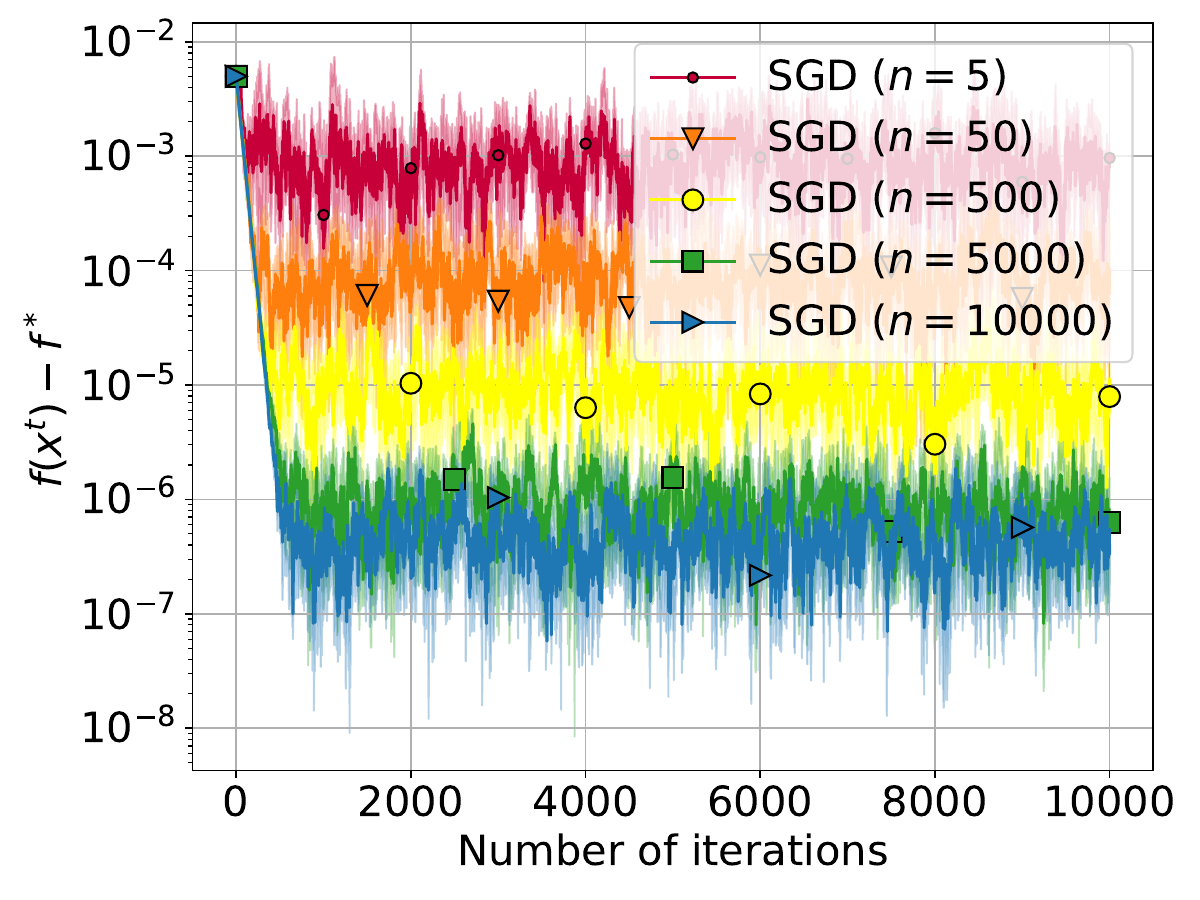}
			\caption{Step-size $\gamma = \frac{0.9}{\sqrt{T}}$.}
			
		\end{subfigure}
		\caption{Improvement with $n$ for \algname{SGD} without compression (inlcuded for a reference).  }
		\label{fig:sgd_toy}
	\end{minipage}\hfill
\end{figure}


\bibliographystyle{plainnat}
\bibliography{bibliography_ef21_extensions}

\end{document}
